\RequirePackage[l2tabu,orthodox]{nag}		
\documentclass[reqno]{amsart}		
\usepackage[margin=1.5in,bottom=1.25in]{geometry}		

\usepackage{amsmath}		
\usepackage{amssymb}		
\usepackage{amsfonts}		
\usepackage{amsthm}		
\usepackage[foot]{amsaddr}		

\usepackage{mathtools}		

\mathtoolsset{%
}

\usepackage[utf8]{inputenc}		
\usepackage[T1]{fontenc}		

\usepackage[
cal=cm,
]
{mathalfa}

\usepackage[proportional,tabular,lining,sf,mono=false]{libertine}

\usepackage{dsfont}		

\usepackage{acronym}		
\renewcommand*{\aclabelfont}[1]{\acsfont{#1}}		
\newcommand{\acdef}[1]{\textit{\acl{#1}} \textup{(\acs{#1})}\acused{#1}}		

\usepackage[labelfont={bf,small},labelsep=colon,font=small]{caption}	
\usepackage[dvipsnames,svgnames]{xcolor}		
\colorlet{MyRed}{Crimson!90!Black}
\colorlet{MyBlue}{MediumBlue!90!Black}
\colorlet{MyGreen}{DarkGreen!80!Black}

\newcommand{\afterhead}{.}		
\newcommand{\para}[1]{\medskip\paragraph{\textbf{#1\afterhead}}}

\usepackage{pifont}

\usepackage{wasysym}		

\usepackage{subcaption}		
\usepackage{tikz}		
\usepackage{tikzit}
\usetikzlibrary{calc,patterns}

\tikzstyle{new style 0}=[]
\tikzstyle{new style 1}=[]
\tikzstyle{toto3}=[dashed]

\tikzstyle{new edge style 0}=[dashed, fill=none, draw=black]
\tikzstyle{toto1}=[dashed, {|->}]
\tikzstyle{toto2}=[dashed, {|-|}]
\tikzstyle{toto3}=[dashed, {-}]

\usepackage[inline,shortlabels]{enumitem}		
\setenumerate{itemsep=\smallskipamount,topsep=\medskipamount,left=1em}
\setitemize{itemsep=\smallskipamount,topsep=\medskipamount,left=1em}

\usepackage[kerning=true]{microtype}		

\usepackage{cancel}		
\usepackage{xspace}		

\usepackage[numbers,sort&compress]{natbib}		

\bibpunct[, ]{[}{]}{,}{n}{,}{,}

\usepackage{hyperref}
\hypersetup{
colorlinks=true,
linktocpage=true,
pdfstartview=FitH,
breaklinks=true,
pdfpagemode=UseNone,
pageanchor=true,
pdfpagemode=UseOutlines,
plainpages=false,
bookmarksnumbered,
bookmarksopen=false,
bookmarksopenlevel=1,
hypertexnames=true,
pdfhighlight=/O,
urlcolor=MyBlue,linkcolor=MyBlue,citecolor=MyBlue,	
pdftitle={},
pdfauthor={},
pdfsubject={},
pdfkeywords={},
pdfcreator={pdfLaTeX},
pdfproducer={LaTeX with hyperref}
}

\usepackage[sort&compress,capitalize,nameinlink]{cleveref}		
\crefname{algorithm}{Algorithm}{Algorithms}
\crefname{equation}{Eq.}{Eqs.}

\usepackage{algorithm}		
\usepackage{algpseudocode}		

\usepackage{thmtools}		
\usepackage{thm-restate}		

\theoremstyle{plain}
\newtheorem{lemma}{Lemma}		
\newtheorem{proposition}{Proposition}		

\newtheorem*{corollary*}{Corollary}		

\theoremstyle{definition}
\newtheorem{example}{Example}		

\newtheorem*{definition*}{Definition}		
\newtheorem*{assumption*}{Assumptions}		
\newtheorem*{example*}{Example}		

\theoremstyle{remark}
\newtheorem{remark}{Remark}		
\newtheorem*{remark*}{Remark}		

\def\endenv{\hfill\raisebox{1pt}{\P}\smallskip}

\usepackage[suppress]{color-edits}	 

\newcommand{\debug}[1]{#1}		

\newcommand{\newmacro}[2]{\newcommand{#1}{\debug{#2}}}		
\newcommand{\newop}[2]{\DeclareMathOperator{#1}{\debug{#2}}}		

\DeclarePairedDelimiter{\braces}{\{}{\}}		
\DeclarePairedDelimiter{\bracks}{[}{]}		
\DeclarePairedDelimiter{\parens}{(}{)}		

\DeclarePairedDelimiter{\abs}{\lvert}{\rvert}		

\DeclarePairedDelimiterX{\setdef}[2]{\{}{\}}{#1:#2}		
\DeclarePairedDelimiterXPP{\exclude}[1]{\mathopen{}\setminus}{\{}{\}}{}{#1}

\newcommand{\R}{\mathbb{R}}		

\DeclareMathOperator*{\argmax}{arg\,max}		
\DeclareMathOperator*{\union}{\bigcup}		

\DeclareMathOperator{\bigoh}{\mathcal{O}}		
\DeclareMathOperator{\one}{\mathds{1}}		
\DeclareMathOperator{\relint}{ri}		
\DeclareMathOperator{\supp}{supp}		

\newcommand{\cf}{cf.\xspace}		
\newcommand{\eg}{e.g.,\xspace}		
\newcommand{\ie}{i.e.,\xspace}		

\newcommand{\dis}{\displaystyle}		
\newcommand{\txs}{\textstyle}		

\newcommand{\alt}[1]{#1'}		
\newcommand{\altalt}[1]{#1''}		

\newmacro{\dd}{\:d}		
\newcommand{\pd}{\partial}		

\newcommand{\insum}{\sum\nolimits}		

\newmacro{\const}{c}		
\newmacro{\coef}{\lambda}		
\newmacro{\param}{\theta}		
\newmacro{\params}{\Theta}		

\newmacro{\step}{\gamma}		

\newmacro{\pexp}{p}		
\newmacro{\qexp}{q}		
\newmacro{\rexp}{r}		

\newmacro{\beforestart}{0}		
\newmacro{\start}{1}		
\newmacro{\afterstart}{2}		
\newmacro{\running}{\start,\afterstart,\dotsc}		

\newmacro{\run}{t}		
\newmacro{\runalt}{s}		
\newmacro{\runaltalt}{k}		
\newmacro{\nRuns}{T}		
\newmacro{\runs}{\mathcal{\nRuns}}		

\newmacro{\state}{X}		
\newmacro{\statealt}{Y}		
\newmacro{\statealtalt}{Z}		

\newcommand{\new}[1]{#1^{+}}		

\newcommand{\init}[1][\state]{\debug{#1}_{\start}}		

\newcommand{\curr}[1][\state]{\debug{#1}_{\run}}		
\renewcommand{\next}[1][\state]{\debug{#1}_{\run+1}}		

\newcommand{\afterlast}[1][\state]{\debug{#1}_{\nRuns+1}}		

\newop{\Nash}{NE}		
\newop{\CE}{CE}		
\newop{\CCE}{CCE}		
\newop{\NI}{NI}		

\newop{\brep}{br}		
\newop{\reg}{Reg}		
\newop{\preg}{\overline{Reg}}		
\newop{\val}{val}		

\newmacro{\play}{i}		
\newmacro{\playalt}{j}		
\newmacro{\playaltalt}{k}		
\newmacro{\nPlayers}{N}		
\newmacro{\players}{\mathcal{\nPlayers}}		

\newmacro{\pure}{\alpha}		
\newmacro{\purealt}{\beta}		
\newmacro{\purealtalt}{\gamma}		
\newmacro{\nPures}{A}		
\newmacro{\pures}{\mathcal{\nPures}}		

\newmacro{\strat}{x}		
\newmacro{\stratalt}{\alt\strat}		
\newmacro{\strataltalt}{\altalt\strat}		
\newmacro{\strats}{\mathcal{X}}		
\newmacro{\intstrats}{\strats^{\circle}}		

\newmacro{\loss}{\ell}		
\newmacro{\cost}{c}		
\newmacro{\pay}{u}		
\newmacro{\payv}{v}		
\newmacro{\pot}{f}		

\newmacro{\game}{\mathcal{G}}		
\newmacro{\gameall}{\game(\players,\points,\loss)}		

\newmacro{\fingame}{\Gamma}		
\newmacro{\fingameall}{\Gamma(\players,\pures,\loss)}		

\newmacro{\vertex}{p}		
\newmacro{\vertexalt}{q}		
\newmacro{\vertexaltalt}{\altalt\vertex}		
\newmacro{\nVertices}{V}		
\newmacro{\vertices}{\mathcal{\nVertices}}		

\newmacro{\edge}{e}		
\newmacro{\edgealt}{\alt\edge}		
\newmacro{\edgealtalt}{\altalt\edge}		
\newmacro{\nEdges}{E}		
\newmacro{\edges}{\mathcal{\nEdges}}		

\newmacro{\graph}{\mathcal{G}}		
\newmacro{\graphall}{\graph(\vertices,\edges)}		

\newmacro{\vecspace}{\mathcal{V}}		
\newmacro{\subspace}{\mathcal{W}}		

\newmacro{\bvec}{e}		
\newmacro{\bvecs}{\mathcal{E}}		

\newmacro{\coord}{i}		
\newmacro{\coordalt}{j}		
\newmacro{\coordaltalt}{k}		
\newmacro{\nCoords}{d}		
\newmacro{\dims}{\nCoords}		
\newmacro{\vdim}{\nCoords}		

\newmacro{\pspace}{\mathcal{X}}		
\newmacro{\dspace}{\mathcal{Y}}		

\newmacro{\ppoint}{x}		
\newmacro{\ppointalt}{\alt\ppoint}		
\newmacro{\ppointaltalt}{\altalt\ppoint}		
\newmacro{\ppoints}{\mathcal{X}}		
\newmacro{\pstate}{X}		

\newmacro{\dpoint}{y}		
\newmacro{\dpointalt}{\alt\dpoint}		
\newmacro{\dpointaltalt}{\altalt\dpoint}		
\newmacro{\dpoints}{\mathcal{Y}}		
\newmacro{\dstate}{Y}		

\newmacro{\pvec}{z}		
\newmacro{\dvec}{w}		

\newmacro{\mat}{M}		
\newmacro{\hmat}{H}		

\newmacro{\ones}{\mathbf{1}}		
\newmacro{\eye}{I}		
\newmacro{\zer}{\mathbf{0}}		

\DeclarePairedDelimiterXPP{\dnorm}[1]{}{\lVert}{\rVert}{_{\ast}}{#1}		

\DeclarePairedDelimiterXPP{\onenorm}[1]{}{\lVert}{\rVert}{_{1}}{#1}		
\DeclarePairedDelimiterXPP{\twonorm}[1]{}{\lVert}{\rVert}{_{2}}{#1}		
\DeclarePairedDelimiterXPP{\supnorm}[1]{}{\lVert}{\rVert}{_{\infty}}{#1}		

\DeclarePairedDelimiterX{\braket}[2]{\langle}{\rangle}{#1,#2}		

\DeclarePairedDelimiterX{\inner}[2]{\langle}{\rangle}{#1,#2}		

\newcommand{\defeq}{\coloneqq}		
\newcommand{\eqdef}{\eqqcolon}		

\newcommand{\from}{\colon}		

\newop{\Opt}{Opt}		
\newop{\Sol}{Sol}		
\newop{\gap}{Gap}		
\newop{\orcl}{Or}		

\newmacro{\obj}{f}		
\newmacro{\objalt}{g}		
\newmacro{\sobj}{F}		

\newmacro{\gvec}{g}		
\newmacro{\oper}{A}		
\newmacro{\vecfield}{v}		

\newmacro{\vbound}{G}		
\newmacro{\lips}{L}		
\newmacro{\strong}{\alpha}		
\newmacro{\smooth}{\beta}		

\newop{\tspace}{T}		
\newop{\tcone}{TC}		
\newop{\dcone}{\tcone^{\ast}}		
\newop{\ncone}{NC}		
\newop{\pcone}{PC}		
\newop{\hull}{\Delta}		

\newmacro{\cvx}{\mathcal{C}}		
\newmacro{\subd}{\partial}		

\newop{\Eucl}{\Pi}		
\newop{\logit}{\Lambda}		
\newop{\dkl}{KL}		

\newmacro{\hreg}{h}		
\newmacro{\hconj}{\hreg^{\ast}}		
\newmacro{\breg}{D}		
\newmacro{\mprox}{P}		
\newmacro{\mirror}{Q}		
\newmacro{\fench}{F}		
\newmacro{\hstr}{K}		
\newmacro{\hrange}{H}		
\newmacro{\proxdom}{\points^{\hreg}}		

\DeclarePairedDelimiterXPP{\proxof}[2]{\mprox_{#1}}{(}{)}{}{#2}		

\newmacro{\point}{x}		
\newmacro{\pointalt}{\alt\point}		
\newmacro{\pointaltalt}{\altalt\point}		
\newmacro{\points}{\mathcal{K}}		
\newmacro{\intpoints}{\relint\points}		

\newmacro{\base}{p}		
\newmacro{\basealt}{q}		
\newmacro{\basealtalt}{u}		

\newmacro{\open}{\mathcal{U}}		
\newmacro{\closed}{\mathcal{C}}		
\newmacro{\cpt}{\mathcal{K}}		
\newmacro{\nhd}{\mathcal{U}}		

\newop{\ex}{\mathbb{E}}		
\newop{\prob}{\mathbb{P}}		
\newop{\Var}{Var}		
\newop{\simplex}{\hull}		

\providecommand{\given}{}

\DeclarePairedDelimiterXPP{\exof}[1]{\ex}{[}{]}{}{
\renewcommand\given{\nonscript\,\delimsize\vert\nonscript\,\mathopen{}} #1}

\DeclarePairedDelimiterXPP{\probof}[1]{\prob}{(}{)}{}{
\renewcommand\given{\nonscript\:\delimsize\vert\nonscript\:\mathopen{}} #1}

\DeclarePairedDelimiterXPP{\oneof}[1]{\one}{\{}{\}}{}{
\renewcommand\given{\nonscript\,\delimsize\vert\nonscript\,\mathopen{}} #1}

\newmacro{\sample}{\omega}		
\newmacro{\samples}{\Omega}		

\newmacro{\filter}{\mathcal{F}}		
\newmacro{\probspace}{(\samples,\filter,\prob)}		

\newmacro{\event}{E}       
\newmacro{\eventalt}{H}       
\newmacro{\mean}{\mu}		
\newmacro{\sdev}{\sigma}		
\newmacro{\variance}{\sdev^{2}}		

\newcommand{\est}[1]{\hat #1}		

\newmacro{\error}{Z}		
\newmacro{\noise}{Z}		
\newmacro{\bias}{b}		

\newmacro{\serror}{\theta}		
\newmacro{\snoise}{\xi}		
\newmacro{\sbias}{\psi}		

\newmacro{\sbound}{M}		
\newmacro{\bbound}{B}		
\newmacro{\noisepar}{\sdev}		
\newmacro{\noisevar}{\variance}		

\newmacro{\elem}{a}		
\newmacro{\elemalt}{\alt\elem}		
\newmacro{\elems}{\mathcal{A}}		
\newmacro{\iElem}{i}		
\newmacro{\nElems}{n}		

\newmacro{\partition}{\mathcal{S}}		
\newcommand{\refined}{\mathrel{}\succcurlyeq\mathrel{}}		

\newmacro{\set}{A}		
\newmacro{\setalt}{\alt\set}		

\newmacro{\class}{S}		
\newmacro{\classalt}{\alt\class}		
\newmacro{\classaltalt}{\altalt\class}		
\newmacro{\classes}{\partition}		
\newmacro{\iClass}{j}		
\newmacro{\nClasses}{n}		
\newmacro{\allClasses}{\mathcal{S}}		
\newmacro{\struct}{\allClasses}		

\newmacro{\range}{R}
\newmacro{\incr}{r}

\newmacro{\basepay}{r}
\newmacro{\basecost}{c}
\newmacro{\rangecost}{C}

\newmacro{\lvl}{\ell}		
\newmacro{\lvlalt}{k}		
\newmacro{\nLvls}{L}		
\newcommand{\atlvl}[1]{_{#1}}

\newmacro{\source}{\elems}		

\DeclarePairedDelimiterXPP{\setof}[1]{\debug{\mathrm{elem}}}{(}{)}{}{#1}
\DeclarePairedDelimiterXPP{\attof}[1]{\debug{\mathrm{attr}}}{(}{)}{}{#1}

\newmacro{\parent}{\class}		
\newmacro{\parentset}{\set}		
\newcommand{\parentof}{\mathrel{}\vartriangleright\mathrel{}}		

\newmacro{\child}{\classalt}		
\newmacro{\childalt}{\classaltalt}		
\newmacro{\childset}{\setalt}		
\newcommand{\childof}{\mathrel{}\vartriangleleft\mathrel{}}		

\newmacro{\sibling}{\sim}
\newop{\children}{ch}
\newmacro{\nChildren}{m}		

\newcommand{\desc}[1][]{\mathrel{}\prec_{#1}\mathrel{}}		
\newcommand{\anc}[1][]{\mathrel{}\succ_{#1}\mathrel{}}		
\newcommand{\desceq}[1][]{\mathrel{}\preccurlyeq_{#1}\mathrel{}}		
\newcommand{\anceq}[1][]{\mathrel{}\succcurlyeq_{#1}\mathrel{}}		

\newcommand{\lineage}[1]{\source \equiv #1\atlvl{0} \parentof #1\atlvl{1} \parentof \dotsm \parentof #1\atlvl{\lvl} \equiv #1}
\newcommand{\leafpath}[1]{#1\atlvl{\nLvls} \childof #1\atlvl{\nLvls-1} \childof \dotsm \childof #1\atlvl{0} \equiv \source}
\newcommand{\classleaf}[2]{#1 \equiv #1\atlvl{\lvl} \parentof #1\atlvl{\lvl+1} \parentof \dotsm \parentof #1\atlvl{\nLvls} \equiv \{#2\}}

\newmacro{\temp}{\mu}		
\newmacro{\diff}{\delta}		
\newmacro{\learn}{\eta}		
\newmacro{\score}{y}		
\newmacro{\pf}{Z}		

\newcommand{\choice}[1][]{%
\renewcommand\given{\mathrel{}{\nonscript\mkern-\medmuskip}|{\nonscript\mkern-\medmuskip} \mathrel{}}%
\debug P_{#1}}		

\newmacro{\energy}{E}		
\newmacro{\test}{p}		
\newmacro{\eff}{\alpha}		

\newmacro{\basemodelcost}{\est\basecost}
\newmacro{\incrmodel}{\est\incr}		
\newmacro{\model}{\est\cost}		

\newmacro{\nEff}{\nElems_{\mathrm{eff}}}		

\newcommand{\redbus}{\textrm{red bus}\xspace}
\newcommand{\bluebus}{\textrm{blue bus}\xspace}
\newcommand{\car}{\textrm{car}\xspace}

\addauthor[\textbf{MM}]{MM}{Crimson}

\addauthor[\textbf{Pan}]{PM}{Blue}
\newcommand{\PM}{\PMmargincomment}
\newcommand{\explain}[1]{\tag*{\itshape\#\:#1}}

\newmacro{\aux}{\tilde\state}
\newmacro{\daux}{\tilde\dstate}

\addauthor[\textbf{TR}]{TR}{Green}
\newcommand{\TR}{\TRmargincomment}

\addauthor[\textbf{HZ}]{HZ}{Purple}

\begin{document}


\newcommand{\longtitle}{\uppercase{Nested Bandits}}

\title{\longtitle}		

\author
[M.~Martin]
{Matthieu Martin$^{\ast}$}
\address{$^{\ast}$\,%
Criteo AI Lab.}
\email{mat.martin@criteo.com}

\author
[P.~Mertikopoulos]
{Panayotis Mertikopoulos$^{\diamond,\ast}$}
\address{$^{\diamond}$\,%
Univ. Grenoble Alpes, CNRS, Inria, Grenoble INP, LIG, 38000 Grenoble, France.}
\email{panayotis.mertikopoulos@imag.fr}

\author
[T.~Rahier]
{\\Thibaud Rahier$^{\ast}$}
\email{t.rahier@criteo.com}

\author
[H.~Zenati]
{Houssam Zenati$^{\ast,\S}$}
\address{$^{\S}$\,%
Univ. Grenoble  Alpes, Inria, CNRS, Grenoble INP, LJK, 38000 Grenoble, France.}
\email{h.zenati@criteo.com}


\thanks{Authors appear in alphabetical order.}


\newacro{LHS}{left-hand side}
\newacro{RHS}{right-hand side}
\newacro{iid}[i.i.d.]{independent and identically distributed}
\newacro{lsc}[l.s.c.]{lower semi-continuous}

\newacro{IIA}{independence from irrelevant alternatives}
\newacro{NLC}{nested logit choice}
\newacro{NEW}{nested exponential weights}
\newacro{IWE}{importance-weighted estimator}
\newacro{NIWE}{nested importance weighted estimator}
\newacro{EW}{exponential weights}
\newacro{PoAf}{price of affinity}
\newacro{FTRL}{``follow the regularized leader''}
\newacro{EXP3}{exponential weights algorithm for exploration and exploitation}
\newacro{MD}{mirror descent}
\newacro{DA}{dual averaging}
\newacro{NE}{Nash equilibrium}
\newacroplural{NE}[NE]{Nash equilibria}

\begin{abstract}
%
%
In many online decision processes, the optimizing agent is called to choose between large numbers of alternatives with many inherent similarities;
in turn, these similarities imply closely correlated losses that may confound standard discrete choice models and bandit algorithms.
We study this question in the context of \emph{nested bandits}, a class of adversarial multi-armed bandit problems where the learner seeks to minimize their regret in the presence of a large number of distinct alternatives with a hierarchy of embedded (non-combinatorial) similarities.
In this setting, optimal algorithms based on the \acl{EW} blueprint (like Hedge, EXP3, and their variants) may incur significant regret because they tend to spend excessive amounts of time exploring irrelevant alternatives with similar, suboptimal costs.
To account for this, we propose a \acdef{NEW} algorithm that performs a layered exploration of the learner's set of alternatives based on a nested, step-by-step selection method.
In so doing, we obtain a series of tight bounds for the learner's regret showing that online learning problems with a high degree of similarity between alternatives can be resolved efficiently, without a red bus / blue bus paradox occurring.
\end{abstract}

\maketitle
\allowdisplaybreaks		
\acresetall		

\section{Introduction}
\label{sec:introduction}

Consider the following discrete choice problem (known as the ``\redbus\,/\,\bluebus paradox'' in the context of transportation economics).
A commuter has a choice between taking a car or bus to work:
commuting by car takes on average half an hour modulo random fluctuations, whereas commuting by bus takes an hour, again modulo random fluctuations (it's a long commute).
Then, under the classical multinomial logit choice model for action selection \cite{Luc59,McF74a}, the commuter's odds for selecting a car over a bus would be $\exp(-1/2) / \exp(-1) \approx 1.6:1$.
This indicates a very clear preference for taking a car to work and is commensurate with the fact that, on average, commuting by bus takes twice as long.

Consider now the same model but with a twist.
The company operating the bus network purchases a fleet of new buses that are otherwise completely identical to the existing ones, except for their color:
old buses are red, the new buses are blue.
This change has absolutely no effect on the travel time of the bus;
however, since the new set of alternatives presented to the commuter is $\{\car,\redbus,\bluebus\}$, the odds of selecting a car over a bus (red or blue, it doesn't matter) now drops to $\exp(-1/2) / [\exp(-1) + \exp(-1)] \approx 0.8:1$.
Thus, by introducing an \emph{irrelevant} feature (the color of the bus), the odds of selecting the alternative with the highest utility have dropped dramatically, to the extent that commuting by car is no longer the most probable outcome in this example.

Of course, the shift in choice probabilities may not always be that dramatic, but the point of this example is that the presence of an irrelevant alternative (the \bluebus) would always induce such a shift \textendash\ 
which is, of course, absurd.
In fact, the red bus\,/\,blue bus paradox was originally proposed as a  sharp criticism of the \ac{IIA} axiom that underlies the multinomial logit choice model \citep{Luc59} and which makes it unsuitable for choice problems with inherent similarities between different alternatives.
In turn, this has led to a vast corpus of literature in social choice and decision theory, with an extensive array of different axioms and models proposed to overcome the failures of the \ac{IIA} assumption.
For an introduction to the topic, we refer the reader to the masterful accounts of \citet{McF74a}, \citet{BAL85} and \citet{AdPT92}.

Perhaps surprisingly, the implications of the red bus\,/\,blue bus paradox have not been explored in the context of online learning, despite the fact that similarities between alternatives are prevalent in the field's application domains \textendash\ for example, in recommender systems with categorized product recommendation catalogues, in the economics of transport and product differentiation, etc.
What makes this gap particularly pronounced is the fact that logit choice underlies some of the most widely used algorithmic schemes for learning in multi-armed bandit problems \textendash\ namely the \ac{EXP3} \citep{Vov90,LW94,ACBFS95} as well as its variants, Hedge \citep{ACBF02}, EXP3.P \citep{ACBFS02}, EXP3-IX \citep{KNVM14}, EXP4 \citep{ACBFS02} / EXP4-IX \citep{Neu15}, etc.
Thus, given the vulnerability of logit choice to irrelevant alternatives, it stands to reason that said algorithms may be suboptimal when faced with a set of alternatives with many inherent similarities.

\para{Our contributions}

Our paper examines this question in the context of repeated decision problems where a learner seeks to minimize their regret in the presence of a large number of distinct alternatives with a hierarchy of embedded (non-combinatorial) similarities.
This similarity structure, which we formalize in \cref{sec:model}, is defined in terms of a nested series of attributes \textendash\ like ``type'' or ``color'' \textendash\ and induces commensurate similarities to the losses of alternatives that lie in the same class (just as the red and blue buses have identical losses in the example described above).

Inspired by the \acl{NLC} model introduced by \citet{McF74a} to resolve the original red bus\,/\,blue bus paradox, we develop in \cref{sec:NEW} a \acdef{NEW} algorithm for no-regret learning in decision problems of this type.
Our main result is that the regret incurred by \ac{NEW} is bounded as $\bigoh(\sqrt{\nEff \log\nElems \cdot \nRuns})$, where $\nElems$ is the total number of alternatives and $\nEff$ is the ``effective'' number when taking similarities into account (for example, in the standard red bus\,/\,blue bus paradox, $\nEff = 2$, \cf \cref{sec:results}).
The gap between nested and non-nested algorithms can be quantified by the problem's \acdef{PoAf}, defined here as the ratio $\eff = \sqrt{\nElems / \nEff}$  measuring the worst-case ratio between the regret guarantees of the \ac{NEW} and \ac{EXP3} algorithms (the latter scaling as $\bigoh(\sqrt{\nElems\log\nElems \cdot \nRuns})$ in the problem at hand).

In practical applications (such as the type of recommendation problems that arise in online advertising), $\eff$ can be exponential in the number of attributes, indicating that the \ac{NEW} algorithm could lead to significant performance gains in this context.
We verify that this is indeed the case in a range of synthetic experiments in \cref{sec:numerics}.

\para{Related Work}

The problem of exploiting the structure of the loss model and/or any side information available to the learner is a staple of the bandit literature.
More precisely, in the setting of contextual bandits, the learner is assumed to observe some ``context-based'' information and tries to learn the ``context to reward'' mapping underlying the model in order to make better predictions.
Bandit algorithms of this type \textendash\ like EXP4 \textendash\ are often studied as ``expert'' models \citep{ACBFS02,CBL06} or attempt to model the agent's loss function with a semi-parametric contextual dependency in the stochastic setting to derive optimistic action selection rules \citep{AYY11};
for a survey, we refer the reader to \cite{LS20} and references therein.
While the nested bandit model we study assumes an additional layer of information relative to standard bandit models, there are no experts or a contextual mapping conditioning the action taken, so it is not comparable to the contextual setup.

The type of feedback we consider assumes that the learner observes the ``intra-class'' losses of their chosen alternative, similar to the semi-bandit in the study of combinatorial bandit algorithms \cite{CBL12,GLLO07}.
However, the similarity with combinatorial bandit models ends there:
even though the categorization of alternatives gives rise to a tree structure with losses obtained at its leaves, there is no combinatorial structure defining these costs, and modeling this as a combinatorial bandit would lead to the same number of arms and ground elements, thus invalidating the concept.

Besides these major threads in the literature, \citep{TS18} recently showed that the range of losses can be exploited with an additional free observation, while \citep{CBS18} improves the regret guarantees by using effective loss estimates.
However, both works are susceptible to the advent of irrelevant alternatives and can incur significant regret when faced with such a problem.
Finally, in the Lipschitz bandit setting, \citep{CBGG+17,HMMR21} obtain order-optimal regret bounds by building a hierarchical covering model in the spirit of \cite{BMSS11};
the correlations induced by a Lipschitz loss model cannot be compared to our model, so there is no overlap of techniques or results.


\section{The model}
\label{sec:model}

We begin in this section by defining our general nested choice model.
Because the technical details involved can become cumbersome at times, it will help to keep in mind the running example of a music catalogue where songs are classified by, say,
genre (classical music, jazz, rock,\dots),
artist (Rachmaninov, Miles Davis, Led Zeppelin,\dots),
and
album.
This is a simple \textendash\ but not simplistic \textendash\ use case which requires the full capacity of our model, so we will use it as our ``go-to'' example throughout.

\subsection{Attributes, classes, and the relations between them}
\label{eq:attributes}

Let $\elems = \setdef{\elem_{\iElem}}{\iElem = 1,\dotsc,\nElems}$ be a set of \emph{alternatives} (or \emph{atoms}) indexed by $\iElem = 1,\dotsc,\nElems$.
A \emph{similarity structure} (or \emph{structure of attributes}) on $\elems$ is defined as a tower of nested \emph{similarity partitions} (or \emph{attributes}) $\classes\atlvl{\lvl}$, $\lvl=0,\dotsc,\nLvls$, of $\elems$ with $\{\elems\} \eqdef \classes\atlvl{0} \refined \classes\atlvl{1} \refined \dotsm \refined \classes\atlvl{\nLvls} \defeq \setdef{\{\elem\}}{\elem \in \elems}$.
As a result of this definition, each partition $\classes\atlvl{\lvl}$ captures successively finer attributes of the elements of $\elems$ (in our music catalogue example, these attributes would correspond to genre, artist, album, etc.).%
\footnote{The trivial partitions $\classes\atlvl{0} = \{\elems\}$ and $\classes\atlvl{\nLvls} = \setdef{\{\elem\}}{\elem\in\elems}$ do not carry much information in themselves, but they are included for completeness and notational convenience later on.}
Accordingly, each constituent set $\set$ of a partition $\classes\atlvl{\lvl}$ will be referred to as a \emph{similarity class} and we assume it collects all elements of $\elems$ that share the attribute defining $\classes\atlvl{\lvl}$:
for example, a similarity class for the attribute ``artist'' might consist of all Beethoven symphonies, all songs by Led Zeppelin, etc.

Collectively, a structure of attributes will be represented by the disjoint union
\begin{equation}
\label{eq:struct}
\struct
	\defeq \coprod\nolimits_{\lvl=0}^{\nLvls} \classes\atlvl{\lvl}
	\equiv \union\nolimits_{\lvl=0}^{\nLvls} \setdef{(\set,\lvl)}{\set\in\classes\atlvl{\lvl}}
\end{equation}
of all class/attribute pairs of the form $(\set,\lvl)$ for $\set\in\classes\atlvl{\lvl}$.
In a slight abuse of terminology (and when there is no danger of confusion),
the pair $\class = (\set,\lvl)$ will also be referred to as a ``class'',
and
we will write $\class\in\classes\atlvl{\lvl}$ and $\elem\in\class$ instead of $\set\in\classes\atlvl{\lvl}$ and $\elem\in\set$ respectively.
By contrast, when we need to clearly distinguish between a class and its underlying set, we will write $\set = \setof{\class}$ for the set of atoms contained in $\class$ and
$\lvl = \attof{\class}$ for the attached attribute label.

\smallskip
\begin{remark}
The reason for including the attribute label $\lvl$ in the definition of $\struct$ is that a set of alternatives may appear in different partitions of $\elems$ in a different context.
For example, if ``IV'' is the only album by Led Zeppelin in the catalogue, the album's track list represents both the set of ``all songs in IV'' as well as the set of ``all Led Zeppelin songs''.
However, the focal attribute in each case is different \textendash\ ``artist'' in the former versus ``album'' in the latter \textendash\ and this additional information would be lost in the non-discriminating union $\union_{\lvl=0}^{\nLvls} \classes\atlvl{\lvl}$ (unless, of course, the partitions $\classes\atlvl{\lvl}$ happen to be mutually disjoint, in which case the distinction between ``union'' and ``disjoint union'' becomes set-theoretically superfluous).
\endenv
\end{remark}

Moving forward, if a class $\parent\in\classes\atlvl{\lvl}$ contains the class $\child\in\classes\atlvl{\lvlalt}$ for some $\lvlalt>\lvl$, we will say that $\child$ is a \emph{descendant} of $\parent$ (resp.~$\parent$ is an \emph{ancestor} of $\child$), and we will write ``$\child \desc \parent$'' (resp.~``$\parent \anc \child$'').%
\footnote{More formally, we will write $\child \desc \parent$ when $\setof{\child} \subseteq \setof{\parent}$ and $\attof{\child} > \attof{\parent}$.
The corresponding weak relation ``$\desceq$'' is defined in the standard way, \ie allowing for the case $\attof{\child} = \attof{\parent}$ which in turn implies that $\child = \parent$.}
As a special case of this relation, if $\child \desc \parent$ and $\lvlalt = \lvl+1$, we will say that $\child$ is a \emph{child} of $\parent$ (resp.~$\parent$ is  \emph{parent} of $\child$) and we will write ``$\child \childof \parent$'' (resp.~``$\parent \parentof \child$'').
For completeness, we will also say that $\alt\class$ and $\altalt\class$ are \emph{siblings} if they are children of the same parent,
and we will write $\alt\class \sibling \altalt\class$ in this case.
Finally, when we wish to focus on descendants sharing a certain attribute, we will write ``$\child \desc[\lvl] \parent$'' as shorthand for the predicate ``$\child \desc \parent$ and $\attof{\child} = \lvl$''.

Building on this, a similarity structure on $\elems$ can also be represented graphically as a rooted directed tree \textendash\ an \emph{arborescence} \textendash\ by connecting two classes $\parent,\child\in\struct$ with a directed edge $\parent\to\child$ whenever $\parent \parentof \child$.
By construction, the root of this tree is $\source$ itself,%
\footnote{Stricto sensu, the root of the tree is $(\elems,0)$, but since there is no danger of confusion, the attribute label ``0'' will be dropped.}
and
the unique directed path $\lineage{\class}$ from $\source$ to any class $\class\in\struct$ will be referred to as the \emph{lineage} of $\class$.
For notational simplicity, we will not distinguish between $\struct$ and its graphical representation, and we will use the two interchangeably;
for an illustration, see \cref{fig:tree}.


\begin{figure}[tbp]
\ctikzfig{Figures/tree1}
\caption{A structure with $\nLvls = 3$ attributes on the set $\elems = \{\elem_{1},\dotsc,\elem_{8}\}$;
for example, the class $\class_{2}^{1}$ consists of $\{\elem_{3},\elem_{4}\}$.}
\label{fig:tree}
\vspace{-2ex}
\end{figure}



\subsection{The loss model}
\label{sec:costs}

Throughout what follows, we will consider loss models in which alternatives that share a common set of attributes incur similar costs, with the degree of similarity depending on the number of shared attributes.
More precisely, given a similarity class $\parent \in \struct$, we will assume that all its immediate subclasses $\child$ share the same base cost $\cost_{\parent}$ (determined by the parent class $\parent$) plus an idiosyncratic cost increment $\incr_{\child}$ (which is specific to the child $\child\childof\parent$ in question).
Formally, starting with $\cost_{\source} = 0$ (for the root class $\source$), this boils down to the recursive definition
\begin{equation}
\label{eq:cost-rec}
\cost_{\child}
	= \cost_{\parent} + \incr_{\child}
	\quad
	\text{for all $\child\childof\parent$},
\end{equation}
which, when unrolled over the lineage $\lineage{\class}$ of a target class $\class\in\classes\atlvl{\lvl}$, yields the expression
\begin{equation}
\label{eq:cost-class}
\cost_{\class}
	= \insum_{\classalt \anceq \class} \incr_{\classalt}
	= \incr_{\class\atlvl{1}} + \dotsm + \incr_{\class\atlvl{\lvl}}.
\end{equation}
Thus, in particular, when $\class \gets \elem \in \elems$, the cost assigned to an individual alternative $\elem\in\elems$ will be given by
\begin{equation}
\label{eq:cost-elem}
\cost_{\elem}
	= \insum_{\lvl=1}^{\nLvls} \incr_{\class\atlvl{\lvl}}
	= \insum_{\class \ni \elem} \incr_{\class}
	\quad
	\text{for all $\elem\in\elems$}.
\end{equation}

Finally, to quantify the ``intra-class'' variability of costs, we will assume throughout that the idiosyncratic cost increments within a given parent class $\parent$ are bounded as
\begin{equation}
\label{eq:range}
\incr_{\child}
	\in [0,\range_{\parent}]
	\quad
	\text{for all $\child\childof\parent$}.
\end{equation}
This terminology is justified by the fact that, under the loss model \eqref{eq:cost-rec}, the costs $\cost_{\child},\cost_{\childalt}$ to any two \emph{sibling} classes $\child,\childalt\childof{\parent}$ (\ie any two classes parented by $\parent$) differ by at most $\range_{\parent}$.
Analogously, the costs to any two alternatives $\elem,\elemalt\in\elems$ that share a set of common attributes $\class\atlvl{1},\dotsc,\class\atlvl{\lvl}$ will differ by at most $\sum_{\lvlalt=\lvl+1}^{\nLvls} \range_{\class\atlvl{\lvlalt}}$.

\begin{example}
To represent the original \redbus\,/\,\bluebus problem as an instance of the above framework,
let $\classes\atlvl{1} = \{\{\redbus,\bluebus\},\car\}$ be the partition of the set $\elems = \braces{\redbus,\bluebus,\car}$ by type (``\texttt{bus}'' or ``\texttt{car}''), and let $\classes\atlvl{2}$ be the corresponding sub-partition by color (``\texttt{red}'' or ``\texttt{blue}'' for elements of the class ``\texttt{bus}'').
The fact that color does not affect travel times may then be represented succinctly by taking $\range\atlvl{\textrm{color}} = 0$ (representing the fact that color does not affect travel times).
\endenv
\end{example}

\begin{remark}
We make no distinction here between $\cost_{\elem}$ and $\cost_{\{\elem\}}$, \ie between an alternative $\elem$ of $\elems$ and the (unique) singleton class of $\{\elem\} \in \classes\atlvl{\nLvls}$ containing it.
This is done purely for reasons of notational convenience.
\endenv
\end{remark}

\begin{remark}
For posterity, we also note that the optimizing agent is assumed to be aware of the cost decomposition \eqref{eq:cost-elem} after selecting an alternative $\elem\in\elems$.
In the context of combinatorial bandits \cite{CBL12} this would correspond to the so-called ``semi-bandit'' setting.
\endenv
\end{remark}

\subsection{Sequence of events}
\label{sec:sequence}

With all this in hand, we will consider a generic online decision process that unfolds over a set of alternatives $\elems$ endowed with a similarity structure $\struct = \coprod_{\lvl}\classes\atlvl{\lvl}$ as follows:
\begin{enumerate}
\item
At each stage $\run = \running$, the learner selects an alternative $\curr[\elem] \in \elems$ by selecting attributes from $\struct$ one-by-one.
\item
Concurrently, nature sets the idiosyncratic, intra-class losses $\incr_{\class,\run}$ for each similarity class $\class\in\struct$.
\item
The learner incurs $\incr_{\class,\run}$ for each chosen class $\class\ni\curr[\elem]$ for a total cost of $\curr[\cost] = \sum_{\class\ni\curr[\elem]} \incr_{\class,\run}$, and the process repeats.
\end{enumerate}
To align our presentation with standard bandit models with losses in $[0,1]$, we will assume throughout that $\sum_{\class\ni\elem} \range_{\class} \leq 1$ for all $\elem\in\elems$, meaning in particular that the maximal cost incurred by any alternative $\elem\in\elems$ is upper bounded by $1$.
Other than this normalization, the sequence of idiosyncratic loss vectors $\incr_{\run} \in \R^{\struct}$, $\run=\running$, is assumed arbitrary and unknown to the learner as per the standard adversarial setting \cite{CBL06,SS11}.

To avoid deterministic strategies that could be exploited by an adversary, we will assume that the learner selects an alternative $\curr[\elem]$ at time $\run$ based on a mixed strategy $\curr \in \simplex(\elems)$, \ie $\curr[\elem] \sim \curr$.
The regret of a policy $\curr$, $\run=\running$, against a benchmark strategy $\test\in\simplex(\elems)$ is then defined as the cumulative difference between the player's mean cost under $\test$ and $\curr$, that is
\begin{equation}
\label{eq:regret-test}
\reg_{\test}(\nRuns)
	= \sum_{\run=\start}^{\nRuns}
		\bracks*{\ex_{\curr}[\cost_{\curr[\elem],\run}] - \ex_{\test}[\cost_{\curr[\elem],\run}]}
	= \sum_{\run=\start}^{\nRuns}
		\braket{\curr[\cost]}{\curr - \test}
\end{equation}
where $\curr[\cost] = (\cost_{\elem,\run})_{\elem\in\elems} \in\R^{\elems}$ denotes the vector of costs encountered by the learner at time $\run$, \ie $\cost_{\elem,\run} = \sum_{\class\ni\elem} \incr_{\class,\run}$ for all $\elem\in\elems$.
This definition will be our main figure of merit in the sequel.

\section{The \acl{NEW} algorithm}
\label{sec:NEW}

Our goal in what follows will be to design a learning policy capable of exploiting the type of similarity structures introduced in the previous section.
The main ingredients of our method are a nested attribute selection and cost estimation rule, which we describe in detail in \cref{sec:choice,sec:NIWE} respectively;
the proposed \acdef{NEW} algorithm is then developed and discussed in \cref{sec:algo}.

\subsection{Probabilities, propensities, and \acl{NLC}}
\label{sec:choice}

We begin by introducing the attribute selection scheme that forms the backbone of our proposed policy.
Our guiding principle in this is the \acdef{NLC} rule of \citet{McF74a} which selects an alternative $\elem\in\elems$ by traversing $\struct$ one attribute at a time and prescribing the corresponding conditional choice probabilities at each level of $\struct$.

To set the stage for all this, if $\strat = (\strat_{1},\dotsc,\strat_{\nElems}) \in \simplex(\elems)$ is a mixed strategy on $\elems$
we will write
\begin{align}
\label{eq:prob-class}
\strat_{\class}
	&\txs
	= \sum_{\elem\in\class} \strat_{\elem}
\shortintertext{for the probability of choosing $\class \in \struct$ under $\strat$, and}
\label{eq:prob-cond}
\strat_{\classalt \vert \class}
	&= \strat_{\classalt} / \strat_{\class}
\end{align}
for the conditional probability of choosing a descendant $\classalt$ of $\class$ assuming that $\class$ has already been selected under $\strat$.%
\footnote{%
Note here that the joint probability of selecting \emph{both} $\class$ and $\classalt$ under $\strat$ is simply $\strat_{\classalt}$ whenever $\classalt \desceq \class$.}
Then the \ac{NLC} rule proceeds as follows:
first, it prescribes choice probabilities $\strat_{\class\atlvl{1}}$ for all classes $\class\atlvl{1}\in\classes\atlvl{1}$ (\ie the coarsest ones);
subsequently, once a class $\class\atlvl{1} \in \classes\atlvl{1}$ has been selected, \ac{NLC} prescribes the conditional choice probabilities $\strat_{\class\atlvl{2} \vert \class\atlvl{1}}$ for all children $\class\atlvl{2}$ of $\class\atlvl{1}$ and draws a class from $\classes\atlvl{2}$ based on $\strat_{\class\atlvl{2} \vert \class\atlvl{1}}$.
The process then continues downwards along $\struct$ until reaching the finest partition $\classes\atlvl{\nLvls}$ and selecting an atom $\{\elem\} \equiv \leafpath{\class}$.

This step-by-step selection process captures the ``nested'' part of the \acl{NLC} rule;
the ``logit'' part refers to the way that the conditional probabilities \eqref{eq:prob-cond} are actually prescribed given the agent's predisposition towards each alternative $\elem\in\elems$.
To make this precise, suppose that the learner associates to each element $\elem\in\elems$ a \emph{propensity score} $\score_{\elem} \in \R$ indicating their tendency \textendash\ or \emph{propensity} \textendash\ to select it.
The associated propensity score of a similarity class $\class\atlvl{\lvl-1}\in\classes\atlvl{\lvl-1}$, $\lvl=1,\dotsc,\nLvls$, is then defined inductively as
\begin{equation}
\label{eq:score}
\score_{\class\atlvl{\lvl-1}}
	= \temp\atlvl{\lvl} \log\insum_{\class\atlvl{\lvl} \childof \class\atlvl{\lvl-1}} \exp(\score_{\class\atlvl{\lvl}} / \temp\atlvl{\lvl})
\end{equation}
where
$\temp\atlvl{\lvl} > 0$ is a tunable parameter that reflects the learner's \emph{uncertainty level} regarding the $\lvl$-th attribute $\classes\atlvl{\lvl}$ of $\struct$.
In words, this means that the score of a class is the weighted softmax of the scores of its children;
thus, starting with the individual alternatives of $\elems$ \textendash\ that is, the \emph{leaves} of $\struct$ \textendash\ propensity scores are propagated backwards along $\struct$, and this is repeated one attribute at a time until reaching the root of $\struct$.

\begin{remark}
We should also note that \cref{eq:score} assigns a propensity score to \emph{any} similarity class $\class\in\struct$.
However, because the primitives of this assignment are the original scores assigned to each alternative $\elem\in\elems$, we will reserve the notation $\score = (\score_{1},\dotsc,\score_{\nElems}) \in \R^{\elems}$ for the \emph{profile} of propensity scores $(\score_{\elem})_{\elem\in\elems}$ that comprises the basis of the recursive definition \eqref{eq:score}.
\endenv
\end{remark}

With all this in hand, given a propensity score profile $\score = (\score_{1},\dotsc,\score_{\nElems})\in\R^{\elems}$, the \acdef{NLC} rule is defined via the family of conditional selection probabilities
\begin{equation}
\label{eq:NLC}
\tag{NLC}
\choice[\class\atlvl{\lvl} \vert \class\atlvl{\lvl-1}](\score)
	= \frac
		{\exp(\score_{\class\atlvl{\lvl}} / \temp\atlvl{\lvl})}
		{\exp(\score_{\class\atlvl{\lvl-1}} / \temp\atlvl{\lvl})}
\end{equation}
where:
\begin{enumerate}
[left=0pt,itemsep=0pt,topsep=0pt,parsep=\smallskipamount]
\item
$\class\atlvl{\lvl} \in \classes\atlvl{\lvl}$ and $\class\atlvl{\lvl-1} \in \classes\atlvl{\lvl-1}$ is a child\,/\,parent pair of similarity classes of $\struct$.
\item
$\temp\atlvl{1} \geq \dotsm \geq \temp\atlvl{\nLvls} > 0$ is a nonincreasing sequence of uncertainty parameters (indicating a higher uncertainty level for coarser attributes;
we discuss this later).
\end{enumerate}
In more detail, the choice of an alternative $\elem\in\elems$ under \eqref{eq:NLC} proceeds as follows:
given a propensity score $\score_{\elem}\in\R$ for each $\elem\in\elems$, every similarity class $\class\atlvl{\nLvls-1} \in \classes\atlvl{\nLvls-1}$ is assigned a propensity score via the recursive softmax expression \eqref{eq:score}, and the same procedure is applied inductively up to the root $\source$ of $\struct$.
Then, to select an alternative $\elem\in\elems$, the conditional logit choice rule \eqref{eq:NLC} proceeds in a top-down manner,
first by selecting a similarity class $\class\atlvl{1} \childof \class\atlvl{0} \equiv \source$,
then by selecting a child $\class\atlvl{2} \childof \class\atlvl{1}$ of $\class\atlvl{1}$,
and so on until reaching a leaf $\{\elem\} \equiv \leafpath{\class}$ of $\struct$.

Equivalently, unrolling \eqref{eq:NLC} over the lineage $\lineage{\class}$ of a target class $\class\in\classes\atlvl{\lvl}$, we obtain the expression
\begin{equation}
\label{eq:NLC-tot}
\choice[\class](\score)
	= \prod\nolimits_{\lvlalt=1}^{\lvl}
		\frac{\exp(\score_{\class\atlvl{\lvlalt}}/\temp\atlvl{\lvlalt})}{\exp(\score_{\class\atlvl{\lvlalt-1}}/\temp\atlvl{\lvlalt})}
\end{equation}
for the total probability of selecting class $\class$ under the propensity score profile $\score\in\R^{\elems}$.
Clearly, \eqref{eq:NLC} and \eqref{eq:NLC-tot} are mathematically equivalent, so we will refer to either one as the definition of the \acl{NLC} rule.

\subsection{The \acl{NIWE}}
\label{sec:NIWE}

The second key ingredient of our method is how to estimate the costs of alternatives that were not chosen under \eqref{eq:NLC}.
To that end, given a cost vector $\cost\in[0,1]^{\elems}$ and a mixed strategy $\strat\in\simplex(\elems)$ with full support, a standard way to do this is via the \acl{IWE} \citep{BCB12,LS20}
\begin{equation}
\label{eq:IWE}
\tag{IWE}
\est\cost_{\elem}
	= \frac{\oneof{\elem = \est\elem}}{\strat_{\elem}} \cost_{\elem}
\end{equation}

where $\est\elem \sim \strat$ is the (random) element of $\elems$ chosen under $\strat$.

This estimator enjoys the following important properties:
\begin{enumerate}
[\itshape a\upshape)]
\item
It is non-negative.
\item
It is \emph{unbiased}, \ie
\begin{equation}
\exof{\est\cost_{\elem}}
	= \cost_{\elem}
	\quad
	\text{for all $\elem\in\elems$}.
\end{equation}
\item
Its \emph{importance-weighted mean square} is bounded as
\begin{equation}
\exof*{\insum_{\elem\in\elems} \strat_{\elem} \est\cost_{\elem}^{2}}
	\leq \nElems
\end{equation}
\end{enumerate}
This trifecta of properties plays a key role in establishing the no-regret guarantees of the vanilla \acl{EW} algorithm \cite{ACBF02,LW94,Vov90};
at the same time however, \eqref{eq:IWE} fails to take into account any side information provided by similarities between different elements of $\elems$.
This is perhaps most easily seen in the original red bus\,/\,blue bus paradox:
if the commuter takes a red bus, the observed utility would be immediately translateable to the blue bus (and vice versa).
However, \eqref{eq:IWE} is treating the red and blue buses as unrelated, so $\est\cost_{\bluebus}$ is not updated under \eqref{eq:IWE}, even though $\cost_{\bluebus} = \cost_{\redbus}$ by default.

To exploit this type of similarities, we introduce below a layered estimator that shadows the step-by-step selection process of \eqref{eq:NLC}.
To define it, let $\strat\in\simplex(\elems)$ be a mixed strategy on $\elems$ with full support, and assume that an element $\est\elem \in \elems$ is selected progressively according to $\strat$ as in the case of \eqref{eq:NLC}:%
\footnote{To clarify, this process adheres to the ``nested'' part of \eqref{eq:NLC};
the conditional probabilities $\strat_{\classalt \vert \class}$ may of course differ.}
First, the learner chooses a similarity class $\est\class\atlvl{1} \in \classes\atlvl{1}$ with probability $\probof{\est\class\atlvl{1} = \class\atlvl{1}} = \strat_{\class\atlvl{1}}$;
subsequently, conditioned on the choice of $\est\class\atlvl{1}$, a class $\est\class\atlvl{2} \childof \est\class\atlvl{1}$ is selected with probability $\probof{\est\class\atlvl{2} = \class\atlvl{2} \vert \est\class\atlvl{1}} = \strat_{\class\atlvl{2} \vert \est\class\atlvl{1}}$,
and the process repeats until reaching a leaf $\est\class\atlvl{\nLvls} = \{\est\elem\}$ of $\struct$ (at which point the selection procedure terminates and returns $\est\elem$).
Then, given
a loss profile $\incr\in [0, +\infty)^{\struct}$
and a mixed strategy $\strat\in\simplex(\elems)$,
the \acdef{NIWE} is defined for all $\lvl=1,\dotsc,\nLvls$ as

\begin{equation}
\label{eq:NIWE}
\tag{NIWE}
\incrmodel_{\class\atlvl{\lvl}}
	= \frac{\oneof[\big]{\class\atlvl{\lvl} = \est\class\atlvl{\lvl},\dotsc,\class\atlvl{1} = \est\class\atlvl{1}}}
			{\strat_{\class\atlvl{\lvl} \vert \class\atlvl{\lvl-1}} \!\dotsm \strat_{\class\atlvl{2} \vert \class\atlvl{1}} \strat_{\class\atlvl{1}}}
			\incr_{\class\atlvl{\lvl}}
\end{equation}

where the chain of categorical random variables $\elems \equiv \est\class\atlvl{0} \parentof \est\class\atlvl{1} \parentof \dotsm \parentof \est\class\atlvl{\nLvls} = \{\est\elem\}$ is drawn according to $\strat\in\simplex(\elems)$ as outlined above.%
\footnote{The indicator in \eqref{eq:NIWE} is assumed to take precedence over $\strat_{\class\atlvl{\lvlalt} \vert \class\atlvl{\lvlalt-1}}$,
\ie $\basemodelcost_{\class\atlvl{\lvl}} = 0$ if $\class\atlvl{\lvlalt} \neq \est\class\atlvl{\lvlalt}$ for some $\lvlalt=1,\dotsc,\lvl$.}

This estimator will play a central part in our analysis, so some remarks are in order.
First and foremost, the non-nested estimator \eqref{eq:IWE} is recovered as a special case of \eqref{eq:NIWE} when there are no similarity attributes on $\elems$ (\ie $\nLvls = 1$).
Second, in a bona fide nested model, we should note that $\basemodelcost_{\class\atlvl{\lvl}}$ is $\est\class\atlvl{\lvl}$-measurable but \emph{not} $\est\class\atlvl{\lvl-1}$-measurable:
this property has no analogue in \eqref{eq:IWE}, and it is an intrinsic feature of the step-by-step selection process underlying \eqref{eq:NIWE}.
Third, it is also important to note that \eqref{eq:NIWE} concerns the idiosyncratic losses of each chosen class, \emph{not} the base costs $\cost_{\elem}$ of each alternative $\elem\in\elems$.
This distinction is again redundant in the non-nested case, but it leads to a distinct estimator for $\cost_{\elem}$ in nested environments, namely
\begin{equation}
\label{eq:NIWE-cost}
\model_{\elem}
	= \insum_{\class \ni \elem} \incrmodel_{\class}
	\quad
	\text{for all $\elem\in\elems$}.
\end{equation}
In particular, in the red bus\,/\,blue bus paradox, this means that an observation for the class ``\texttt{bus}'' automatically updates both $\est\cost_{\redbus}$ and $\est\cost_{\bluebus}$, thus overcoming one of the main drawbacks of \eqref{eq:IWE} when facing irrelevant alternatives.

To complete the comparison with the non-nested setting,
we summarize below the most important properties of the layered estimator \eqref{eq:NIWE}:

\begin{restatable}{proposition}{NIWE}
\label{prop:NIWE}
Let $\struct = \coprod_{\lvl=1}^{\nLvls} \classes\atlvl{\lvl}$ be a similarity structure on $\elems$.
Then,
given a mixed strategy $\strat\in\simplex(\elems)$ and a vector of cost increments $\incr\in\R^{\struct}$ as per \eqref{eq:range},
the estimator \eqref{eq:NIWE} satisfies the following:
\begin{enumerate}
\item
It is unbiased:
\begin{equation}
\label{eq:unbiased}
\exof*{\incrmodel_{\class}}
	= \incr_{\class}
	\quad
	\text{for all $\class\in\classes$}.
\end{equation}
\item
It enjoys the importance-weighted mean-square bound
\begin{equation}
\label{eq:varbound-base}
\exof*{\strat_{\class} \incrmodel_{\class}^{2}}
	\leq \range_{\class}^{2}
	\quad
	\text{for all $\class\in\classes$}.
\end{equation}
\end{enumerate}
Accordingly, the loss estimator \eqref{eq:NIWE-cost} is itself unbiased and enjoys the bound
\begin{equation}
\label{eq:varbound-cost}
\exof*{\insum_{\elem\in\elems} \strat_{\elem} \model_{\elem}^{2}}
	\leq \nEff
\end{equation}
where $\nEff$ is defined as
\begin{equation}
\label{eq:n-eff}
\sqrt{\nEff}
	= \insum_{\lvl=1}^{\nLvls} \sqrt{\nClasses\atlvl{\lvl}}\bar{\range}\atlvl{\lvl}
\end{equation}
with
$\nClasses\atlvl{\lvl} = \abs{\classes\atlvl{\lvl}}$ denoting the number of classes of attribute $\classes\atlvl{\lvl}$,
and
\begin{equation}
\label{eq:range-mean}
\bar\range\atlvl{\lvl}
	= \sqrt{
		\frac{1}{\nClasses\atlvl{\lvl}}
		\insum_{\class\atlvl{\lvl} \in \classes\atlvl{\lvl}} \range_{\class\atlvl{\lvl}}^{2}
		}
\end{equation}
denoting the ``root-mean-square'' range of all classes in $\classes\atlvl{\lvl}$.
\end{restatable}

Of course, \cref{prop:NIWE} yields the standard properties of \eqref{eq:IWE} as a special case when $\nLvls = 1$ (in which case there are no similarities to exploit between alternatives).
To streamline our presentation, we prove this result in \cref{app:aux}.

\subsection{The \acl{NEW} algorithm}
\label{sec:algo}

We are finally in a position to present the \acdef{NEW} algorithm in detail.
Building on the original \acl{EW} blueprint \cite{LW94,ACBF02,Vov90}, the main steps of the \ac{NEW} algorithm can be summed up as follows:
\begin{enumerate}

\item
For each stage $\run=\running$, the learner maintains and updates a propensity score profile $\curr[\dstate] \in \R^{\elems}$.

\item
The learner selects an action $\curr[\elem] \in \elems$ based on the \acl{NLC} rule $\curr[\elem] \sim \choice(\curr[\learn]\curr[\dstate])$ where $\curr[\learn] \geq 0$ is the method's \emph{learning rate} and $\choice$ is given by \eqref{eq:NLC}.

\item
The learner incurs $\incr_{\class,\run}$ for each class $\class\ni\curr[\elem]$ and constructs a model
$\curr[\model]$ of the cost vector $\curr[\cost]$ of stage $\run$ via \eqref{eq:NIWE}.

\item
The learner updates their propensity score profile based on $\curr[\model]$ and the process repeats.
\end{enumerate}
For a presentation of the algorithm in pseudocode form, see \cref{alg:NEW};
the tuning of the method's uncertainty parameters $\temp\atlvl{1} \geq \dotsc \geq \temp\atlvl{L} > 0$ and the learning rate $\curr[\learn]$ is discussed in the next section, where we undertake the analysis of the \ac{NEW} algorithm.


\begin{algorithm}[tbp]
\small
\caption{\Acf{NEW}}

\begin{algorithmic}[1]
\setlength{\abovedisplayskip}{\smallskipamount}
\setlength{\belowdisplayskip}{\smallskipamount}
\addtolength{\baselineskip}{1pt}

\Require
	set of alternatives $\elems$;
	attribute structure $\struct = \coprod_{\lvl=1}^{\nLvls} \classes\atlvl{\lvl}$

\Statex
\hspace{-\parindent}
\textbf{Params:}
	uncertainty levels $\temp\atlvl{1},\dotsc,\temp\atlvl{\nLvls} > 0$;
	learning rate $\curr[\learn]\geq0$

\Statex
\hspace{-\parindent}
\textbf{Input:}
	sequence of class costs $\curr[\incr] \in [0,1]^{\struct}$, $\run=\running$
    
\smallskip
\hrule
\smallskip

\State
\textbf{initialize}
	$\dstate \gets 0 \in \R^{\elems}$,
	$\class\atlvl{0} = \elems$
	\Comment{initialization}

\For{$\run=\running$}
	\Comment{scoring phase}

	\For{$\lvl=\nLvls-1,\dotsc,0$ and \textbf{for all} $\parent \in \classes\atlvl{\lvl}$}
	
		\State
		$\dstate_{\parent}	\gets \temp\atlvl{\lvl+1} \log\sum_{\child \childof \parent} \exp(\dstate_{\child} / \temp\atlvl{\lvl+1})$
		\Comment{as per \eqref{eq:score}}

	\State
	\textbf{set}
		$\incrmodel_{\class} \gets 0$
		\Comment{baseline guess}
	
	\EndFor

	\For{$\lvl=1,\dotsc,\nLvls$}
		\Comment{selection phase}
	
		\State
		\textbf{select}
			class $\class\atlvl{\lvl} \childof \class\atlvl{\lvl-1}$
			\Comment{class choice}
		\begin{equation*}
		\explain{\upshape(\ref{eq:NLC})}
		\class\atlvl{\lvl}
			\sim \state_{\class\atlvl{\lvl} \vert \class\atlvl{\lvl-1}}
			= \frac
				{\exp(\curr[\learn] \dstate_{\class\atlvl{\lvl}} / \temp\atlvl{\lvl})}
				{\exp(\curr[\learn] \dstate_{\class\atlvl{\lvl-1}} / \temp\atlvl{\lvl})}
			\hspace{-2em}
		\end{equation*}

		\State
		\textbf{get}
			$\incr_{\class\atlvl{\lvl},\run}$
			\Comment{intra-class cost}
		
		\State
		\textbf{set}
			$\dis\incrmodel_{\class\atlvl{\lvl}} \gets \incrmodel_{\class\atlvl{\lvl}} + \frac{\incr_{\class\atlvl{\lvl},\run}}{\state_{\class\atlvl{\lvl} \vert \class\atlvl{\lvl-1}} \!\!\dotsm \state_{\class\atlvl{1} \vert \class\atlvl{0}}}$
			\Comment{\eqref{eq:NIWE}}
	
	\EndFor
	
	\State
	\textbf{set}
		$\model_{\elem} \gets \sum_{\class\ni\elem} \incrmodel_{\class}$ for all $\elem\in\elems$
			\Comment{loss model}
	\State
	\textbf{set}
		$\dstate \gets \dstate - \model$
		\Comment{update propensities}

\EndFor

\end{algorithmic}
\label{alg:NEW}
\end{algorithm}


\section{Analysis and results}
\label{sec:results}

We are now in a position to state and discuss our main regret guarantees for the \ac{NEW} algorithm.
These are as follows:

\begin{restatable}{theorem}{NEW}
\label{thm:NEW}
Suppose that \cref{alg:NEW} is run with a non-increasing learning rate $\curr[\learn] > 0$ and uncertainty parameters $\temp\atlvl{1} \geq \dotsm \geq \temp\atlvl{\nLvls} > 0$ against a sequence of cost vectors $\curr[\cost] \in [0,1]^{\elems}$, $\run=\running$, as per \eqref{eq:cost-elem}.
Then, for all $\test\in\simplex(\elems)$, the learner enjoys the regret bound
\begin{align}
\label{eq:reg-NEW}
\exof{\reg_{\test}(\nRuns)}
	&\leq \frac{\hrange}{\afterlast[\learn]}
		+ \frac{\nEff}{2\temp\atlvl{\nLvls}} \sum_{\run=\start}^{\nRuns} \curr[\learn]
\end{align}
with
$\nEff$ given by \eqref{eq:n-eff}
and
$\hrange \equiv \hrange(\temp\atlvl{1},\dotsc,\temp\atlvl{\nLvls})$ defined by setting $\score = 0$ in \eqref{eq:score} and taking $\hrange = \score_{\source}$, \ie
\begin{equation}
\label{eq:hrange}
\hrange
	= \log\bracks*{
		\sum_{\class\atlvl{1} \childof \class\atlvl{0}}
		\bracks*{
			\sum_{\class\atlvl{2} \childof \class\atlvl{1}}
				\!\dotsi
				\bracks*{
					\sum_{\class\atlvl{\nLvls}\childof\class\atlvl{\nLvls-1}}
						\!\!\!\!1
				}^{\frac{\temp\atlvl{\nLvls}}{\temp\atlvl{\nLvls-1}}}
			\!\!\!\!\!\dotsi\,
		}^{\frac{\temp\atlvl{2}}{\temp\atlvl{1}}}
	}^{\temp\atlvl{1}}
\end{equation}
In particular, if \cref{alg:NEW} is run with $\temp\atlvl{1} = \dotsm = \temp\atlvl{\nLvls} = \sqrt{\nEff/2}$ and $\curr[\learn] = \sqrt{\log\nElems/(2\run)}$, we have
\begin{equation}
\label{eq:reg-NEW-tuned}
\exof{\reg_{\test}(\nRuns)}
	\leq 2 \sqrt{\nEff \log\nElems \cdot \nRuns}.
\end{equation}
\end{restatable}

\cref{thm:NEW} is our main regret guarantee for \ac{NEW} so, before discussing its proof (which we carry out in detail in \cref{app:entropy,app:aux,app:regret}), some remarks are in order.

The first thing of note is  the comparison to the corresponding bound for \ac{EXP3}, namely
\begin{equation}
\label{eq:reg-EXP3}
\exof{\reg_{\test}(\nRuns)}
	\leq 2\sqrt{\nElems \log\nElems \cdot \nRuns}.
\end{equation}
This shows that the guarantees of \ac{NEW} and \ac{EXP3} differ by a factor of%
\footnote{Depending on the source, the bound \eqref{eq:reg-EXP3} may differ up to a factor of $\sqrt{2}$, compare for example \citep[Corollary 4.2]{SS11} and \citep[Theorem 11.2]{LS20}.
This factor is due to the fact that \eqref{eq:reg-EXP3} is usually stated for a known horizon $\nRuns$ (which saves a factor of $\sqrt{2}$ relative to anytime algorithms).
Ceteris paribus, the bound \eqref{eq:reg-NEW-tuned} can be sharpened by the same factor, but we omit the details.}
\begin{equation}
\label{eq:eff}
\eff
	= \sqrt{\nElems/\nEff},
\end{equation}
which, for reasons that become clear below,
we call the \acdef{PoAf}.

Since the variabilities of the idiosyncratic losses within each attribute have been normalized to $1$ (recall the relevant discussion in \cref{sec:sequence}), Hölder's inequality trivially gives $\nEff \leq \nElems$, no matter the underlying similarity structure.
Of course, if there are no similarities to exploit ($\nLvls = 1$), we get $\nEff = \nElems$, in which case the two bounds coincide ($\eff=1$).

At the other extreme, suppose we have a red bus\,/\,blue bus type of problem with, say, $\nClasses\atlvl{1} = 2$ similarity classes, $\nClasses\atlvl{2} = 100$ alternatives per class, and a negligible intra-class loss differential ($\range\atlvl{2} \approx 0$).
In this case, \ac{EXP3} would have to wrestle with $\nElems = \nClasses\atlvl{1} \nClasses\atlvl{2} = 200$ alternatives, while \ac{NEW} would only need to discriminate between $\nEff \approx \nClasses\atlvl{1} = 2$ alternatives, leading to an improvement by a factor of $\eff \approx 10$ in terms of regret guarantee.
Thus, even though the red bus\,/\,blue bus paradox could entangle \ac{EXP3} and cause the algorithm to accrue significant regret over time, this is no longer the case under the \ac{NEW} method;
we also explore this issue numerically in \cref{sec:numerics}.

As another example, suppose that each non-terminal class in $\struct$ has $\nChildren$ children and the variability of the idiosyncratic losses likewise scales down by a factor of $\nChildren$ per attribute.
In this case, a straightforward calculation shows that $\nEff$ scales as $\Theta(\nChildren)$, so the gain in efficiency would be of the order of $\eff = \sqrt{\nElems / \nEff} = \Theta(\nChildren^{(\nLvls-1)/2})$,
\ie polynomial in $\nChildren$ and exponential in $\nLvls$.
This gain in performance can become especially pronounced
when there is a very large number of atlernatives organized in categories and subcategories of geometrically decreasing impact on the end cost of each alternative.
We explore this issue in practical scenarios in \cref{sec:numerics,app:numerics}.

Finally, we should also note that the parameters of \ac{NEW} have been tuned so as to facilitate the comparison with \ac{EXP3}.
This tuning is calibrated for the case where $\struct$ is fully symmetric, \ie all subcategories of a given attribute have the same number of children.
Otherwise, in full generality, the tuning of the algorithm's uncertainty levels would boil down to a transcedental equation involving the nested term $\hrange(\temp\atlvl{1},\dotsc,\temp\atlvl{\nLvls})$ of \eqref{eq:reg-NEW}.
This can be done efficiently offline via a line search, but since the result would be structure-dependent, we do not undertake this analysis here.

\para{Proof outline of \cref{thm:NEW}}

The detailed proof of \cref{thm:NEW} is quite lengthy, so we defer it to \cref{app:entropy,app:aux,app:regret} and only sketch here the main ideas.

The first basic step is to derive a suitable ``potential function'' that can be used to track the evolution of the \ac{NEW} policy relative to the benchmark $\test \in \simplex(\elems)$.
The main ingredient of this potential is the ``nested'' entropy function
\begin{equation}
\label{eq:hreg}
\hreg(\strat)
	= \insum_{\lvlalt=0}^{\nLvls}
		\diff\atlvl{\lvlalt} \insum_{\class\atlvl{\lvlalt} \in \classes\atlvl{\lvlalt}}
		\strat_{\class\atlvl{\lvlalt}} \log\strat_{\class\atlvl{\lvlalt}},
\end{equation}
where $\diff\atlvl{\lvlalt} = \temp\atlvl{\lvlalt} - \temp\atlvl{\lvlalt+1}$ for all $\lvlalt = 1,\dotsc,\nLvls$ (with $\temp\atlvl{\nLvls+1} = 0$ by convention).%
\footnote{In the non-nested case, \eqref{eq:hreg} boils down to the standard (negative) entropy $\hreg(\strat) = \sum_{\elem} \strat_{\elem} \log\strat_{\elem}$.
However, the inverse problem of deriving the ``correct'' form of $\hreg$ in a nested environment involves a technical leap of faith and a fair degree of trial-and-error.}
As we show in \cref{prop:nest2cond} in \cref{app:entropy}, the ``tiers'' of $\hreg$ can be unrolled to give the ``non-tiered'' recursive representation
\begin{equation}
\hreg(\strat)
	= \insum_{\class\in\struct} \hreg(\strat \vert \class)
\end{equation}
where $\hreg(\strat \vert \class) = \temp\atlvl{\lvl+1} \sum_{\classalt \childof \class} \strat_{\classalt} \log (\strat_{\classalt} / \strat_{\class})$ denotes the ``conditional'' entropy of $\strat$ relative to class $\class \in \classes\atlvl{\lvl}$.
Then, by means of this decomposition and a delicate backwards induction argument, we show in \cref{prop:mirror} that
\begin{enumerate*}
[\itshape a\upshape)]
\item
the recursively defined propensity score $\score_{\source}$ of $\source$ can be expressed \emph{non-recursively} as $\score_{\source} = \argmax_{\strat\in\simplex(\elems)} \{ \braket{\score}{\strat} - \hreg(\strat) \}$;
and
\item
that the choice rule \eqref{eq:NLC} can be expressed itself as
\end{enumerate*}
\begin{equation}
\label{eq:potential}
\choice[\elem](\score)
	= \frac{\pd\score_{\elems}}{\pd\score_{\elem}}
	\quad
	\text{for all $\score\in\R^{\elems}$, $\elem\in\elems$}.
\end{equation}

This representation of \eqref{eq:NLC} provides the first building block of our proof because, by Danskin's theorem \cite{Ber97}, it allows us to rewrite \cref{alg:NEW} in more concise form as
\begin{equation}
\label{eq:NEW}
\tag{\acs{NEW}}
\begin{gathered}
\next[\dstate]
	= \curr[\dstate]
		- \curr[\model]
	\\
\next
	= \argmax_{\strat\in\simplex(\elems)}
		\{ \braket{\next[\learn] \next[\dstate]}{\strat} - \hreg(\strat) \}
\end{gathered}
\end{equation}
with $\curr[\model]$ given by \eqref{eq:NIWE-cost} appplied to $\strat \gets \curr$.
Importantly, this shows that the \ac{NEW} algorithm is an instance of the well-known \acdef{FTRL} algorithmic framework \cite{SSS06,SS11}.
Albeit interesting, this observation is not particularly helpful in itself because
there is no universal, ``regularizer-agnostic'' analysis giving optimal (or near-optimal) regret rates for \ac{FTRL} with bandit/partial information.%
\footnote{For the analysis of specific versions of \ac{FTRL} with non-entropic regularizers, \cf \cite{ABL11,ZS19} and references therein.}
Nonetheless, by adapting a series of techniques that are used in the analysis of \ac{FTRL} algorithms, we show in \cref{app:regret} that the iterates of \eqref{eq:NEW} satisfy the ``energy inequality''
\begin{align}
\braket{\model_{\run}}{\curr - \test}
	&\leq \curr[\energy] - \energy_{\run+1}
		+ \frac{1}{\curr[\learn]} \fench(\curr,\curr[\learn]\next[\dstate])
	\notag\\
	&+ \parens{\next[\learn]^{-1} - \curr[\learn]^{-1}} \bracks{\hreg(\test) - \min\hreg}
\end{align}
where
$\curr[\model]$ is the \acl{NIWE} \eqref{eq:NIWE-cost} for the cost vector encountered $\curr[\cost]$, and we have set
\begin{equation}
\fench(\strat,\score)
	= \hreg(\strat) + \score_{\source} - \braket{\score}{\strat}
\end{equation}
and
$\curr[\energy] = \curr[\learn]^{-1} \fench(\test,\curr[\learn]\curr[\dstate])$.

Then, by \cref{prop:NIWE}, we obtain:

\begin{restatable}{proposition}{template}
\label{prop:template}
The \ac{NEW} algorithm enjoys the bound
\begin{equation}
\label{eq:template}
\exof{\reg_{\test}(\nRuns)}
	\leq \frac{\hrange}{\afterlast[\learn]}
		+ \sum_{\run=\start}^{\nRuns} \frac{\exof{\fench(\curr,\curr[\learn]\next[\dstate])}}{\curr[\learn]}.
\end{equation}
\end{restatable}

\cref{prop:template} provides the first half of the bound \eqref{eq:reg-NEW}, with the precise form of $\hrange$ derived in \cref{lem:hmin}.
The second half of \eqref{eq:reg-NEW} revolves around the term $\exof{\fench(\curr,\curr[\learn]\next[\dstate])}$ and boils down to estimating how propensity scores are back-propagated along $\struct$.
In particular, the main difficulty is to bound the difference $\new\score_{\source} - \score_{\source}$ in the propensity score of the root node $\source$ of $\struct$ when the underlying score profile $\score\in\R^{\elems}$ is incremented to $\new\score = \score + \dvec$ for some $\dvec\in\R^{\elems}$.

A first bound that can be obtained by convex analysis arguments is $\abs{\new\score_{\source} - \score_{\source}} \leq \braket{\score}{\choice(\score)} + \supnorm{\dvec}^{2}$;
however, because the increments of \eqref{eq:NEW} are unbounded in norm, this global bound is far too lax for our puposes.
A similar issue arises in the analysis of \ac{EXP3}, and is circumvented by deriving a bound for the log-sum-exp function using the identity $\exp(x) \leq 1 + x + x^{2}/2$ for $x\leq0$ and the fact that the estimator \eqref{eq:IWE} is non-negative \citep{LS20,SS11,CBL06}.
Extending this idea to nested environments is a very delicate affair, because each tier in $\struct$ introduces an additional layer of error propagation in the increments $\next[\dstate] - \curr[\dstate]$.
However, by a series of inductive arguments that traverse $\struct$ both forward and backward, we are able to show the bound
\begin{equation}
\label{eq:hconj-diff}
\new\score_{\source} - \score_{\source}
	\leq \braket{\score}{\choice(\score)}
		+ \frac{1}{2\temp\atlvl{\nLvls}}
			\sum_{\lvl=1}^{\nLvls}
			\sum_{\class\atlvl{\lvl}\in\classes\atlvl{\lvl}}
			\choice[\class\atlvl{\lvl}](\score) \incr_{\class\atlvl{\lvl}}^{2}
\end{equation}
which, after taking expecations and using the bounds of \cref{prop:NIWE}, finally yields the pseudo-regret bound \eqref{eq:reg-NEW}.


\section{Numerical experiments}
\label{sec:numerics}

\begin{figure}[t]
    \centering
    \includegraphics[height=4cm]{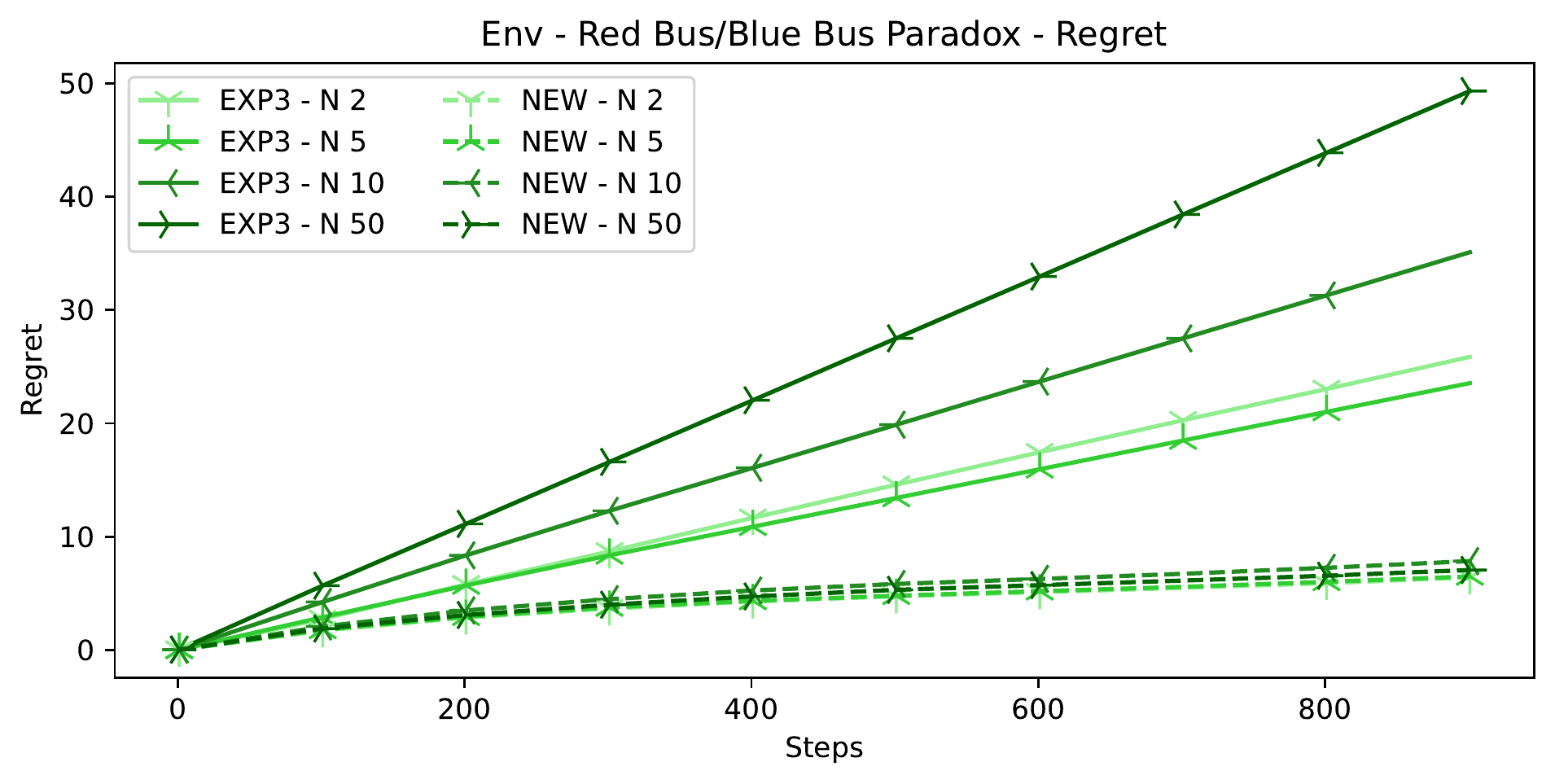}
    \caption{Regret of \ac{EXP3} and \ac{NEW} in the red bus\,/\,blue bus problem with different numbers of buses.}
    \label{fig:bbrb_regret}
    \vspace{-1ex}
\end{figure}

\begin{figure}[t]
    \centering
    \includegraphics[height=4cm]{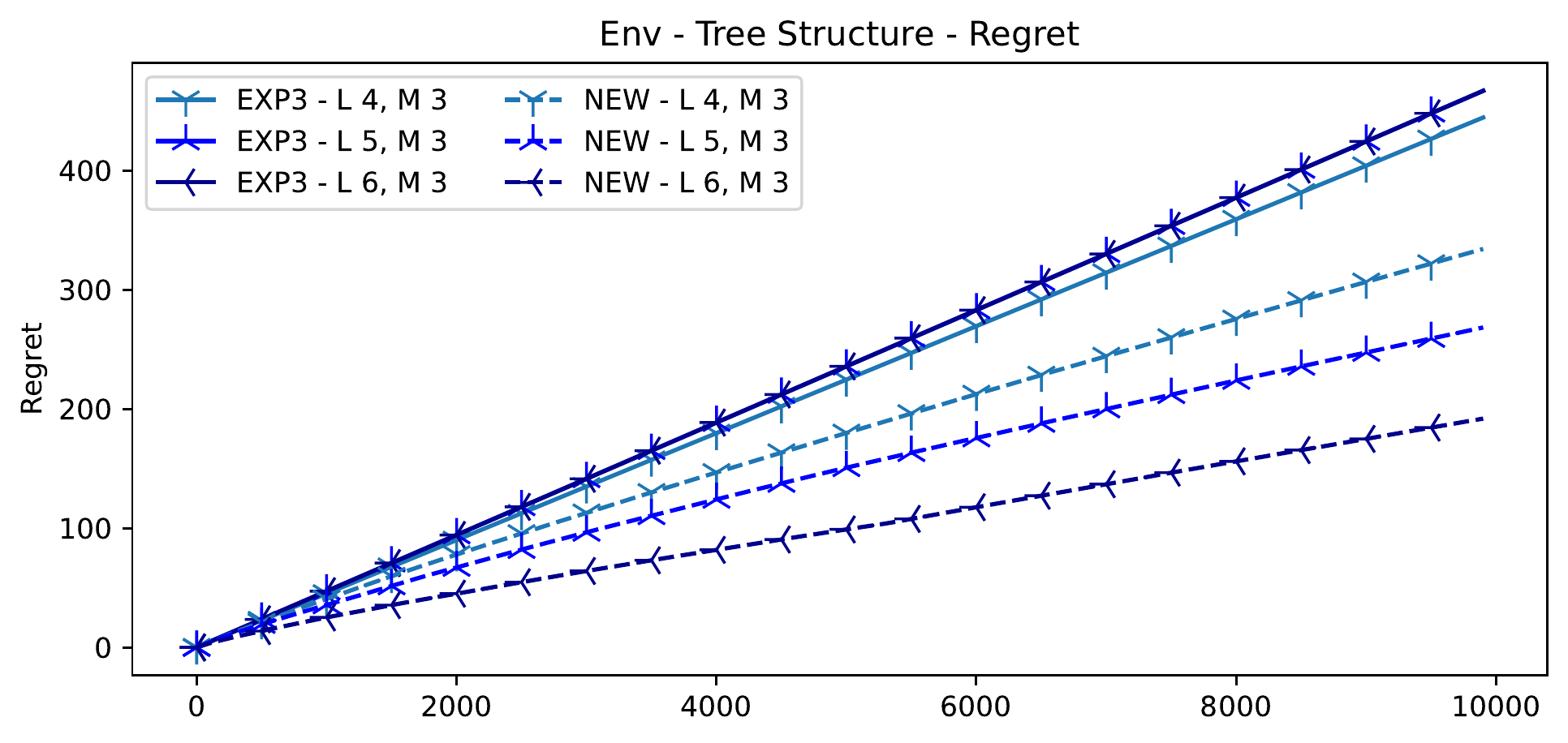}
    \caption{Regret of \ac{EXP3} and \ac{NEW} in a tree environment with different values of levels $L$ and classes per level $M$}
    \label{fig:tree_regret_depth}
    \vspace{-1ex}
\end{figure}

\begin{figure}[t]
    \centering
    \includegraphics[height=4cm]{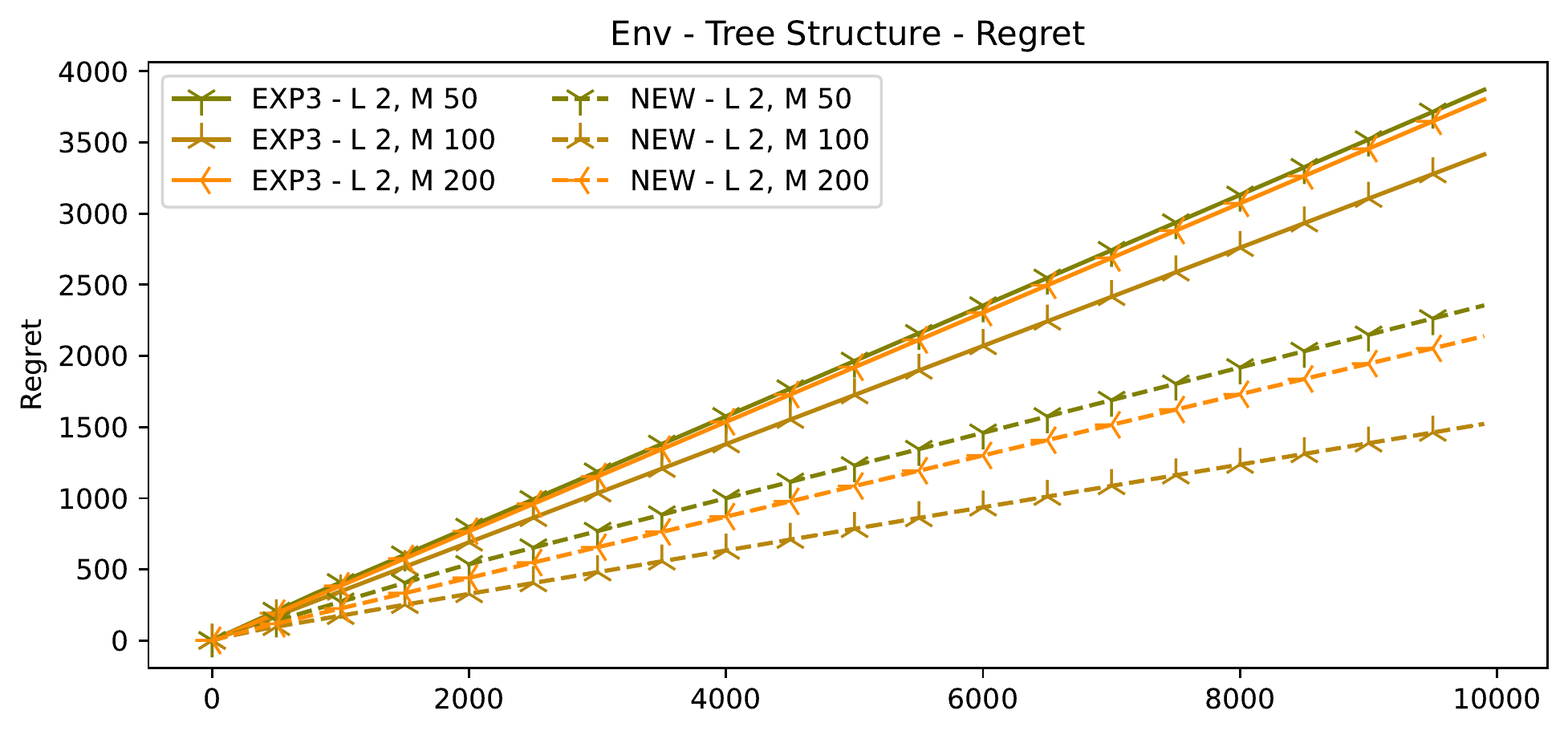}
    \caption{Regret of \ac{EXP3} and \ac{NEW} in a tree environment with different values of levels $L$ and classes per level $M$}
    \label{fig:tree_regret_breadth}
    \vspace{-1ex}
\end{figure}

In this section we present a series of numerical experiments designed to test the efficiency of our method compared to \ac{EXP3}.
We use a synthetic environment where we simulate nested similarity partitions with trees. While \ac{NEW} exploits the similarity structure by making forward/backward passes through the associated tree with its logit choice rule \eqref{eq:NLC}, \ac{EXP3} is simply run over the leaves of the tree, \ie $\elems$.
All experiment details (as well as additional results) are presented in \cref{sec:num-app}.
For every setting, we report the results of our experiments by plotting the average regret of each algorithm for $20$ seeds of randomly drawn losses.
The code to reproduce the experiments can be found at \url{https://github.com/criteo-research/Nested-Exponential-Weights}. 

\para{Benefits in the red bus/blue bus problem}

We consider here a variant of the red bus/blue bus problem with $N$ different buses (the original paradox has $N=2$).
In this experiment (see illustration in \cref{fig:blueredbus}, Appendix \ref{sec:num-bbrb}) we allow each bus to have non-zero intrinsic losses and illustrate in \cref{fig:bbrb_regret} how both algorithms perform when $N$ grows. 
We observe there that for all configurations \ac{NEW} achieves better regret than \ac{EXP3}. While both methods achieve sublinear regret, \ac{EXP3} requires far more steps to identify the best alternative as $N$ grows and suffers overall from worse regret while \ac{NEW} achieves similar regret and does not suffer as much from the number of irrelevant alternatives.
We provide additional plots in \cref{sec:num-bbrb} which show that \ac{NEW} performs consistently better than \ac{EXP3} when there exists a similarity structure allowing to efficiently update scores of classes that have very similar losses.

\para{Performance in general nested structures}
In this setting we generate symmetric trees and experiment with different values of number of levels $\nLvls$ and number of child per nodes $M = \vert \class\atlvl{\lvl} \vert$ for $\lvl=1, \dots, \nLvls$.
Specifically, in \cref{fig:tree_regret_depth} with a fixed $M$, we see that \ac{NEW} obtains better regret than \ac{EXP3} even when $\nLvls$ increases.
We provide variance plots for the experiments that generated the same performance on the plots in \ref{sec:num-tree} as well as additional visualisations.   
Finally, in \cref{fig:tree_regret_breadth}, we can see that for a shallow tree ($\nLvls=2$)  \ac{NEW} performs always better than \ac{EXP3}, even for high values of $M$. Indeed, when the number of children per nodes $M$ increases, the tree loses its ``factorized'' structure which also affects \ac{NEW} due to the less "structured" tree.
Thus, again, \ac{NEW} performs consistently better than \ac{EXP3} when it is possible to efficiently handle classes with similar losses. 

Overall, our experiments confirm that a learning algorithm based on \acl{NLC} can lead to significant benefits in problems with a high degree of similarity between alternatives.
This leaves open the question of whether a similar approach can be applied to structures with \emph{non-nested} attributes;
we defer this question to future work.


\appendix
\setcounter{remark}{0}
\numberwithin{equation}{section}		
\numberwithin{lemma}{section}		
\numberwithin{proposition}{section}		
\numberwithin{theorem}{section}		
\numberwithin{corollary}{section}		

\section{The nested entropy and its properties}
\label{app:entropy}

Our aim in this appendix is to prove the basic properties of the series of (negative) entropy functions that fuel the regret analysis of the \acf{NEW} algorithm.

To begin with, given a similarity structure $\struct$ on $\elems$ and a sequence of uncertainty parameters $\temp\atlvl{1} \geq \dotsm \geq \temp\atlvl{\nLvls} > 0$ (with $\temp\atlvl{\nLvls+1} = 0$ by convention), we define:
\begin{enumerate}
\addtolength{\itemsep}{\smallskipamount}

\item
The \emph{conditional entropy} of $\strat\in\simplex(\elems)$ relative to a target class $\class\in\classes\atlvl{\lvl}$:
\begin{equation}
\label{eq:entropy-cond}
\hreg(\strat \vert \class)
	= \temp\atlvl{\lvl+1} \sum_{\classalt \childof \class}
		\strat_{\classalt} \log \frac{\strat_{\classalt}}{\strat_{\class}}
	= \temp\atlvl{\lvl+1} \, \strat_{\class}
		\sum_{\classalt \childof \class}
			\strat_{\classalt \vert \class} \log\strat_{\classalt \vert \class}.
\end{equation}

\item
The \emph{nested entropy} of $\strat\in\simplex(\elems)$ relative to $\class\in\classes\atlvl{\lvl}$:
\begin{equation}
\label{eq:entropy-nest}
\hreg_{\class}(\strat)
	= \sum_{\lvlalt=\lvl}^{\nLvls} \diff\atlvl{\lvlalt}
		\;\sum_{\mathclap{\class\atlvl{\lvlalt} \desceq[\lvlalt] \class}}\;
			\strat_{\class\atlvl{\lvlalt}} \log\strat_{\class\atlvl{\lvlalt}}
\end{equation}
where $\diff\atlvl{\lvlalt} = \temp\atlvl{\lvlalt} - \temp\atlvl{\lvlalt+1}$ for all $\lvlalt = 1,\dotsc,\nLvls$.

\item
The \emph{restricted entropy} of $\strat\in\simplex(\elems)$ relative to $\class\in\classes\atlvl{\lvl}$:
\begin{equation}
\label{eq:entropy-restr}
\hreg_{\vert\class}(\strat)
	= \hreg_{\class}(\strat)
		+ \chi_{\simplex(\class)}(\strat)
	= \begin{cases}
		\hreg_{\class}(\strat)
			&\quad
			\text{if $\strat\in\simplex(\class)$},
		\\
		\infty
			&\quad
			\text{otherwise},
	\end{cases}
\end{equation}
where $\chi_{\simplex(\class)}$ denotes the (convex) characteristic function of $\simplex(\class)$, \ie 
$\chi_{\simplex(\class)}(\strat) = 0$ if $\strat\in\simplex(\class)$ and $\chi_{\simplex(\class)}(\strat) = \infty$ otherwise.
[Obviously, $\hreg_{\vert\class}(\strat) = \hreg_{\class}(\strat)$ whenever $\strat \in \simplex(\class)$.]
\end{enumerate}

\begin{remark}
As per our standard conventions, we are treating $\class$ interchangeably as a subset of $\elems$ or as an element of $\struct$;
by analogy, to avoid notational inflation, we are also viewing $\simplex(\class)$ as a subset of $\simplex(\elems)$ \textendash\ more precisely, a face thereof.
Finally, in all cases, the functions $\hreg(\strat\vert\class)$, $\hreg_{\class}(\strat)$ and $\hreg_{\vert\class}(\strat)$ are assumed to take the value $+\infty$ for $\strat \in \R^{\elems} \setminus \simplex(\elems)$.
\endenv
\end{remark}

\begin{remark}
For posterity, we also note that the nested and restricted entropy functions ($\hreg_{\class}(\strat)$ and $\hreg_{\vert\class}(\strat)$ respectively) are both convex \textendash\ though not necessarily \emph{strictly} convex \textendash\ over $\simplex(\elems)$.
This is a consequence of the fact that each summand $\strat_{\class} \log\strat_{\class}$ in \eqref{eq:entropy-nest} is convex in $\strat$ and that $\diff\atlvl{\lvlalt} = \temp\atlvl{\lvlalt} - \temp\atlvl{\lvlalt+1} \geq 0$ for all $\lvlalt=1,\dotsc,\nLvls$.
Of course, any two distributions $\strat,\stratalt\in\simplex(\elems)$ that assign the same probabilities to elements of $\class$ but not otherwise have $\hreg_{\class}(\strat) = \hreg_{\class}(\stratalt)$, so $\hreg_{\class}$ is \emph{not} strictly convex over $\simplex(\elems)$ if $\class \neq \elems$.
However, since the function $\sum_{\elem\in\class} \strat_{\elem} \log\strat_{\elem}$ is strictly convex over $\simplex(\class)$, it follows that $\hreg_{\class}$ \textendash\ and hence $\hreg_{\vert\class}$ \textendash\ \emph{is} strictly convex over $\simplex(\class)$.
\endenv
\end{remark}

Our main goal in the sequel will be to prove the following fundamental properties of the entropy functions defined above:
\smallskip

\begin{proposition}
\label{prop:nest2cond}
For all $\class \in \struct\atlvl{\lvl}$, $\lvl=1,\dotsc,\nLvls$, and for all $\strat\in\simplex(\elems)$, we have:
\begin{align}
\label{eq:nest2cond}
\hreg_{\class}(\strat)
	&= \sum_{\classalt\desceq\class} \hreg(\strat \vert \classalt)
	+ \temp\atlvl{\lvl} \, \strat_{\class} \log\strat_{\class}.
\intertext{Consequently, for all $\strat\in\simplex(\class)$, we have:}
\label{eq:restr2cond}
\hreg_{\vert\class}(\strat)
	&= \sum_{\classalt\desceq\class} \hreg(\strat \vert \classalt).
\end{align}
\end{proposition}

\begin{proposition}
\label{prop:mirror}
For all $\class \in \struct$ and all $\score\in\R^{\elems}$, we have:
\begin{enumerate}
\item
The recursively defined propensity score $\score_{\class}$ of $\class$ as given by \eqref{eq:score} can be equivalently expressed as
\begin{align}
\label{eq:ent-conj}
\score_{\class}
	&= \max_{\strat \in \simplex(\elems)}
		\braces{\braket{\score}{\strat} - \hreg_{\vert\class}(\strat)}
\intertext{%
\item
The conditional probability of choosing $\elem\in\elems$ given that $\class$ has been chosen under \eqref{eq:NLC} is given by
}
\label{eq:ent-mirror}
\choice[\elem \vert \class](\score)
	&= \frac{\pd\score_{\class}}{\pd\score_{\elem}}
	= \argmax_{\strat \in \simplex(\elems)}
		\braces{\braket{\score}{\strat} - \hreg_{\vert\class}(\strat)}
\end{align}
\end{enumerate}
\end{proposition}

These propositions will be the linchpin of the analysis to follow, so some remarks are in order:

\begin{remark}
Note here that the maximum in \eqref{eq:ent-conj} is taken over the \emph{restricted} entropy function $\hreg_{\vert\class}$, \emph{not} the nested entropy $\hreg_{\class}$.
This distinction will play a crucial role in the sequel;
in particular, since $\hreg_{\vert\class}$ is strictly convex over $\simplex(\class)$, it implies that the $\argmax$ in \eqref{eq:ent-mirror} is a singleton.
\endenv
\end{remark}

\begin{remark}
The first part of \cref{prop:mirror} can be rephrased more concisely (but otherwise equivalently) as
\begin{equation}
\label{eq:score2conj}
\score_{\class}
	= \hconj_{\vert\class}(\score)
\end{equation}
where
\begin{equation}
\label{eq:conj}
\hconj_{\vert\class}(\score)
	= \max_{\strat\in\simplex(\elems)}
		\braces{\braket{\score}{\strat} - \hreg_{\vert\class}(\strat)}
\end{equation}
denotes the convex conjugate of $\hreg_{\vert\class}$.
This interpretation is conceptually important because it spells out the precise functional dependence between the (primitive) propensity score profile $\score \in \R^{\elems}$ and the propensity scores $\score_{\class}$ that are propagated to higher-tier similarity classes $\class\in\struct$ via the recursive definition \eqref{eq:score}.
In particular, this observation leads to the recursive rule
\begin{equation}
\label{eq:conj-recursive}
\exp\parens*{\frac{\hconj_{\vert\class}(\score)}{\temp\atlvl{\lvl+1}}}
	= \sum_{\child\childof\class}
		\exp\parens*{\frac{\hconj_{\vert\child}(\score)}{\temp\atlvl{\lvl+1}}}
	\quad
	\text{for all $\class\in\classes\atlvl{\lvl}$, $\lvl=0,1,\dotsc,\nLvls-1$}.
\end{equation}
We will we use this representation freely in the sequel.
\endenv
\end{remark}

\begin{remark}
It is also worth noting that the propensity scores $\score_{\class}$, $\class\atlvl{\lvl} \in \classes\atlvl{\lvl}$, can also be seen as primitives for the arborescence $\alt\struct = \coprod_{\lvlalt=0}^{\lvl} \classes\atlvl{\lvlalt}$ obtained from $\struct$ by excising all (proper) descendants of $\classes\atlvl{\lvl}$.
Under this interpretation, the second part of \cref{prop:mirror} readily gives the more general expression
\begin{equation}
\label{eq:mirror-class}
\choice[\classalt \vert \class](\score)
	= \frac{\pd\score_{\class}}{\pd\score_{\classalt}}
	\quad
	\text{for all $\classalt \desceq \class$},
\end{equation}
where, in the \acl{RHS}, $\score_{\class}$ is to be construed as a function of $\score_{\classalt}$, defined recursively via \eqref{eq:score} applied to the truncated arborescence $\alt\struct$.
Even though we will not need this specific result, it is instructive to keep it in mind for the sequel.
\end{remark}

The rest of this appendix is devoted to the proofs of \cref{prop:nest2cond,prop:mirror}.

\begin{proof}[Proof of \cref{prop:nest2cond}]
Let $\lvl = \attof{\class}$, and fix some attribute label $\lvlalt > \lvl$.
We will proceed inductively by collecting all terms in \eqref{eq:nest2cond} associated to the attribute $\classes\atlvl{\lvlalt}$ and then summing everything together.
Indeed, we have:
\begin{subequations}
\label{eq:ent-nest1}
\begin{align}
\temp\atlvl{\lvlalt}
	\sum_{\classalt \desceq[\lvlalt] \class}
		\strat_{\classalt} \log\strat_{\classalt}
	&= \temp\atlvl{\lvlalt}
		\sum_{\class\atlvl{\lvlalt-1} \desceq[\lvlalt-1] \class}
			\bracks*{
				\sum_{\child \childof \class\atlvl{\lvlalt-1}}
					\strat_{\child} \log\strat_{\child}
				}
	\explain{collect attributes}
	\\
	&= \temp\atlvl{\lvlalt}
		\sum_{\class\atlvl{\lvlalt-1} \desceq[\lvlalt-1] \class}
			\bracks*{
				\sum_{\child \childof \class\atlvl{\lvlalt-1}}
					\strat_{\child \vert \class\atlvl{\lvlalt-1}} \strat_{\class\atlvl{\lvlalt-1}}
					\log(\strat_{\child \vert \class\atlvl{\lvlalt-1}} \strat_{\class\atlvl{\lvlalt-1}})
				}
	\explain{by definition}
	\\
	&= \temp\atlvl{\lvlalt}
		\sum_{\class\atlvl{\lvlalt-1} \desceq[\lvlalt-1] \class} 
			\bracks*{
				\sum_{\child \childof \class\atlvl{\lvlalt-1}}
					\strat_{\child \vert \class\atlvl{\lvlalt-1}} \strat_{\class\atlvl{\lvlalt-1}}
					\log\strat_{\child \vert \class\atlvl{\lvlalt-1}}
				}
	\label{eq:ent-nest1a}\\
	&+ \temp\atlvl{\lvlalt}
		\sum_{\class\atlvl{\lvlalt-1} \desceq[\lvlalt-1] \class}
			\bracks*{
				\sum_{\child \childof \class\atlvl{\lvlalt-1}}
					\strat_{\child \vert \class\atlvl{\lvlalt-1}} \strat_{\class\atlvl{\lvlalt-1}}
					\log\strat_{\class\atlvl{\lvlalt-1}}
				}
	\label{eq:ent-nest1b}
\end{align}
\end{subequations}
with the tacit understanding that any empty sum that appears above is taken equal to zero.

Now, by the definition of the nested entropy, we readily obtain that
\begin{subequations}
\label{eq:ent-nest2}
\begin{equation}
\label{eq:ent-nest2a}
\eqref{eq:ent-nest1a}
	= \sum_{\class\atlvl{\lvlalt-1} \desceq[\lvlalt-1] \class}
			\hreg(\strat \vert \class\atlvl{\lvlalt-1})
\end{equation}
whereas, by noting that $\sum_{\child \childof \class\atlvl{\lvlalt-1}} \strat_{\child \vert \class\atlvl{\lvlalt-1}} = 1$ (by the definition of conditional class choice probabilities), \cref{eq:ent-nest1b} becomes
\begin{equation}
\label{eq:ent-nest2b}
\eqref{eq:ent-nest1b}
	= \temp\atlvl{\lvlalt}
		\sum_{\class\atlvl{\lvlalt-1} \desceq[\lvlalt-1] \class}
			\strat_{\class\atlvl{\lvlalt-1}} \log\strat_{\class\atlvl{\lvlalt-1}}.
\end{equation}
\end{subequations}
Hence, combining \cref{eq:ent-nest1,eq:ent-nest2a,eq:ent-nest2b}, we get:
\begin{equation}
\label{eq:ent-nest3}
\temp\atlvl{\lvlalt}
	\sum_{\classalt \desceq[\lvlalt] \class}
		\strat_{\classalt} \log\strat_{\classalt}
	= \sum_{\class\atlvl{\lvlalt-1} \desceq[\lvlalt-1] \class}
			\hreg(\strat \vert \class\atlvl{\lvlalt-1})
	+ \temp\atlvl{\lvlalt}
		\sum_{\class\atlvl{\lvlalt-1} \desceq[\lvlalt-1] \class}
			\strat_{\class\atlvl{\lvlalt-1}} \log\strat_{\class\atlvl{\lvlalt-1}}.
\end{equation}

The above expression is our basic inductive step.
Indeed, summing \eqref{eq:ent-nest3} over all $\lvlalt = \nLvls,\dotsc,\lvl = \attof{\class}$, we obtain:
\begin{align}
\hreg_{\class}(\strat)
	&= \sum_{\lvlalt=\lvl}^{\nLvls}
		(\temp\atlvl{\lvlalt} - \temp\atlvl{\lvlalt+1})
		\;\sum_{\mathclap{\classalt \desceq[\lvlalt] \class}}\;
			\strat_{\classalt} \log\strat_{\classalt}
	\explain{by definition}
	\\
	&= \sum_{\lvlalt=\nLvls}^{\lvl+1}
		\bracks*{
			\temp\atlvl{\lvlalt} \sum_{\mathclap{\classalt \desceq[\lvlalt] \class}}
				\strat_{\classalt} \log\strat_{\classalt}
			- \temp\atlvl{\lvlalt+1} \sum_{\mathclap{\classalt \desceq[\lvlalt] \class}}
				\strat_{\classalt} \log\strat_{\classalt}
			}
	+ (\temp\atlvl{\lvl} - \temp\atlvl{\lvl+1}) \, \strat_{\class} \log\strat_{\class}
	\explain{isolate $\class$}
	\\
	&= \sum_{\lvlalt=\nLvls}^{\lvl+1}
		\bracks*{
			\sum_{\class\atlvl{\lvlalt-1} \desceq[\lvlalt-1] \class}
				\hreg(\strat \vert \class\atlvl{\lvlalt-1})
			+ \temp\atlvl{\lvlalt}
				\sum_{\class\atlvl{\lvlalt-1} \desceq[\lvlalt-1] \class}
					\strat_{\class\atlvl{\lvlalt-1}} \log\strat_{\class\atlvl{\lvlalt-1}}
			- \temp\atlvl{\lvlalt+1} \sum_{\mathclap{\classalt \desceq[\lvlalt] \class}}
				\strat_{\classalt} \log\strat_{\classalt}
			}
	\notag\\
	&\qquad
		+ (\temp\atlvl{\lvl} - \temp\atlvl{\lvl+1}) \, \strat_{\class} \log\strat_{\class}
	\explain{by \eqref{eq:ent-nest3}}
	\\
	&= \sum_{\lvlalt=\lvl}^{\nLvls-1}
		\sum_{\classalt \desceq[\lvlalt] \class} \hreg(\strat \vert \classalt)
	+ \temp\atlvl{\lvl} \,
			\strat_{\class} \log\strat_{\class}
	- \temp\atlvl{\nLvls+1}
		\;\sum_{\mathclap{\classalt \desceq[\nLvls] \class}}\;
			\strat_{\classalt} \log\strat_{\classalt}
\label{eq:ent-nest4}
\end{align}
with the last equality following by telescoping the terms involving $\temp\atlvl{\lvlalt}$.
Now, given that $\temp\atlvl{\nLvls+1} = 0$ by convention, the third sum above is zero.
Finally, since the conditional entropy of $\strat$ relative to any childless class is zero by definition, the first sum in \eqref{eq:ent-nest4} can be rewritten as $\sum_{\lvlalt=\lvl}^{\nLvls-1} \sum_{\classalt \desceq[\lvlalt] \class} \hreg(\strat \vert \classalt) = \sum_{\classalt \desceq \class} \hreg(\strat \vert \classalt)$, and our claim follows.

Finally, \eqref{eq:restr2cond} is a consequence of the fact that $\strat_{\class} = 1$ whenever $\strat \in \simplex(\class)$ \textendash\ \ie whenever $\supp(\strat) \subseteq \class$.
\end{proof}

%

\begin{proof}[Proof of \cref{prop:mirror}]
We begin by noting that the optimization problem \eqref{eq:ent-conj} can be written more explicitly as
\begin{equation}
\label{eq:opt-class}
\tag{$\Opt_{\class}$}
\begin{aligned}
\textup{maximize}
	&\quad
	\braket{\score}{\strat} - \hreg_{\class}(\strat),
	\\
\textup{subject to}
	&\txs
	\quad
	\strat \in \simplex(\elems)
	\text{ and }
	\supp(\strat) \subseteq \class.
\end{aligned}
\end{equation}
We will proceed to show that the (unique) solution of \eqref{eq:opt-class} is given by the vector of conditional probabilities $(\choice[\elem \vert \class](\score))_{\elem\in\elems}$.
The expression \eqref{eq:ent-conj} for the maximal value of \eqref{eq:opt-class} will then be derived from \cref{prop:nest2cond}, and the differential representation \eqref{eq:ent-mirror} will follow from Legendre's identity.
We make all this precise in a series of individual steps below.

\para{Step 1: Optimality conditions for \eqref{eq:opt-class}}

For all $\elem\in\class$, the definition of the nested entropy gives
\begin{align}
\label{eq:nest-grad}
\frac{\pd\hreg_{\class}}{\pd\strat_{\elem}}
	&= \sum_{\lvlalt=\lvl}^{\nLvls} \diff\atlvl{\lvlalt}
		\;\sum_{\mathclap{\classalt \desceq[\lvlalt] \class}}\;
			\frac{\pd}{\pd\strat_{\elem}}
				\parens{\strat_{\classalt} \log\strat_{\classalt}}
	= \sum_{\lvlalt=\lvl}^{\nLvls} \diff\atlvl{\lvlalt}
		\;\sum_{\mathclap{\classalt \desceq[\lvlalt] \class}}\;
			(1 + \log\strat_{\classalt})
			\frac{\pd\strat_{\classalt}}{\pd\strat_{\elem}}
	\notag\\
	&= \sum_{\lvlalt=\lvl}^{\nLvls} \diff\atlvl{\lvlalt}
		\;\sum_{\mathclap{\classalt \desceq[\lvlalt] \class}}\;
			(1 + \log\strat_{\classalt})
			\oneof{\elem \in \classalt}
	\notag\\
	&= \sum_{\lvlalt=\lvl}^{\nLvls} \diff\atlvl{\lvlalt}
				(1 + \log\strat_{\class\atlvl{\lvlalt}})
	\notag\\
	&= \temp\atlvl{\lvl}
		+ \sum_{\lvlalt=\lvl}^{\nLvls} \diff\atlvl{\lvlalt} \log\strat_{\class\atlvl{\lvlalt}}
\end{align}
where $\classleaf{\class}{\elem}$ denotes the lineage of $\elem$ up to $\class$ (inclusive).
This implies that $\pd_{\elem}\hreg_{\class}(\strat) \to -\infty$ whenever $\strat_{\elem}\to0$, so any solution $\strat$ of \eqref{eq:opt-class} must have $\strat_{\elem} > 0$ for all $\elem\in\class$.
In view of this, the first-order optimality conditions for \eqref{eq:opt-class} become
\begin{equation}
\label{eq:KKT}
\score_{\elem}
	- \frac{\pd\hreg_{\class}}{\pd\strat_{\elem}}
	= \score_{\elem}
		- \temp\atlvl{\lvl}
		- \sum_{\lvlalt=\lvl}^{\nLvls} \diff\atlvl{\lvlalt} \log\strat_{\class\atlvl{\lvlalt}}
	= \coef
	\quad
	\text{for all $\elem\in\class$},
\end{equation}
where $\coef$ is the Lagrange multiplier for the equality constraint $\sum_{\elem\in\elems} \strat_{\elem} = 1$.%
\footnote{Since $\strat_{\elem} > 0$ for all $\elem\in\class$, 
the multipliers for the corresponding inequality constraints all vanish by complementary slackness.}
Thus, after rearranging terms and exponentiating, we get
\begin{equation}
\label{eq:chain}
\strat_{\class\atlvl{\nLvls}}^{\diff\atlvl{\nLvls}}
	\cdot \strat_{\class\atlvl{\nLvls-1}}^{\diff\atlvl{\nLvls-1}}
	\dotsm
	\strat_{\class\atlvl{\lvl}}^{\diff\atlvl{\lvl}}
	= \frac{\exp(\score_{\elem})}{\pf},
\end{equation}
for some proportionality constant $\pf \equiv \pf(\score) > 0$.

\para{Step 2: Solving \eqref{eq:opt-class}}

The next step of our proof will focus on unrolling the chain \eqref{eq:chain}, one attribute at a time.
To start, recall that $\diff\atlvl{\nLvls} = \temp\atlvl{\nLvls}$, so \eqref{eq:chain} becomes
\begin{equation}
\label{eq:chain1}
\strat_{\class\atlvl{\nLvls}}
	\cdot \strat_{\class\atlvl{\nLvls-1}}^{\diff\atlvl{\nLvls-1} / \temp\atlvl{\nLvls}}
	\dotsm \strat_{\class\atlvl{\lvl}}^{\diff\atlvl{\lvl} / \temp\atlvl{\nLvls}}
	= \frac
		{\exp(\score_{\class\atlvl{\nLvls}}/\temp\atlvl{\nLvls})}
		{\pf^{1/\temp\atlvl{\nLvls}}},
\end{equation}
where we used the fact that $\class\atlvl{\nLvls} = \elem$ by definition.
Now, since $\class\atlvl{\nLvls-1} \desceq \class\atlvl{\lvl} = \class$, it follows that all children of $\class\atlvl{\nLvls-1}$ are also desendants of $\class$, so \eqref{eq:chain1} applies to all siblings of $\class\atlvl{\nLvls}$ as well.
Hence, summing \eqref{eq:chain1} over $\class\atlvl{\nLvls} \childof \class\atlvl{\nLvls-1}$, we get
\begin{equation}
\label{eq:chain2}
\strat_{\class\atlvl{\nLvls-1}}
	\cdot \strat_{\class\atlvl{\nLvls-1}}^{\diff\atlvl{\nLvls-1} / \temp\atlvl{\nLvls}}
	\dotsm \strat_{\class\atlvl{\lvl}}^{\diff\atlvl{\lvl} / \temp\atlvl{\nLvls}}
	= \frac
		{\exp(\score_{\class\atlvl{\nLvls-1}} / \temp\atlvl{\nLvls})}
		{\pf^{1/\temp\atlvl{\nLvls}}},
\end{equation}
where we used
the definition \eqref{eq:prob-class} of $\strat_{\class\atlvl{\nLvls-1}} = \sum_{\class\atlvl{\nLvls} \childof \class\atlvl{\nLvls-1}} \strat_{\class\atlvl{\nLvls}}$
and the recursive definition \eqref{eq:score} for $\score_{\class\atlvl{\nLvls-1}}$, \ie the fact that
$\exp(\score_{\class\atlvl{\nLvls-1}} / \temp\atlvl{\nLvls}) = \sum_{\class\atlvl{\nLvls} \childof \class\atlvl{\nLvls-1}} \exp(\score_{\class\atlvl{\nLvls}} / \temp\atlvl{\nLvls})$.
Therefore, noting that
\begin{equation}
1 + \frac{\diff\atlvl{\nLvls-1}}{\temp\atlvl{\nLvls}}
	= 1 + \frac{\temp\atlvl{\nLvls-1} - \temp\atlvl{\nLvls}}{\temp\atlvl{\nLvls}}
	= \frac{\temp\atlvl{\nLvls-1}}{\temp\atlvl{\nLvls}}
\end{equation}
the product \eqref{eq:chain2} becomes
\begin{equation}
\strat_{\class\atlvl{\nLvls-1}}^{\temp\atlvl{\nLvls-1}}
	\cdot \strat_{\class\atlvl{\nLvls-2}}^{\diff\atlvl{\nLvls-2}}
	\dotsm \strat_{\class\atlvl{\lvl}}^{\diff\atlvl{\lvl}}
	= \frac
		{\exp(\score_{\class\atlvl{\nLvls-1}})}
		{\pf}
\end{equation}
or, equivalently
\begin{equation}
\label{eq:chain3}
\strat_{\class\atlvl{\nLvls-1}}
	\cdot \strat_{\class\atlvl{\nLvls-2}}^{\diff\atlvl{\nLvls-2} / \temp\atlvl{\nLvls-1}}
	\dotsm \strat_{\class\atlvl{\lvl}}^{\diff\atlvl{\lvl} / \temp\atlvl{\nLvls-1}}
	= \frac
		{\exp(\score_{\class\atlvl{\nLvls-1}} / \temp\atlvl{\nLvls-1})}
		{\pf^{1/\temp\atlvl{\nLvls-1}}}.
\end{equation}
This last equation has the same form as \eqref{eq:chain2} applied to the chain $\class\atlvl{\lvl} \parentof \class\atlvl{\lvl+1} \parentof \dotsm \parentof \class\atlvl{\nLvls-1}$ instead of $\class\atlvl{\lvl} \parentof \class\atlvl{\lvl+1} \parentof \dotsm \parentof \class\atlvl{\nLvls}$.
Thus, proceeding inductively, we conclude that
\begin{equation}
\label{eq:chain4}
\strat_{\class\atlvl{\lvlalt}}^{\temp\atlvl{\lvlalt}} \prod_{j=\lvlalt-1}^{\lvl} \strat_{\class\atlvl{j}}^{\diff\atlvl{j}}
	= \frac{\exp(\score_{\class\atlvl{\lvlalt}})}{\pf}
	\quad
	\text{for all $\lvlalt=\nLvls,\dotsc,\lvl$}
\end{equation}
with the empty product $\prod_{j\in\varnothing} \strat_{\class\atlvl{j}}^{\diff\atlvl{j}}$ taken equal to $1$ by standard convention.

Now, substituting $\lvlalt \gets \lvlalt+1$ in \eqref{eq:chain4}, we readily get
\begin{equation}
\label{eq:chain5}
\strat_{\class\atlvl{\lvlalt+1}}^{\temp\atlvl{\lvlalt+1}}
	\cdot \strat_{\class\atlvl{\lvlalt}}^{\diff\atlvl{\lvlalt}} \prod_{j=\lvlalt-1}^{\lvl} \strat_{\class\atlvl{j}}^{\diff\atlvl{j}}
	= \frac{\exp(\score_{\class\atlvl{\lvlalt+1}})}{\pf}
	\quad
	\text{for all $\lvlalt=\nLvls-1,\dotsc,\lvl$}.
\end{equation}
Consequently, recalling that $\diff\atlvl{\lvlalt} = \temp\atlvl{\lvlalt} - \temp\atlvl{\lvlalt+1}$ and dividing \eqref{eq:chain4} by \eqref{eq:chain5}, we get
\begin{equation}
\label{eq:chain6}
\frac
	{\strat_{\class\atlvl{\lvlalt+1}}^{\temp\atlvl{\lvlalt+1}}}
	{\strat_{\class\atlvl{\lvlalt}}^{\temp\atlvl{\lvlalt+1}}}
	= \frac
		{\exp(\score_{\class\atlvl{\lvlalt+1}})}
		{\exp(\score_{\class\atlvl{\lvlalt}})},
\end{equation}
and hence
\begin{equation}
\label{eq:chain-cond}
\frac
	{\strat_{\class\atlvl{\lvlalt+1}}}
	{\strat_{\class\atlvl{\lvlalt}}}
	= \frac
		{\exp(\score_{\class\atlvl{\lvlalt+1}} / \temp\atlvl{\lvlalt+1})}
		{\exp(\score_{\class\atlvl{\lvlalt}} / \temp\atlvl{\lvlalt+1})}
	= \choice[\class\atlvl{\lvlalt+1} \vert \class\atlvl{\lvlalt}](\score)
\end{equation}
by the definition of the conditional logit choice model \eqref{eq:NLC}.
Therefore, by unrolling the chain
\begin{equation}
\label{eq:chain7}
\strat_{\elem \vert \class}
	= \frac{\strat_{\elem}}{\strat_{\class}}
	= \frac{\strat_{\class\atlvl{\nLvls}}}{\strat_{\class\atlvl{\nLvls-1}}}
	\cdot \frac{\strat_{\class\atlvl{\nLvls-1}}}{\strat_{\class\atlvl{\nLvls-2}}}
	\dotsm \frac{\strat_{\class\atlvl{\lvl+1}}}{\strat_{\class\atlvl{\lvl}}}
	= \choice[\class\atlvl{\nLvls} \vert \class\atlvl{\nLvls-1}](\score)
	\times \choice[\class\atlvl{\nLvls-1} \vert \class\atlvl{\nLvls-2}](\score)
	\times \dotsm
	\times \choice[\class\atlvl{\lvl+1} \vert \class\atlvl{\lvl}](\score)
\end{equation}
we obtain the nested expression
\begin{equation}
\strat_{\elem}
	= \strat_{\class} \prod_{\lvlalt=\lvl}^{\nLvls-1} \choice[\class\atlvl{\lvlalt+1} \vert \class\atlvl{\lvlalt}](\score)
	\quad
	\text{for all $\elem\in\class$}.
\end{equation}
Thus, with $\strat_{\class} = 1$ (by the fact that $\supp(\strat) = \class$), we finally conclude that
\begin{equation}
\label{eq:prob-anc}
\strat_{\elem}
	= \prod_{\lvlalt=\lvl}^{\nLvls-1} \choice[\class\atlvl{\lvlalt+1} \vert \class\atlvl{\lvlalt}](\score)
	= \choice[\elem \vert \class](\score)
	\quad
	\text{for all $\elem\in\class$}.
\end{equation}

\para{Step 3: The maximal value of \eqref{eq:opt-class}}

To obtain the value of the maximization problem \eqref{eq:opt-class}, we will proceed to substitute \eqref{eq:prob-anc} in the expression \eqref{eq:nest2cond} provided by \cref{prop:nest2cond} for $\hreg_{\class}(\strat)$.
To that end, for all $\lvlalt = \lvl,\dotsc,\nLvls-1$ and all $\class\atlvl{\lvlalt} \desceq[\lvlalt] \class$, the definition \eqref{eq:entropy-cond} of the conditional entropy gives:
\begin{align}
\hreg(\strat \vert \class\atlvl{\lvlalt})
	&= \temp\atlvl{\lvlalt+1} \, \strat_{\class\atlvl{\lvlalt}}
		\sum_{\class\atlvl{\lvlalt+1} \childof \class\atlvl{\lvlalt}}
			\strat_{\class\atlvl{\lvlalt+1} \vert \class\atlvl{\lvlalt}}
			\log \strat_{\class\atlvl{\lvlalt+1} \vert \class\atlvl{\lvlalt}}
	\explain{by definition}
	\\
	&= \temp\atlvl{\lvlalt+1} \, \strat_{\class\atlvl{\lvlalt}}
		\sum_{\class\atlvl{\lvlalt+1} \childof \class\atlvl{\lvlalt}}
			\strat_{\class\atlvl{\lvlalt+1} \vert \class\atlvl{\lvlalt}}
			\log \frac
				{\exp(\score_{\class\atlvl{\lvlalt+1}} / \temp\atlvl{\lvlalt+1})}
				{\exp(\score_{\class\atlvl{\lvlalt}} / \temp\atlvl{\lvlalt+1})}
	\explain{by \eqref{eq:chain-cond}}
	\\
	&= \strat_{\class\atlvl{\lvlalt}}
		\sum_{\class\atlvl{\lvlalt+1} \childof \class\atlvl{\lvlalt}}
			\strat_{\class\atlvl{\lvlalt+1} \vert \class\atlvl{\lvlalt}}
			\score_{\class\atlvl{\lvlalt+1}}
		- \strat_{\class\atlvl{\lvlalt}} \score_{\class\atlvl{\lvlalt}}
	\explain{since $\sum_{\class\atlvl{\lvlalt+1} \childof \class\atlvl{\lvlalt}} \strat_{\class\atlvl{\lvlalt+1} \vert \class\atlvl{\lvlalt}} = 1$}\\
	&= \sum_{\class\atlvl{\lvlalt+1} \childof \class\atlvl{\lvlalt}}
			\strat_{\class\atlvl{\lvlalt+1}}
			\score_{\class\atlvl{\lvlalt+1}}
		- \strat_{\class\atlvl{\lvlalt}} \score_{\class\atlvl{\lvlalt}}
\end{align}
and hence
\begin{align}
\label{eq:ent-cond}
\sum_{\class\atlvl{\lvlalt} \desceq[\lvlalt] \class}
	\hreg(\strat \vert \class\atlvl{\lvlalt})
	= \sum_{\class\atlvl{\lvlalt} \desceq[\lvlalt] \class}
		\bracks*{
			\sum_{\class\atlvl{\lvlalt+1} \childof \class\atlvl{\lvlalt}}
				\strat_{\class\atlvl{\lvlalt+1}}
				\score_{\class\atlvl{\lvlalt+1}}
			- \strat_{\class\atlvl{\lvlalt}} \score_{\class\atlvl{\lvlalt}}
			}
	= \sum_{\class\atlvl{\lvlalt+1} \desceq[\lvlalt+1] \class}
			\strat_{\class\atlvl{\lvlalt+1}}
			\score_{\class\atlvl{\lvlalt+1}}
		- \sum_{\class\atlvl{\lvlalt} \desceq[\lvlalt] \class}
			\strat_{\class\atlvl{\lvlalt}}
			\score_{\class\atlvl{\lvlalt}}.
\end{align}
Thus, telescoping this last releation over $\lvlalt = \lvl,\dotsc,\nLvls$ and invoking \cref{prop:nest2cond}, we obtain:
\begin{align}
\hreg_{\class}(\strat)
	&= \sum_{\classalt \desceq \class} \hreg(\strat \vert \classalt)
		+ \temp\atlvl{\lvlalt} \, \strat_{\class} \cancel{\log\strat_{\class}}
	\explain{by \cref{prop:nest2cond}}
	\\
	&= \sum_{\lvlalt = \lvl}^{\nLvls-1} \sum_{\class\atlvl{\lvlalt} \desceq[\lvlalt] \class}
		\hreg(\strat \vert \class\atlvl{\lvlalt})
	\explain{collect parent classes}
	\\
	&= \sum_{\lvlalt = \lvl}^{\nLvls-1}
		\bracks*{
			\sum_{\class\atlvl{\lvlalt+1} \desceq[\lvlalt+1] \class}
				\strat_{\class\atlvl{\lvlalt+1}}
				\score_{\class\atlvl{\lvlalt+1}}
			- \sum_{\class\atlvl{\lvlalt} \desceq[\lvlalt] \class}
				\strat_{\class\atlvl{\lvlalt}}
				\score_{\class\atlvl{\lvlalt}}
				}
	\explain{by \eqref{eq:ent-cond}}
	\\
	&= \braket{\score}{\strat} - \strat_{\class} \score_{\class}
\end{align}
where, in the second line, we used the fact that the conditional entropy $\hreg(\strat \vert \class\atlvl{\nLvls})$ relative to any childless class $\class\atlvl{\nLvls} \in \classes\atlvl{\nLvls}$ is zero by definition.
Accordingly, substituting back to \eqref{eq:opt-class} we conclude that
\begin{equation}
\val\eqref{eq:opt-class}
	= \braket{\score}{\strat} - \hreg_{\class}(\strat)
	= \strat_{\class} \score_{\class}
	= \score_{\class},
\end{equation}
as claimed.

\para{Step 4: Differential representation of conditional probabilities}

To prove the second part of the proposition, recall that the restricted entropy function $\hreg_{\vert\class}$ is convex, and let
\begin{equation}
\label{eq:hconj}
\hconj_{\vert\class}(\score)
	= \max_{\strat \in \simplex(\elems)}
		\braces{\braket{\score}{\strat} - \hreg_{\vert\class}(\strat)}
\end{equation}
denote its convex conjugate.%
\footnote{Note here that $\hconj_{\vert\class}(\score)$ is bounded from above by the convex conjugate $\hconj_{\class}(\score)$ of $\hreg_{\class}(\strat)$ because the latter does not include the constraint $\supp(\strat) \subseteq \class$.}
By standard results in convex analysis \citep[\eg Theorem 23.5 in][]{Roc70}, $\hconj_{\vert\class}$ is differentiable in $\score$ and we have the Legendre identity:
\begin{equation}
\label{eq:Legendre}
\strat
	= \nabla\hconj_{\vert\class}(\score)
	\iff
\score
	\in \subd\hreg_{\vert\class}(\strat)
	\iff
\strat
	\in \argmax_{\stratalt\in\simplex(\elems)}
		\braces{\braket{\score}{\stratalt} - \hreg_{\vert\class}(\stratalt)}
\end{equation}
Now, by \eqref{eq:prob-anc}, we have $\strat_{\elem} = \choice[\elem \vert \class](\score)$ whenever $\strat$ solves \eqref{eq:opt-class} and hence, by Fermat's rule, whenever $\score - \subd\hreg_{\vert\class}(\strat) \ni 0$.
Our claim then follows by noting that $\hconj_{\vert\class}(\score) = \score_{\class}$ and combining the first and third legs of the equivalence \eqref{eq:Legendre}.
\end{proof}

These properties of the nested entropy function (and its restricted variant) will play a key role in deriving a suitable energy function for the \acl{NEW} algorithm.
We make this precise in \cref{app:regret} below.

\section{Auxiliary bounds and results}
\label{app:aux}

\newmacro{\nps}{\sigma}
\newmacro{\varcst}{\alpha}
\newmacro{\conjlvl}{k}
\newmacro{\rbound}{R}

Throughout this appendix, we assume the following primitives:
\begin{itemize}
\item
A fixed sequence of real numbers $\temp\atlvl{1} \geq \temp\atlvl{2} \geq \cdots \geq \temp\atlvl{\nLvls} > 0$;
all entropy-related objects will be defined relative to this sequence as per the previous section.
\item
A score vector $\score \in \R^{\elems}$ that defines inductively the score $\score_{\class}$ of any class $\class \in \classes$ using \eqref{eq:score}, as well as the associated nested choice probability $\choice(\score)$ as per \eqref{eq:NLC}.
\item
A vector of cost increments $\incr = (\incr_{\class})_{\class\in\classes} \in \R^{\classes}$ that defines an associated \emph{cost vector} $\cost \in \R^{\elems}$ as per \eqref{eq:cost-elem}, viz.
\begin{equation}
\cost_{\elem}
	= \sum_{\class\ni\elem}\incr_{\class}
	\quad
	\text{for all $\elem\in\elems$}.
\end{equation}
\end{itemize}

Moreover, for all $\cost, \score \in \R^{\elems}$, we define the \emph{nested power sum} function $\nps_{\cost,\score} \from \classes \backslash \classes_{\nLvls} \rightarrow \R$ which, to any $\class \in \classes \backslash \classes_{\nLvls}$, associates the real number
\begin{equation}
\nps_{\cost, \score}(\class) 
= \begin{cases}
		\sum\limits_{\elem \childof \class}
			\choice[\elem\vert\class](\score)\exp\parens*{-\cost_{\elem}/\temp\atlvl{\nLvls}}
			&\quad
			\text{if $\attof{\class}=\nLvls - 1$},
		\\
		\sum\limits_{\alt{\class} \childof \class}
			\choice[\alt{\class}\vert\class](\score)
			\nps_{\cost, \score}(\alt{\class})^{\frac{\temp\atlvl{\lvl + 2}}{\temp\atlvl{\lvl + 1}}}
			&\quad
			\text{if $\attof{\class}=\lvl<\nLvls-1$}.
	\end{cases}
\end{equation}

The following lemma links the increments of the conjugate entropy $\hconj$ to the nested power sum defined above:

\begin{lemma}
\label{lem:hconj-and-nps}
For all $\score \in \R^{\elems}$, $\cost \in \R^{\elems}$, we have
\begin{equation}
\label{eq:hconj-and-nps}
\hconj\parens*{\score - \cost}
	= \hconj\parens*{\score} + \temp\atlvl{1} \log\parens*{\nps_{\cost, \score}(\elems)}.
\end{equation}
\end{lemma}

\Cref{lem:hconj-and-nps} will be proved as a corollary of the more general result below:

\begin{lemma}
\label{lem:hconj-and-nps-internal}
Fix some $\score \in \R^{\elems}$ and $\cost \in \R^{\elems}$.
Then, for all $\class\atlvl{\lvl} \in \classes\atlvl{\lvl}$, $\lvl < \nLvls$,we have
\begin{equation}
\label{eq:hconj-and-nps-internal}
\exp\parens*{\frac{\hconj_{\vert\class\atlvl{\lvl}}\parens*{\score - \cost}}{\temp\atlvl{\lvl+1}}}
	= \exp\parens*{\frac{\hconj_{\vert\class\atlvl{\lvl}}\parens*{\score}}{\temp\atlvl{\lvl+1}}}\nps_{\cost, \score}(\class\atlvl{\lvl})
\end{equation}
\end{lemma}

\begin{proof}[Proof of \cref{lem:hconj-and-nps}]
Simply invoke \cref{lem:hconj-and-nps-internal} with $\class \gets \elems$.
\end{proof}

\begin{proof}[Proof of \cref{lem:hconj-and-nps-internal}]

We proceed by descending induction on $\lvl = \attof{\class}$. 

\para{Base step}
Fix some $\class \in \classes$ with $\attof{\class}=\nLvls-1$.
We then have:
\begin{align}
\exp\parens*{\frac{\hconj_{\vert\class}\parens*{\score - \cost}}{\temp\atlvl{\nLvls}}}
	&= \sum\limits_{\elem \childof \class}
		\exp\parens*{\frac{\hconj_{\vert\elem}\parens*{\score - \cost}}{\temp\atlvl{\nLvls}}}
	\explain{by \cref{eq:conj-recursive}}
	\\
	&= \sum\limits_{\elem \childof \class}
		\exp\parens*{\frac{\hconj_{\vert \elem}\parens*{\score} - \cost_{\elem}}{\temp\atlvl{\nLvls}}} \explain{the \elem 's \ are leaves}
	\\
	&= \sum\limits_{\elem \childof \class}
		\exp\parens*{\frac{\hconj_{\vert \elem}\parens*{\score}}{\temp\atlvl{\nLvls}}}
		\exp\parens*{-\frac{\cost_{\elem}}{\temp\atlvl{\nLvls}}}
	\notag\\
	&= \exp\parens*{\frac{\hconj_{\vert \class} \parens*{\score}}{\temp\atlvl{\nLvls}}}
			\underbrace{\sum\limits_{\elem \childof \class} \left[ \dfrac{\exp\parens*{\frac{\hconj_{\vert \elem}\parens*{\score}}{\temp\atlvl{\nLvls}}}}{\exp\parens*{\frac{\hconj_{\vert \class}\parens*{\score}}{\temp\atlvl{\nLvls}}}} \right]\exp\parens*{-\frac{\cost_{\elem}}{\temp\atlvl{\nLvls}}}}_{= \nps_{\cost, \score}(\class)\text{ by definition}}
	\notag\\
	&= \exp\parens*{\frac{\hconj_{\vert\class}\parens*{\score}}{\temp\atlvl{\nLvls}}}\nps_{\cost, \score}(\class)
\end{align}
with the last equality following from the definition of $\choice[\elem\vert\class]$ via \eqref{eq:NLC} and by the definition of $\nps_{\cost,\score}(\class)$. 
This concludes the start of the induction process.

\para{Induction step}
Fix some $\class \in \classes$ with $\attof{\class}=\lvl-1$, $\lvl < \nLvls$, and suppose that \eqref{eq:hconj-and-nps-internal} holds at level $\lvl$.
We then have:
\begin{align}
	\exp\parens*{\frac{\hconj_{\vert\class}\parens*{\score - \cost}}{\temp\atlvl{\lvl}}} 
	&= \sum\limits_{\classalt \childof \class}
		\exp\parens*{\frac{\hconj_{\vert\classalt}\parens*{\score - \cost}}{\temp\atlvl{\lvl}}} 
	\notag\\ 
	&= \sum\limits_{\classalt \childof \class}
		\exp\parens*{\frac{\hconj_{\vert \classalt}\parens*{\score - \cost}}{\temp\atlvl{\lvl+1}}}^{\frac{\temp\atlvl{\lvl+1}}{\temp\atlvl{\lvl}}}
	\notag\\
	&= \sum\limits_{\classalt \childof \class}
		\bracks*{\exp\parens*{\frac{\hconj_{\vert \classalt}\parens*{\score}}{\temp\atlvl{\lvl+1}}} \nps_{\cost, \score}(\classalt)}^{\frac{\temp\atlvl{\lvl+1}}{\temp\atlvl{\lvl}}}
	\explain{inductive hypothesis}\\
	&= \sum\limits_{\classalt \childof \class}
		\exp\parens*{\frac{\hconj_{\vert \classalt}\parens*{\score}}{\temp\atlvl{\lvl}}} \nps_{\cost, \score}(\classalt)^{\frac{\temp\atlvl{\lvl+1}}{\temp\atlvl{\lvl}}}
	\notag\\
	&= \exp\parens*{\frac{\hconj_{\vert \class}\parens*{\score}}{\temp\atlvl{\lvl}}}
		\underbrace{\sum\limits_{\classalt \childof \class} \left[ \dfrac{ \exp\parens*{\frac{\hconj_{\vert \classalt}\parens*{\score}}{\temp\atlvl{\lvl}}}}{\exp\parens*{\frac{\hconj_{\vert \class}\parens*{\score}}{\temp\atlvl{\lvl}}}} \right] \nps_{\cost, \score}(\classalt)^{\frac{\temp\atlvl{\lvl+1}}{\temp\atlvl{\lvl}}}}_{= \nps_{\cost, \score}(\class)\text{ by definition}}
	\notag\\
	&= \exp\parens*{\frac{\hconj_{\vert\class}\parens*{\score}}{\temp\atlvl{\nLvls}}}\nps_{\cost, \score}(\class)
\end{align}
with the last equality following from the definition of $\choice[\classalt\vert\class]$ and $\nps_{\cost, \score}(\class)$.
This being true for all $\class \in \classes$ with $\attof{\class}=\lvl-1$, the inductive step and \textendash\ a fortiori \textendash\ our proof are complete.
\end{proof}

The next lemma provides an upper bound for $\nps_{\cost,\score}(\elems)$, which will in turn allow us to derive a bound for the increment of $\hconj$.
\begin{lemma}
\label{lem:upperbound-nps}
For $\score \in \R^{\elems}$ and $\cost \in [0,+\infty)^{\elems}$, we have:
\begin{equation}
\label{eq:upperbound-nps}
\nps_{\cost, \score}(\elems)
	\leq 1
	- \frac{1}{\temp\atlvl{1}}
		\bracks*{
		\sum\limits_{\elem\in\elems}\choice[\elem](\score)\cost_{\elem}
		- \frac{1}{2\temp\atlvl{\nLvls}}
		\sum\limits_{\elem\in\elems}\choice[\elem](\score)\cost_{\elem}^{2}
		}.
\end{equation}
\end{lemma}

As in the case of \ref{lem:hconj-and-nps}, \cref{lem:upperbound-nps} will follow as a special case of the more general, class-based result below:

\begin{lemma}
\label{lem:upperbound-nps-internal}
Fix some $\score \in \R^{\elems}$ and $\cost \in \R_{+}^{\elems}$.
Then, for all $\class\atlvl{\lvl} \in \classes\atlvl{\lvl}$, $\lvl < \nLvls$,we have
\begin{equation}
\label{eq:upperbound-nps-internal}
\nps_{\cost, \score}(\class\atlvl{\lvl})
	\leq 1
	- \frac{1}{\temp\atlvl{\lvl+1}}
		\bracks*{
		\sum\limits_{\elem\in\class\atlvl{\lvl}}\choice[\elem\vert\class\atlvl{\lvl}](\score)\cost_{\elem}
		- \frac{1}{2\temp\atlvl{\nLvls}}
		\sum\limits_{\elem\in\class\atlvl{\lvl}}\choice[\elem\vert\class\atlvl{\lvl}](\score)\cost_{\elem}^{2}
		},
\end{equation}
\end{lemma}

\begin{proof}[Proof of \cref{lem:upperbound-nps}]
Simply invoke \cref{lem:upperbound-nps-internal} with $\class \gets \elems$.
\end{proof}

\begin{proof}[Proof of \cref{lem:upperbound-nps-internal}]

We proceed again by descending induction on $\lvl = \attof{\class}$. 

\para{Base step}
Fix some $\class \in \classes$ with $\attof{\class}=\nLvls-1$.
We then have:
\PM{I think there was an indexing problem here: the primes $'$ went to the costs instead of the classes.
I fixed, but please check?}
\begin{align}
\nps_{\cost, \score}(\class)
	&= \sum\limits_{\classalt \childof \class} \choice[\classalt\vert\class] (\score) \exp(-\frac{\cost_{\classalt}}{\temp\atlvl{\nLvls}}) 
	\notag\\
	&\leq \sum\limits_{\classalt \childof \class}
		\choice[\classalt\vert\class] (\score) (1- \frac{\cost_{\classalt}}{\temp\atlvl{\nLvls}}
		\frac{\cost_{\classalt}^{2}}{2\temp\atlvl{\nLvls}^{2}})
	\explain{ $e^{-x} \leq 1 - x + x^{2}/2$ for $x \geq 0 $}\\
	&=1
	- \frac{1}{\temp\atlvl{\nLvls}}
		\bracks*{
		\sum\limits_{\classalt \childof \class} \choice[\classalt\vert\class] (\score) \cost_{\classalt}
		- \frac{1}{2\temp\atlvl{\nLvls}} \sum\limits_{\classalt \childof \class} \choice[\classalt\vert\class] (\score) \cost_{\classalt}^{2}
		}
	\notag\\
	&=1
	- \frac{1}{\temp_{(\nLvls-1)+1}}
		\bracks*{
		\sum\limits_{\elem \childof \class} \choice[\elem\vert\class] (\score) \cost_{\elem}
		- \frac{1}{2\temp\atlvl{\nLvls}} \sum\limits_{\elem \childof \class} \choice[\elem\vert\class] (\score) \cost_{\elem}^{2}
		}
\end{align}
so the initialization of the induction process is complete.

\para{Induction step}
Fix some $\class \in \classes$ with $\attof{\class}=\lvl-1$, $\lvl < \nLvls$, and suppose that \eqref{eq:upperbound-nps-internal} holds at level $\lvl$.
We then have:
\begin{align}
\nps_{\cost, \score}(\class)
	&= \sum\limits_{\classalt \childof \class}
		\choice[\classalt\vert\class] (\score) \nps_{\cost, \score}(\classalt)^{\frac{\temp_{\lvl+1}}{\temp_{\lvl}}}
	\notag\\
	&= \sum\limits_{\classalt \childof \class}
		\choice[\classalt\vert\class] (\score)
		\bracks*{1 + \frac{1}{\temp\atlvl{\lvl+1}} \left(-\sum\limits_{\elem \childof \classalt} \choice[\elem\vert\classalt] (\score) \cost_{\elem} + \frac{1}{2\temp\atlvl{\nLvls}} \sum\limits_{\elem \childof \classalt} \choice[\elem\vert\classalt] (\score) \cost_{\elem}^{2}  \right)
		}^{\frac{\temp\atlvl{\lvl+1}}{\temp\atlvl{\lvl}}}
	\explain{inductive hypothesis}\\
	&\leq \sum\limits_{\classalt \childof \class} \choice[\classalt\vert\class] (\score) \left[ 1 + \frac{1}{\temp\atlvl{\lvl}} \left(-\sum\limits_{\elem \childof \classalt} \choice[\elem\vert\classalt] (\score) \cost_{\elem} + \frac{1}{2\temp\atlvl{\nLvls}} \sum\limits_{\elem \childof \classalt} \choice[\elem\vert\classalt] (\score) \cost_{\elem}^{2}  \right) \right]
	\explain{$(1+x)^{\beta} \leq 1 + \beta x$ for $\beta \leq 1$}
	\\
	&= 1 + \frac{1}{\temp\atlvl{\lvl}} \left[ -\sum\limits_{\classalt \childof \class} \sum\limits_{\elem \childof \classalt} \choice[\elem\vert\classalt](\score) \choice[\classalt\vert\class](\score) \cost_{\elem} + \frac{1}{2\temp\atlvl{\nLvls}} \sum\limits_{\classalt \childof \class} \sum\limits_{\elem \childof \classalt} \choice[\elem\vert\classalt](\score) \choice[\classalt\vert\class](\score) \cost_{\elem}^{2} \right]
	\\
	&= 1 + \frac{1}{\temp\atlvl{(\lvl-1)+1}} \left[ \sum\limits_{\elem \childof \class} \choice[\elem\vert\class] (\score) \cost_{\elem} + \frac{1}{2\temp\atlvl{\nLvls}} \sum\limits_{\elem \childof \class} \choice[\elem\vert\class] (\score) \cost_{\elem}^{2} \right]
\end{align}

This being true for all $\class \in \classes$ s.t. $\attof{\class}=\lvl-1$, the induction step and the proof of our assertion are complete.
\end{proof}

With all this in hand, we are now in a position to upper bound the increments of the conjugate nested entropy $\hconj$.
\TR{negative increment ?}
\PM{Just ``increments'' I think is fine.}

\begin{proposition}
\label{prop:hconj-increment-bound}
For $\score \in \R^{\elems}$
and $\cost \in [0,+\infty)^{\elems}$, we have:
\begin{equation}
	\label{eq:hconj-increment-bound}
	\hconj(\score - \cost) - \hconj(\score) \leq -\braket{\choice[](\score)}{\cost}
	+ \frac{1}{2\temp\atlvl{\nLvls}}
	\sum\limits_{\elem\in\elems}\choice[\elem](\score)\cost_{\elem}^{2}.
\end{equation}
\end{proposition}

\begin{proof}
Using \cref{lem:hconj-and-nps,lem:upperbound-nps} and the concavity inequality $\log x\leq 1 + x$ directly delivers our assertion.
\end{proof}

\begin{remark}
It is useful to note that, given a cost increment vector $\incr \in \R^\classes$ with associated aggregate costs given by $\cost \in \R^{\elems}$ we have:
\TR{if we need this remark somewhere, it might be to show that $\exof*{Z_t|\mathcal{F}_t}=0$ at the end of the proof of Proposition 2? This could be expressed as a lemma if needed}
\PM{I think we don't need to say more\dots}
\begin{align*}
	\braket{\choice[](\score)}{\cost} &= \sum\limits_{\elem \in \elems}\choice[\elem](\score)\cost_{\elem} \\
	&= \sum\limits_{\elem \in \elems}\choice[\elem](\score)\sum\limits_{\class\ni\elem}\incr_{\class} \\
	&= \sum\limits_{\elem \in \elems}\choice[\elem](\score)\sum\limits_{\class\in\classes}\incr_{\class} \one_{\elem\in\class} \\
	&= \sum\limits_{\class \in \classes} \bracks*{\sum\limits_{\elem\in\elems} \choice[\elem](\score) \one_{\elem\in\class}}\incr_{\class} \\
	&= \sum\limits_{\class \in \classes} \choice[\class](\score)\incr_{\class}.
	\qedhere
\end{align*}
\end{remark}

We are finally in a position to prove the basic properties of the \ac{NIWE} estimator, which we restate below for convenience:

\NIWE*

\begin{proof}
Fix some $\class \in \classes$ with $\attof{\class} = \lvl \in \{1,\dotsc,\nLvls\}$ and lineage $\lineage{\class}$.
We will now prove both properties of the \eqref{eq:NIWE} estimator.
\begin{enumerate}
[label={\itshape\bfseries Part \arabic*.}]
\item 
We begin by showing that the estimator \eqref{eq:NIWE} is unbiased.
Indeed, we have:
\begin{align}
\exof*{\incrmodel_{\class}}
	&= \exof*{\frac
			{\oneof[\big]{\class\atlvl{\lvl} = \est\class\atlvl{\lvl},\dotsc,\class\atlvl{1} = \est\class\atlvl{1}}}
			{\strat_{\class\atlvl{\lvl} \vert \class\atlvl{\lvl - 1}} \!\dotsm \strat_{\class\atlvl{2} \vert \class\atlvl{1}} \strat_{\class\atlvl{1}}}
			\incr_{\class\atlvl{\lvl}}}
	= \exof*{\frac
			{\oneof[\big]{\class = \est\class}}
			{\strat_{\class}}
			\incr_{\class}}
	\explain{Rewriting \eqref{eq:NIWE}}\\
	&= \frac
			{\incr_{\class}}
			{\strat_{\class}}
			\underbrace{\exof*{\oneof[\big]{\class = \est\class}}}_{\strat_{\class}}
	= \incr_{\class}.
\end{align}

\item
We now turn to the proof of the importance-weighted mean-square bound of the estimator \eqref{eq:NIWE}.
In this case, for any $\class\in\classes$, we have:
\begin{align}
\exof*{\strat_{\class}\incrmodel_{\class}^{2}}
	&=\strat_{\class} \exof*{\incrmodel_{\class}^{2}} 
	=\strat_{\class}
		\exof*{\parens*{\frac
			{\oneof[\big]{\class = \est\class}}
			{\strat_{\class}}
			\incr_{\class\atlvl{\lvl}}}^{2}}
	\notag\\
	&=\strat_{\class}\frac{\incr_{\class\atlvl{\lvl}}^{2}}{\strat_{\class}^{2}}
	\exof*{\oneof[\big]{\class = \est\class}}
	= \incr_{\class\atlvl{\lvl}}^{2}
	\explain{because $\exof[\big]{\oneof[\big]{\class = \est\class}} = \strat_{\class}$}\\
\label{eq:order-2-incr-bound}
	&\leq \range_{\class}^{2}.
	\end{align}

\end{enumerate}

We are left to derive the bound for the aggregate cost estimator \eqref{eq:NIWE-cost}, viz.
\TR{total? global? or just `cost estimator'?}
\PM{Just wrote ``aggregate'', WDYT?}
\begin{equation}
\model_{\elem}
	= \sum\limits_{\class\ni\elem} \incrmodel_{\class}.
\end{equation}
With this in mind, we can write:
\begin{align}
\sum_{\elem\in\elems}\strat_{\elem} \model_{\elem}^{2}
	&= \sum_{\elem\in\elems}\strat_{\elem}
		\parens*{\sum_{\class\ni\elem} \incrmodel_{\class}}^{2}
	\notag\\
	&= \sum_{\elem\in\elems} \strat_{\elem}
	\bracks*{
		\sum_{\class\ni\elem}\incrmodel_{\class}^{2}
		+ 2\sum_{\alt{\class}\ni\elem}\sum_{\class\anc\alt{\class}}\incrmodel_{\class}\incrmodel_{\alt{\class}}}
	\notag\\
	&= \sum_{\elem\in\elems} \sum_{\class\in\classes}
		\strat_{\elem}\incrmodel_{\class}^{2}\one_{\elem\in\class}
	+ 2 \sum_{\elem\in\elems} \sum_{\alt{\class}\in\classes} \sum_{\class\anc\alt{\class}}
		\strat_{\elem} \incrmodel_{\class} \incrmodel_{\alt{\class}} \one_{\elem\in\alt{\class}}
	\notag\\
	&= \sum_{\class\in\classes}
		\incrmodel_{\class}^{2}
		\underbrace{\sum_{\elem\in\elems}\strat_{\elem} \one_{\elem\in\class}}_{\strat_{\class}}
	+ 2 \sum_{\alt{\class}\in\classes} \sum_{\class\anc\alt{\class}}
		\incrmodel_{\class} \incrmodel_{\alt{\class}}
		\underbrace{\sum_{\elem\in\elems}\strat_{\elem}\one_{\elem\in\alt{\class}}}_{\strat_{\alt{\class}}}
	\notag\\
\label{eq:order2-derivation-1}	
	&= \sum_{\class\in\classes}
		\strat_{\class}\incrmodel_{\class}^{2}
	+ 2 \sum_{\alt{\class}\in\classes} \sum_{\class\anc\alt{\class}}
		\strat_{\alt{\class}} \incrmodel_{\class}\incrmodel_{\alt{\class}}.
\end{align}
Now, decomposing the above sums attribute-by-attribute and taking expectations in \eqref{eq:order2-derivation-1}, we get:
\begin{equation}
\label{eq:order2-level-by-level}
\exof*{\sum_{\elem\in\elems}\strat_{\elem} \model_{\elem}^{2}}
	= \sum_{\lvl=1}^{\nLvls} \sum_{\class\atlvl{\lvl}\in\classes\atlvl{\lvl}}\strat_{\class\atlvl{\lvl}}\exof*{\incrmodel_{\class\atlvl{\lvl}}^{2}}
	+2 \sum_{1\leq \lvl < \alt{\lvl} \leq \nLvls}
		\sum_{\substack{
			\class\atlvl{\lvl}\in\classes\atlvl{\lvl}
			\\\class\atlvl{\alt{\lvl}}\desc_{\alt{\lvl}}\classes\atlvl{\lvl}}}
	\strat_{\class\atlvl{\alt{\lvl}}}\exof*{\incrmodel_{\class\atlvl{\lvl}}\incrmodel_{\class\atlvl{\alt{\lvl}}}}.
\end{equation}
The first term in \eqref{eq:order2-level-by-level} can simply be bounded using \eqref{eq:order-2-incr-bound}. Indeed:
\begin{equation}
\label{eq:order2-bound-first-term}
    \sum_{\lvl=1}^{\nLvls}\sum_{\class\atlvl{\lvl}\in\classes\atlvl{\lvl}}\strat_{\class\atlvl{\lvl}}\exof*{\incrmodel_{\class\atlvl{\lvl}}^{2}} 
    \leq \sum_{\lvl=1}^{\nLvls} \sum_{\class\atlvl{\lvl}\in\classes\atlvl{\lvl}} \range^{2}_{\class\atlvl{\lvl}}
    = \sum_{\lvl=1}^{\nLvls} \nClasses\atlvl{\lvl} \bar{\range}^{2}_{\lvl}.
\end{equation}
with $\bar{\range}_{\lvl} = \sqrt{\frac{1}{\nClasses\atlvl{\lvl}} \sum_{\class\atlvl{\lvl} \in \classes\atlvl{\lvl}} \range^{2}_{\class\atlvl{\lvl}}}$ for any $\lvl = 1, \dotsc, \nLvls$.

We now turn to the second term in \eqref{eq:order2-level-by-level}. Let $\{\epsilon_{\lvl, \alt{\lvl}}\}_{1\leq\alt{\lvl} < \lvl\leq\nLvls}$ be any fixed sequence of positive numbers. For any $\lvl, \alt{\lvl} \in \{1,\dotsc,\nLvls\}$ and any $\class\atlvl{\lvl} \in \classes\atlvl{\lvl}$ and $\class\atlvl{\alt{\lvl}} \in \classes\atlvl{\alt{\lvl}}$, the Peter-Paul inequality yields:
\begin{align}
\label{eq:peter-paul-incrmodels}
    2\incrmodel_{\class\atlvl{\alt{\lvl}}}\incrmodel_{\class\atlvl{\lvl}}
    \leq \frac{1}{\epsilon_{\lvl, \alt{\lvl}}} \incrmodel_{\class\atlvl{\alt{\lvl}}}^{2}
    + \epsilon_{\lvl, \alt{\lvl}} \incrmodel_{\class\atlvl{\lvl}}^{2}
\end{align}    

Injecting \eqref{eq:peter-paul-incrmodels} into the second term of \eqref{eq:order2-level-by-level} yields:
\begin{align}
\notag
    2\sum_{1\leq \lvl < \alt{\lvl} \leq \nLvls} &\sum_{\substack{
			\class\atlvl{\lvl}\in\classes\atlvl{\lvl}
	\\
\class\atlvl{\alt{\lvl}}\desc_{\alt{\lvl}}\classes\atlvl{\lvl}}} 
	\strat_{\class\atlvl{\alt{\lvl}}}\exof*{\incrmodel_{\class\atlvl{\lvl}}\incrmodel_{\class\atlvl{\alt{\lvl}}}}
	\\
\notag
	&\leq \sum_{1\leq \lvl < \alt{\lvl} \leq \nLvls} \sum_{\substack{
			\class\atlvl{\lvl}\in\classes\atlvl{\lvl}
	\\
\class\atlvl{\alt{\lvl}}\desc_{\alt{\lvl}}\classes\atlvl{\lvl}}} 
	\strat_{\class\atlvl{\alt{\lvl}}}\parens*{\frac{1}{\epsilon_{\lvl, \alt{\lvl}}}\exof*{\incrmodel_{\class\atlvl{\alt{\lvl}}}^{2}} 
    + \epsilon_{\lvl, \alt{\lvl}} \exof*{\incrmodel_{\class\atlvl{\lvl}}^{2}}}
    \\
\notag
    &= \sum_{1\leq \lvl < \alt{\lvl} \leq \nLvls}
    \frac{1}{\epsilon_{\lvl, \alt{\lvl}}}
    \sum_{\substack{
			\class\atlvl{\lvl}\in\classes\atlvl{\lvl}
	\\
\class\atlvl{\alt{\lvl}}\desc_{\alt{\lvl}}\classes\atlvl{\lvl}}}
    \strat_{\class\atlvl{\alt{\lvl}}}\exof*{\incrmodel_{\class\atlvl{\alt{\lvl}}}^{2}}
    + \sum_{1\leq \lvl < \alt{\lvl} \leq \nLvls}
     \epsilon_{\lvl, \alt{\lvl}}
     \sum_{\substack{
			\class\atlvl{\lvl}\in\classes\atlvl{\lvl}
	\\
\class\atlvl{\alt{\lvl}}\desc_{\alt{\lvl}}\classes\atlvl{\lvl}}}
    \strat_{\class\atlvl{\alt{\lvl}}}\exof*{\incrmodel_{\class\atlvl{\lvl}}^{2}}
    \\
\notag
    &= \sum_{1\leq \lvl < \alt{\lvl} \leq \nLvls}
    \frac{1}{\epsilon_{\lvl, \alt{\lvl}}}
    \sum_{\substack{
			\class\atlvl{\lvl}\in\classes\atlvl{\lvl}
	\\
\class\atlvl{\alt{\lvl}}\desc_{\alt{\lvl}}\classes\atlvl{\lvl}}}
    \strat_{\class\atlvl{\alt{\lvl}}}\exof*{\incrmodel_{\class\atlvl{\alt{\lvl}}}^{2}}
    + \sum_{1\leq \lvl < \alt{\lvl} \leq \nLvls}
     \epsilon_{\lvl, \alt{\lvl}}
     \sum_{\class\atlvl{\lvl}\in\classes\atlvl{\lvl}}
     \exof*{\incrmodel_{\class\atlvl{\lvl}}^{2}}
     \underbrace{\sum_{\class\atlvl{\alt{\lvl}}\desc_{\alt{\lvl}}\classes\atlvl{\lvl}} \strat_{\class\atlvl{\alt{\lvl}}}}_{\strat_{\class\atlvl{\lvl}}}
     \\
\notag
    &= \sum_{1\leq \lvl < \alt{\lvl} \leq \nLvls}
    \frac{1}{\epsilon_{\lvl, \alt{\lvl}}}
    \sum_{\class\atlvl{\alt{\lvl}}\in\classes\atlvl{\alt{\lvl}}}
    \strat_{\class\atlvl{\alt{\lvl}}}\exof*{\incrmodel_{\class\atlvl{\alt{\lvl}}}^{2}}
    + \sum_{1\leq \lvl < \alt{\lvl} \leq \nLvls}
     \epsilon_{\lvl, \alt{\lvl}}
     \sum_{\class\atlvl{\lvl}\in\classes\atlvl{\lvl}}
     \strat_{\class\atlvl{\lvl}}\exof*{\incrmodel_{\class\atlvl{\lvl}}^{2}}
     \\
\notag
    &\leq \sum_{1\leq \lvl < \alt{\lvl} \leq \nLvls}
    \frac{1}{\epsilon_{\lvl, \alt{\lvl}}}
    \sum_{\class\atlvl{\alt{\lvl}}\in\classes\atlvl{\alt{\lvl}}}
    \range_{\class\atlvl{\alt{\lvl}}}^{2}
    + \sum_{1\leq \lvl < \alt{\lvl} \leq \nLvls}
     \epsilon_{\lvl, \alt{\lvl}}
     \sum_{\class\atlvl{\lvl}\in\classes\atlvl{\lvl}}
     \range^{2}_{\class\atlvl{\lvl}}\explain{by \eqref{eq:order-2-incr-bound}}
     \\
\label{eq:order2-bound-second-term}
    &\leq \sum_{1\leq \lvl < \alt{\lvl} \leq \nLvls}
    \frac{1}{\epsilon_{\lvl, \alt{\lvl}}}
    \nClasses\atlvl{\alt{\lvl}}
    \bar{\range}_{\alt{\lvl}}^{2}
    + \sum_{1\leq \lvl < \alt{\lvl} \leq \nLvls}
     \epsilon_{\lvl, \alt{\lvl}}\nClasses\atlvl{\lvl}
     \bar{\range}^{2}_{\lvl}.
\end{align}

Injecting \eqref{eq:order2-bound-first-term} and \eqref{eq:order2-bound-second-term} into \eqref{eq:order2-level-by-level} ensures that:
\begin{equation*}
    \exof*{\insum_{\elem\in\elems}\strat_{\elem} \model_{\elem}^{2}} \leq \insum_{\lvl=1}^{\nLvls} \nClasses\atlvl{\lvl} \bar{\range}^{2}_{\lvl} + \insum_{1\leq \lvl < \alt{\lvl} \leq \nLvls}
    \parens*{
    \frac{1}{\epsilon_{\lvl, \alt{\lvl}}}
    \nClasses\atlvl{\alt{\lvl}}
    \bar{\range}_{\alt{\lvl}}^{2}
    + \epsilon_{\lvl, \alt{\lvl}}\nClasses\atlvl{\lvl}
    \bar{\range}^{2}_{\lvl}}
\end{equation*}
holds for any sequence of positive numbers $\{\epsilon_{\lvl, \alt{\lvl}}\}_{1\leq\alt{\lvl} < \lvl\leq\nLvls}$.
As a result, taking $\epsilon_{\lvl, \alt{\lvl}} = \sqrt{\frac{\nClasses\atlvl{\alt{\lvl}}}{\nClasses\atlvl{\lvl}}}\frac{\bar{\range}_{\alt{\lvl}}}{\bar{\range}_{\lvl}}$ yields the tight bound
\begin{equation}
    \label{order2-bound-all-before-tuning}
    \exof*{\insum_{\elem\in\elems}\strat_{\elem} \model_{\elem}^{2}} \leq \insum_{\lvl=1}^{\nLvls} \nClasses\atlvl{\lvl} \bar{\range}^{2}_{\lvl} + 2\insum_{1\leq \lvl < \alt{\lvl} \leq \nLvls}
    \sqrt{\nClasses\atlvl{\alt{\lvl}}}\bar{\range}_{\alt{\lvl}}\sqrt{\nClasses\atlvl{\lvl}}\bar{\range}^{2}_{\lvl} = \parens*{\insum_{\lvl=1}^{\nLvls} \sqrt{\nClasses}\atlvl{\lvl}\bar{\range}_{\lvl}}^{2},
\end{equation}
which proves our original assertion.
\end{proof}


\section{Regret analysis}
\label{app:regret}

As we mentioned in the main text, the principal component of our analysis is a recursive inequality which, when telescoped over $\run=\running$, will yield the desired regret bound.
To establish this ``template inequality'', we will first require an energy function measuring the disparity between a benchmark strategy $\strat\in\simplex(\elems)$ and a propensity score profile $\score \in \R^{\elems}$.
To that end, building on the notions introduced in \cref{app:entropy}, let $\hreg\from\simplex(\elems)\to\R$ denote the total nested entropy function
\begin{alignat}{2}
\hreg(\strat)
	&= \hreg_{\source}(\strat)
	= \sum_{\lvlalt=0}^{\nLvls}
		\diff\atlvl{\lvlalt} \sum_{\class\atlvl{\lvlalt} \in \classes\atlvl{\lvlalt}}
		\strat_{\class\atlvl{\lvlalt}} \log\strat_{\class\atlvl{\lvlalt}},
	&\quad
	&\text{$\strat\in\simplex(\elems)$},
\shortintertext{and let}
\hconj(\score)
	&= \max_{\strat\in\simplex(\elems)}
		\braces*{\braket{\score}{\strat} - \hreg(\strat)},
	&\quad
	&\text{$\score\in\R^{\elems}$},
\end{alignat}
denote the convex conjugate of $\hreg$ so, by \cref{prop:mirror}, we have
\begin{equation}
\label{eq:mirror}
\hconj(\score)
	= \score_{\source}
	\quad
	\text{and}
	\quad
\choice[\elem](\score)
	= \frac{\pd\hconj}{\pd\score_{\elem}}
	\quad
	\text{for all $\score\in\R^{\elems}$}.
\end{equation}
The \emph{Fenchel coupling} between $\strat\in\simplex(\elems)$ and $\score\in\R^{\elems}$ is then defined as
\begin{equation}
\label{eq:Fench}
\fench(\strat,\score)
	= \hreg(\strat) + \hconj(\score) - \braket{\score}{\strat}
	\quad
	\text{for all $\strat\in\simplex(\elems)$, $\score\in\R^{\elems}$},
\end{equation}
and we have the following key result:

\begin{proposition}
\label{prop:Fench}
Let $\struct = \coprod_{\lvl=0}^{\nLvls} \classes\atlvl{\lvl}$ be a similarity structure on $\elems$ with uncertainty parameters $\temp\atlvl{1} \geq \dotsm \geq \temp\atlvl{\nLvls} > 0$.
Then:
\begin{enumerate}
\item
The Fenchel coupling \eqref{eq:Fench} is positive-definite, \ie
\begin{equation}
\fench(\strat,\score)
	\geq 0
	\qquad
	\text{for all $\strat\in\simplex(\elems)$ and all $\score\in\R^{\elems}$},
\end{equation}
with equality if and only if $\strat$ is given by \eqref{eq:NLC}, \ie if and only if $\strat = \choice(\score)$.
\item
For all $\strat\in\elems$, we have
\begin{equation}
\fench(\strat,0)
	= \hreg(\strat) + \hconj(0)
	= \hreg(\strat) - \min\hreg
\end{equation}
where $\min\hreg \equiv \min_{\stratalt\in\simplex(\elems)} \hreg(\stratalt)$ denotes the minimum of $\hreg$ over $\simplex(\elems)$.
\end{enumerate}
\end{proposition}

\begin{proof}
Our first claim follows by setting $\class \gets \source$ in \cref{prop:nest2cond,prop:mirror} and noting that $\hreg_{\class} = \hreg_{\vert\class}$ when $\class = \source$:
indeed, by Young's inequality, we have $\hreg(\strat) + \hconj(\score) - \braket{\score}{\strat} \geq 0$ with equality if and only if $\score \in \subd\hreg(\strat)$, so the equality $\strat = \choice(\score)$ follows from \eqref{eq:Legendre} applied to $\class \gets \source$ and the fact that $\choice[\elem \vert \source](\score) = \choice[\elem](\score)$.
As for our second claim, simply note that
\(
\hconj(0)
	= \max_{\strat\in\simplex(\elems)}
		\braces*{\braket{0}{\strat} - \hreg(\strat)}
	= - \min_{\strat\in\simplex(\elems)} \hreg(\strat)
\)
and set $\score \gets 0$ in the definition \eqref{eq:Fench} of the Fenchel coupling.
\end{proof}

With all this in hand, the specific energy function that we will use for our regret analysis is the ``\emph{rate-deflated}'' Fenchel coupling
\begin{equation}
\label{eq:energy}
\curr[\energy]
	= \frac{1}{\curr[\learn]}
		\fench(\test,\curr[\learn]\curr[\dstate])
\end{equation}
where
$\test\in\simplex(\elems)$ is the regret comparator,
$\curr[\learn]$ is the algorithm's learning rate at stage $\run$,
and
$\curr[\dstate]$ is the corrsponding propensity score estimate.
In words, since the mixed strategy employed by the learner at stage $\run$ is $\curr = \choice(\curr[\learn]\curr[\dstate])$, the energy $\curr[\energy]$ essentially measures the disparity between $\curr$ and the target strategy $\test$ (suitably rescaled by the method's learning rate).
We then have the following fundamental estimate:

\begin{proposition}
\label{prop:energy}
For all $\test\in\simplex(\elems)$ and all $\run=\running$, we have:
\begin{equation}
\label{eq:template}
\energy_{\run+1}
	\leq \curr[\energy]
		+ \braket{\model_{\run}}{\curr - \test}
		+ \parens{\next[\learn]^{-1} - \curr[\learn]^{-1}} \bracks{\hreg(\test) - \min\hreg}
		+ \frac{1}{\curr[\learn]} \fench(\curr,\curr[\learn]\next[\dstate]).
\end{equation}
\end{proposition}

\begin{proof}
By the definition of $\curr[\energy]$, we have
\begin{subequations}
\label{eq:energy-basic}
\begin{align}
\energy_{\run+1} - \curr[\energy]
	= \frac{1}{\next[\learn]} \fench(\test,\next[\learn]\next[\dstate])
		- \frac{1}{\curr[\learn]} \fench(\test,\curr[\learn]\curr[\dstate])
	&\label{eq:energy-const}
	= \frac{1}{\next[\learn]} \fench(\test,\next[\learn]\next[\dstate])
		- \frac{1}{\curr[\learn]} \fench(\test,\curr[\learn]\next[\dstate])
	\\
	&\label{eq:energy-update}
	+ \frac{1}{\curr[\learn]} \fench(\test,\curr[\learn]\next[\dstate])
		- \frac{1}{\curr[\learn]} \fench(\test,\curr[\learn]\curr[\dstate]).
\end{align}
\end{subequations}
We now proceed to upper-bound each of the two terms \eqref{eq:energy-const} and \eqref{eq:energy-update} separately.

For the term \eqref{eq:energy-const}, the definition of the Fenchel coupling \eqref{eq:Fench} readily yields:
\begin{align}
\eqref{eq:energy-const}
	&= \bracks*{\frac{1}{\next[\learn]} - \frac{1}{\curr[\learn]}} \hreg(\test)
		+ \frac{1}{\next[\learn]} \hconj(\next[\learn]\next[\dstate]) - \frac{1}{\curr[\learn]} \hconj(\curr[\learn]\next[\dstate]).
\end{align}
Inspired by a trick of \citet{Nes09}, consider the function $\varphi(\learn) = \learn^{-1} [\hconj(\learn\score) + \min\hreg]$.
Then, by \cref{prop:mirror}, letting $\strat = \choice(\learn\score)$ and differentiating $\varphi$ with respect to $\learn$ gives
\begin{align}
\varphi'(\learn)
	&= \frac{1}{\learn} \braket{\score}{\choice(\learn\score)}
		-\frac{1}{\learn^{2}} \bracks{ \hconj(\learn\score) + \min\hreg }
	\notag\\
	&= \frac{1}{\learn^{2}} \bracks{ \braket{\learn\score}{\strat} - \hconj(\learn\score) - \min\hreg}
	\notag\\
	&= \frac{1}{\learn^{2}} \bracks{\hreg(\strat) - \min\hreg}
	\geq 0.
\end{align}

Since $\next[\learn] \leq \curr[\learn]$, the above shows that $\varphi(\curr[\learn]) \geq \varphi(\next[\learn])$.
Accordingly, setting $\score \gets \next[\dstate]$ in the definition of $\varphi$ yields
\begin{equation}
\frac{1}{\next[\learn]} \hconj(\next[\learn]\next[\dstate])
	- \frac{1}{\curr[\learn]} \hconj(\curr[\learn]\next[\dstate])
	\leq \bracks*{\frac{1}{\curr[\learn]} - \frac{1}{\next[\learn]}} \min\hreg
\end{equation}
and hence
\begin{equation}
\label{eq:hconj-delta}
\eqref{eq:energy-const}
	\leq \parens{\next[\learn]^{-1} - \curr[\learn]^{-1}} \bracks{\hreg(\test) - \min\hreg}.
\end{equation}

Now, after a straightforward rearrangement, the second term of \eqref{eq:energy-basic} becomes
\begin{align}
\eqref{eq:energy-update}
	&= \frac{1}{\curr[\learn]}
		\bracks*{\hreg(\test) + \hconj(\curr[\learn]\next[\dstate]) - \curr[\learn]\braket{\next[\dstate]}{\test}}
	- \frac{1}{\curr[\learn]}
		\bracks*{\hreg(\test) + \hconj(\curr[\learn]\curr[\dstate]) - \curr[\learn]\braket{\curr[\dstate]}{\test}}
	\notag\\
	&= \frac{1}{\curr[\learn]}
		\bracks*{
			\hconj(\curr[\learn]\next[\dstate])
			- \hconj(\curr[\learn]\curr[\dstate])
			- \curr[\learn]\braket{\curr[\model]}{\test}
				}
	\explain{by \eqref{eq:NEW}}
	\\
	&= \frac{1}{\curr[\learn]}
		\bracks*{
			\hconj(\curr[\learn]\next[\dstate])
			- \hconj(\curr[\learn]\curr[\dstate])
			- \curr[\learn] \braket{\curr[\model]}{\curr}
				}
	+ \braket{\curr[\model]}{\curr - \test}
	\explain{isolate benchmark}
	\\
	&= \frac{1}{\curr[\learn]}
		\bracks*{
			\hconj(\curr[\learn]\next[\dstate])
			- \braket{\curr[\learn]\curr[\dstate]}{\curr}
			+ \hreg(\curr)
			- \curr[\learn] \braket{\curr[\model]}{\curr}
				}
	+ \braket{\curr[\model]}{\curr - \test}
	\explain{by \cref{prop:mirror}}
	\\
	&= \frac{1}{\curr[\learn]} \fench(\curr,\curr[\learn]\next[\dstate])
		+ \braket{\curr[\model]}{\curr - \test}
\end{align}

Thus, combining the above with \eqref{eq:hconj-delta}, we finally obtain
\begin{align}
\label{eq:energy-inter}
\next[\energy]
	&= \curr[\energy]
		+ \eqref{eq:energy-const}
		+ \eqref{eq:energy-update}
	\notag\\
	&\leq \curr[\energy]
		+ \parens{\next[\learn]^{-1} - \curr[\learn]^{-1}} \bracks{\hreg(\test) - \min\hreg}
		+ \braket{\model_{\run}}{\curr - \test}
		+ \frac{1}{\curr[\learn]} \fench(\curr,\curr[\learn]\next[\dstate])
\end{align}
and our proof is complete.
\end{proof}

We are now in a position to state and prove the template inequality that provides the scaffolding for our regret bounds:

\template*

\begin{proof}
Let $\curr[\noise] = \curr[\model] - \curr[\payv]$ denote the error in the learner's estimation of the $\run$-th stage payoff vector $\curr[\payv]$.
Then, by substituting in \cref{prop:energy} and rearranging, we readily get:
\begin{equation}
\label{eq:template1}
\braket{\curr[\payv]}{\test - \curr}
	\leq \curr[\energy] - \energy_{\run+1}
		+ \braket{\noise_{\run}}{\curr - \test}
		+ \parens*{\next[\learn]^{-1} - \curr[\learn]^{-1}} \bracks{\hreg(\test) - \min\hreg}
		+ \curr[\learn] \fench(\test,\curr[\learn]\next[\dstate])
\end{equation}
Thus, telescoping over $\run=\running,\nRuns$, we have
\begin{align}
\label{eq:template2}
\reg_{\test}(\nRuns)
	&\leq \init[\energy] - \afterlast[\energy]
		+ \parens*{\frac{1}{\afterlast[\learn]} - \frac{1}{\init[\learn]}} \bracks{\hreg(\test) - \min\hreg}
	\notag\\
	&\qquad
		+ \sum_{\run=\start}^{\nRuns} \braket{\noise_{\run}}{\curr - \test}
		+ \sum_{\run=\start}^{\nRuns} \frac{1}{\curr[\learn]} \fench(\curr,\curr[\learn]\next[\dstate])
	\notag\\
	&\leq \frac{\hreg(\test) - \min\hreg}{\afterlast[\learn]}
		+ \sum_{\run=\start}^{\nRuns} \braket{\noise_{\run}}{\curr - \test}
		+ \sum_{\run=\start}^{\nRuns} \frac{1}{\curr[\learn]} \fench(\curr,\curr[\learn]\next[\dstate])
\end{align}
where we used the fact that
\begin{enumerate*}
[\itshape a\upshape)]
\item
$\curr[\energy] \geq 0$ for all $\run$ (a consequence of the first part of \cref{prop:Fench});
and that
\item
$\init[\energy] = \init[\learn]^{-1} \bracks{\hreg(\test) + \hconj(0)} = \init[\learn]^{-1} \bracks{\hreg(\test) - \min\hreg}$
\end{enumerate*}
(from the second part of the same proposition).
Our claim then follows by taking expectations in \eqref{eq:template2} and noting that $\exof{\curr[\noise] \given \curr[\filter]} = 0$
(by \cref{prop:NIWE}).
\end{proof}

In view of the above, our main regret bound follows by bounding the two terms in the template inequality \eqref{eq:template}.
The second term is by far the most difficult one to bound, and is where \cref{app:aux} comes in;
the first term is easier to handle, and it can be bounded as follows:

\begin{lemma}
\label{lem:hmin}
Suppose that each class $\class \in \classes\atlvl{\lvl-1}$ has at most $\nChildren\atlvl{\lvl}$ children, $\lvl = 1,\dotsc,\nLvls$.
Then, for all $\test\in\simplex(\elems),$ we have 
\begin{align}
\label{eq:hrange-bound-general}
 \hrange &\leq \sum_{\lvl=1}^{\nLvls} \temp\atlvl{\lvl} \log\nChildren\atlvl{\lvl} &\text{with equality iff the tree is symmetric,}\\
 \label{eq:hrange-equality-equal-mus}
 \hrange &= \temp \log(\nElems) &\text{if $\temp\atlvl{1} = \temp\atlvl{2} = \dots = \temp\atlvl{\nLvls} = \temp$}.
\end{align}

\end{lemma}

\begin{proof}
Suppose that $\score_{\elem} = 0$ for all $\elem\in\elems$.
Then, applying \eqref{eq:score} inductively, we have:
\begin{equation}
\begin{alignedat}{2}
\score_{\class\atlvl{\nLvls}}
	&= 0
	&\qquad
	&\text{for all $\class\atlvl{\nLvls} \in \classes\atlvl{\nLvls}$}
	\\
\score_{\class\atlvl{\nLvls-1}}
	&= \temp\atlvl{\nLvls}\,
		\log \;\;\smashoperator{\sum_{\class\atlvl{\nLvls} \childof \class\atlvl{\nLvls-1}}}\;\;
		\exp(0)
	\leq \temp\atlvl{\nLvls} \log\nChildren\atlvl{\nLvls}
	&\qquad
	&\text{for all $\class\atlvl{\nLvls-1} \in \classes\atlvl{\nLvls-1}$}
	\\
\score_{\class\atlvl{\nLvls-2}}
	&= \temp\atlvl{\nLvls-1}\,
		\log \;\smashoperator{\sum_{\class\atlvl{\nLvls-1} \childof \class\atlvl{\nLvls-2}}}\;
		\exp\parens*{\frac{\score_{\class\atlvl{\nLvls-1}}}{\temp\atlvl{\nLvls-1}}}
	\leq \temp\atlvl{\nLvls-1} \log\nChildren\atlvl{\nLvls-1}
		+ \temp\atlvl{\nLvls} \log\nChildren\atlvl{\nLvls}
	&\qquad
	&\text{for all $\class\atlvl{\nLvls-2} \in \classes\atlvl{\nLvls-2}$}
	\\
	&\;\;\vdots
	&
	&\;\;\vdots
	\\
\score_{\class\atlvl{\lvl-1}}
	&= \temp\atlvl{\lvl}\,
		\log \;\;\smashoperator{\sum_{\class\atlvl{\lvl} \childof \class\atlvl{\lvl-1}}}\;\;
		\exp(\score_{\class\atlvl{\lvl}} / \temp\atlvl{\lvl})
	\leq \sum_{\lvlalt=\lvl}^{\nLvls}
		\temp\atlvl{\lvlalt} \log\nChildren\atlvl{\lvlalt}
	&\qquad
	&\text{for all $\class\atlvl{\lvl-1} \in \classes\atlvl{\lvl-1}$}
\end{alignedat}
\end{equation}
and hence $\hrange = \hconj(0) = \score_{\source} \leq \sum_{\lvl=1}^{\nLvls} \temp\atlvl{\lvl} \log\nChildren\atlvl{\lvl}$.
\cref{eq:hrange-bound-general} then follows from \cref{prop:Fench}.

Now, if $\temp\atlvl{1} = \temp\atlvl{2} = \dots = \temp\atlvl{\nLvls} = \temp$, we have
\begin{align}
\hrange
	&= \log\bracks*{
		\sum_{\class\atlvl{1} \childof \class\atlvl{0}}
		\bracks*{
			\sum_{\class\atlvl{2} \childof \class\atlvl{1}}
				\!\dotsi
				\bracks*{
					\sum_{\class\atlvl{\nLvls}\childof\class\atlvl{\nLvls-1}}
						\!\!\!\!1
				}^{\frac{\temp\atlvl{\nLvls}}{\temp\atlvl{\nLvls-1}}}
			\!\!\!\!\!\dotsi\,
		}^{\frac{\temp\atlvl{2}}{\temp\atlvl{1}}}
	}^{\temp\atlvl{1}}
	\notag\\
	&=\temp \log\sum_{\class\atlvl{1} \childof \class\atlvl{0}}
		\bracks*{
			\sum_{\class\atlvl{2} \childof \class\atlvl{1}}
				\dotsi
				\bracks*{
					\sum_{\class\atlvl{\nLvls}\childof\class\atlvl{\nLvls-1}}
						\!\!\!\!1
				}
		\dotsi
		}
	\notag\\
	&= \temp \log \bracks*{\sum_{\class\atlvl{\nLvls} \childof_{\nLvls} \class\atlvl{0}} 1}
	= \temp \log\nElems,
\end{align}
which proves \cref{eq:hrange-equality-equal-mus} and completes our proof.
\end{proof}

\begin{proposition}
\label{prop:fenchel-bound}
For all $\test\in\simplex(\elems)$ and all $\run=\{\running\}$, we have:
\begin{equation}
\label{eq:fenchel-bound}
\fench(\curr[\state], \curr[\temp]\next[\dstate]) + \curr[\learn]\braket{\curr[\model]}{\curr[\state]}
= \hconj(\curr[\learn]\curr[\dstate] + \curr[\learn]\curr[\model]) -  \hconj(\curr[\learn]\curr[\dstate]).
\end{equation}
\end{proposition}

\begin{proof}
Let $\test\in\simplex(\elems)$ and $\run\in\running$, we simply write:
\begin{align}
\fench(\curr[\state], \curr[\learn]\next[\dstate]) 
&= \hreg(\curr[\state]) + \hconj(\curr[\learn]\next[\dstate]) - \curr[\learn]\braket{\next[\dstate]}{\curr[\state]}
	\notag\\
	&= \underbrace{\hreg(\curr[\state]) + \hconj(\curr[\learn]\curr[\dstate]) - \braket{\curr[\learn]\curr[\dstate]}{\curr[\state]}}_{=\fench(\curr[\state], \curr[\learn]\curr[\dstate])} + \hconj(\curr[\learn]\next[\dstate]) - \hconj(\curr[\learn]\curr[\dstate]) - \curr[\learn]\braket{\curr[\model]}{\curr[\state]}
	\notag\\
	&= \hconj(\curr[\learn]\curr[\dstate] + \curr[\learn]\curr[\model]) - \hconj(\curr[\dstate]) - \curr[\learn]\braket{\curr[\model]}{\curr[\state]}
\explain{$\fench(\curr[\state], \curr[\learn]\curr[\dstate])=0$}
\end{align}
and our assertion follows.
\end{proof}

We are finally in a position to prove our main result (which we restate below for convenience):

\NEW*

\begin{proof}
Injecting \cref{eq:fenchel-bound} in the result of \cref{prop:template} and using \cref{prop:hconj-increment-bound} and \cref{eq:varbound-cost} of \cref{prop:NIWE} directly yields the pseudo-regret bound \eqref{eq:reg-NEW}.

Finally, if we choose $\temp\atlvl{1} = \dotsm = \temp\atlvl{\nLvls} = \sqrt{\nEff/2}$, \cref{lem:hmin} gives
\begin{equation}
    \label{eq:H-bound-equal-mus}
    \hrange = \sqrt{\nEff/2} \log\nElems.
\end{equation}
Thus, taking $\curr[\learn] = \sqrt{\log\nElems/(2\run)}$ and substituting in \eqref{eq:reg-NEW} along with \eqref{eq:H-bound-equal-mus} finally delivers
\begin{equation}
\exof{\reg_{\test}(\nRuns)}
	\leq 2 \sqrt{\nEff \log\nElems \cdot \nRuns},
\end{equation}
and our claim follows.
\end{proof}

\section{Additional Experiment Details and Discussions}
\label{app:numerics}

In this appendix we provide additional details on the experiments as well as further discussions on the settings we presented. The code with the implementation of the algorithms as well as the code to reproduce the figures will be open-sourced and is provided along with the supplementary materials. 

\label{sec:num-app}

\subsection{Experiment additional details}

In the synthetic environment, at each level, the rewards are generated randomly according for each class nodes, through uniform distributions of randomly generated means and fixed bandwidth. From a level $\lvl$ to the next $\lvl+1$, the rewards range are divided by a multiplicative factor $\range\atlvl{\lvl}/\range\atlvl{\lvl+1}=10$. 
The implemented method of NEW uses the reward based IW. Moreover, no model selection was used in this experiment as no hyperparameter was tuned. Indeed, a decaying rate of $\frac{1}{\sqrt{t}}$ was used for the score updates for all methods, as is common in the bandit litterature \citep{LS20}. 

\subsection{Blue Bus/ Red Bus environment}

\label{sec:num-bbrb}

We detail in Figure \ref{fig:blueredbus} a graphical representation of such blue bus/red bus environment, where many colors of the bus item build irrelevant alternatives. In this setting, with few arms, we run the methods up to the horizon $T=1000$. We provide in Figure \ref{fig:app-bbrb} the average reward of the two methods NEW and EXP3 with varying number of subclasses of the ``\texttt{bus}''. 

\begin{figure}[h!]
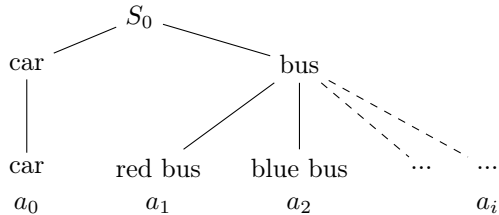

\ctikzfig{Figures/blueredbus}
\caption{Diagram of the blue Bus/Red Bus environment.}
\label{fig:blueredbus}
\end{figure}

While the NEW method ends up selecting the best alternative and having the lowest regret, the EXP3 seems to pick wrong alternative in some experiments, and ends up having higher regret and requiring more iterations to converge to higher average reward. In some of our experiments over the multiple random runs, alternatives of very low sampling probability that were sampled changed the score vector too brutally in the IPS estimator which seemed to hurt the EXP3 method much more than the NEW algorithm. 

\begin{figure}[h!]

\begin{minipage}{.5\textwidth}
\includegraphics[width=0.99\linewidth]{Figures/figure_bbrb_regret.pdf}
\end{minipage}%
\begin{minipage}{.5\textwidth}
\includegraphics[width=0.99\linewidth]{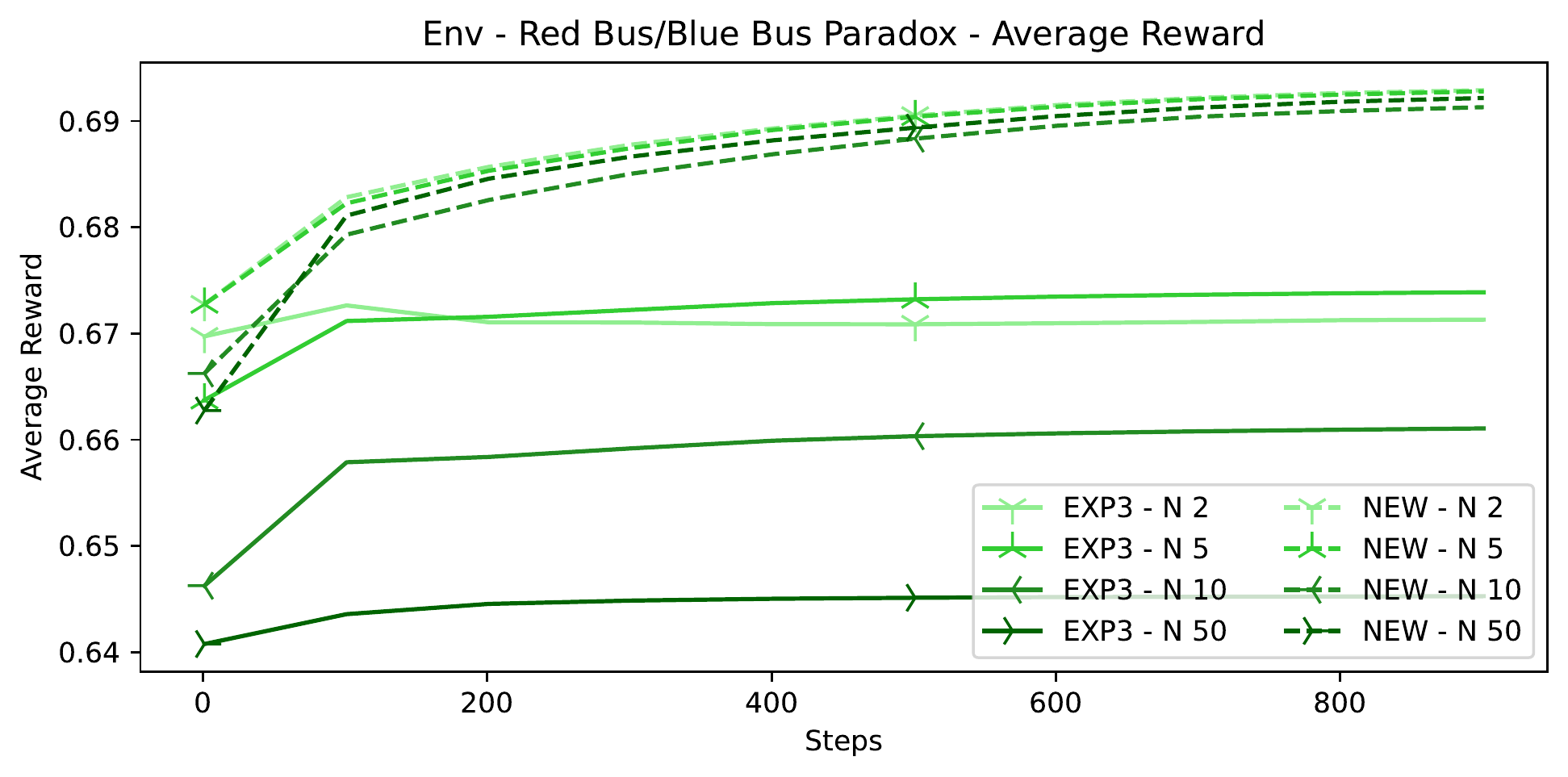}
\end{minipage}%

\caption{Regret and Average Reward of NEW and EXP3 on the Blue Bus/ Red Bus environment.}
\label{fig:app-bbrb}
\end{figure}

\subsection{Tree structures}

In this appendix we show additional results and visualisations for the second setting presented in the main paper. We start with discussions on the depth parameter $\nLvls$ and follow with the breadth parameter related to the number of child per class $M=\vert \class \vert$. 

\paragraph{Influence of the depth parameter $\nLvls$} In Figure \ref{fig:app-depth} we show the influence of the depth parameter with a fixed number of child per class. By making the tree deeper, we illustrate the effect of knowing the nested structure compared to running the logit choice to the whole alternative set. As shown in both the regret and average reward plots, the NEW method outperforms the EXP3 algorithm. While the NEW method also use an IPS estimator, it is less prone to variance issues than the EXP3 method. Indeed, due to the nested structure and the reward decay related to the ratio $\range\atlvl{\lvl+1}/\range\atlvl{\lvl}$, the NEW estimator end up not hurting the regret by still selecting "right" parent classes.

\begin{figure}[h!]
\begin{minipage}{.5\textwidth}
\includegraphics[width=0.99\linewidth]{Figures/figure_general_regret_depth.pdf}
\end{minipage}%
\begin{minipage}{.5\textwidth}
\includegraphics[width=0.99\linewidth]{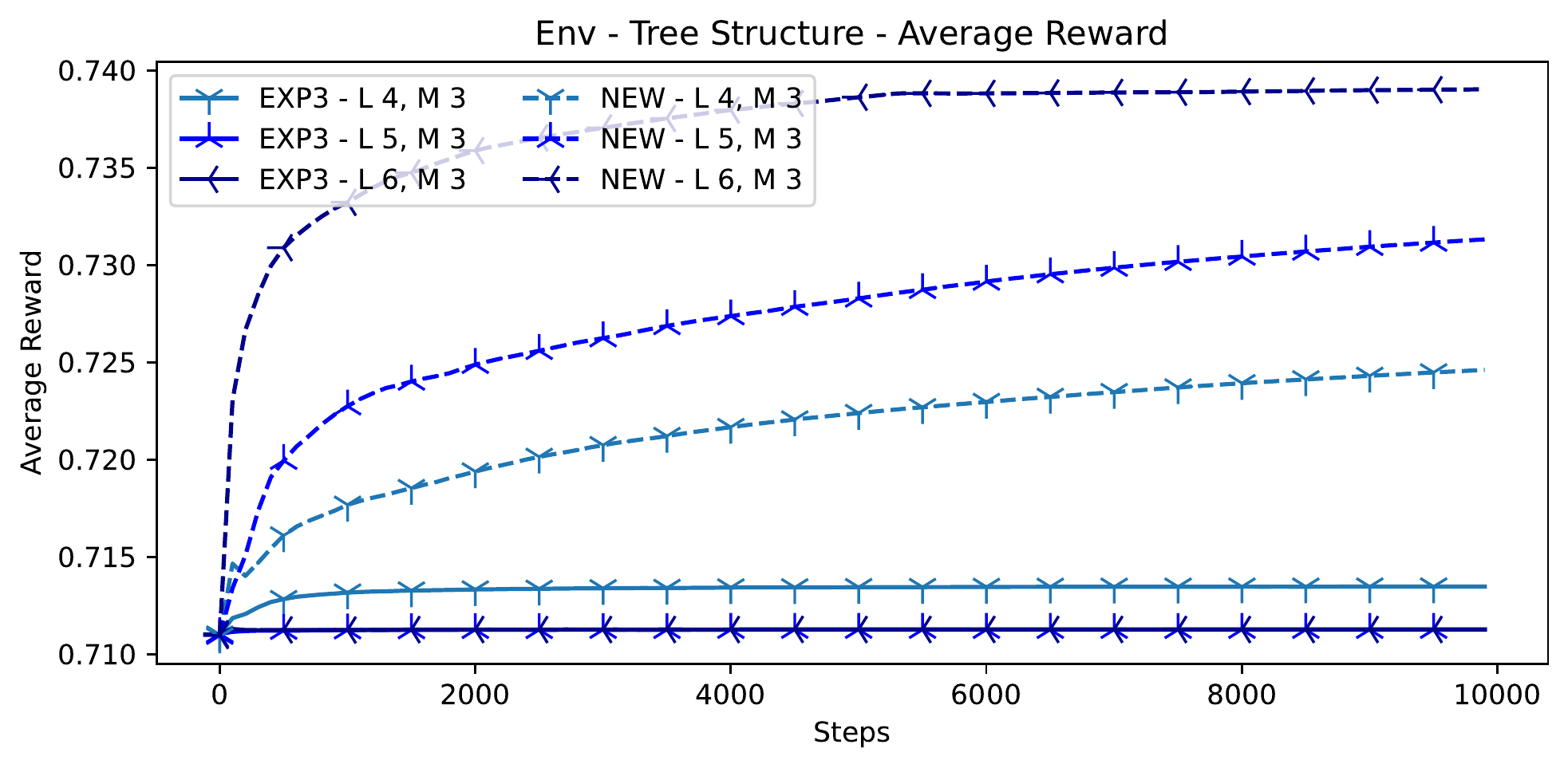}
\end{minipage}%
\caption{Regret and Average Reward of NEW and EXP3 on the synthetic environment with varying number of levels $\nLvls$.}
\label{fig:app-depth}
\end{figure}

\paragraph{Influence of the number of child per class (wideness) $M=\vert \class \vert$} In this setting we fix the number of levels $\nLvls$ and vary the number of child per classes $M$. In Figure \ref{fig:app-breadth} we can see that the NEW method outperforms the EXP3 in terms of regret and average reward. Interestingly, we see that the gap between the two methods shrinks when the number of child per class augments. This is because when the size of a class increase, the NEW method also end up having less knowledge locally and end up having a large number of alternatives to choose among. 

\begin{figure}[h!]
\begin{minipage}{.5\textwidth}
\includegraphics[width=0.99\linewidth]{Figures/figure_general_regret_breadth.pdf}
\end{minipage}%
\begin{minipage}{.5\textwidth}
\includegraphics[width=0.99\linewidth]{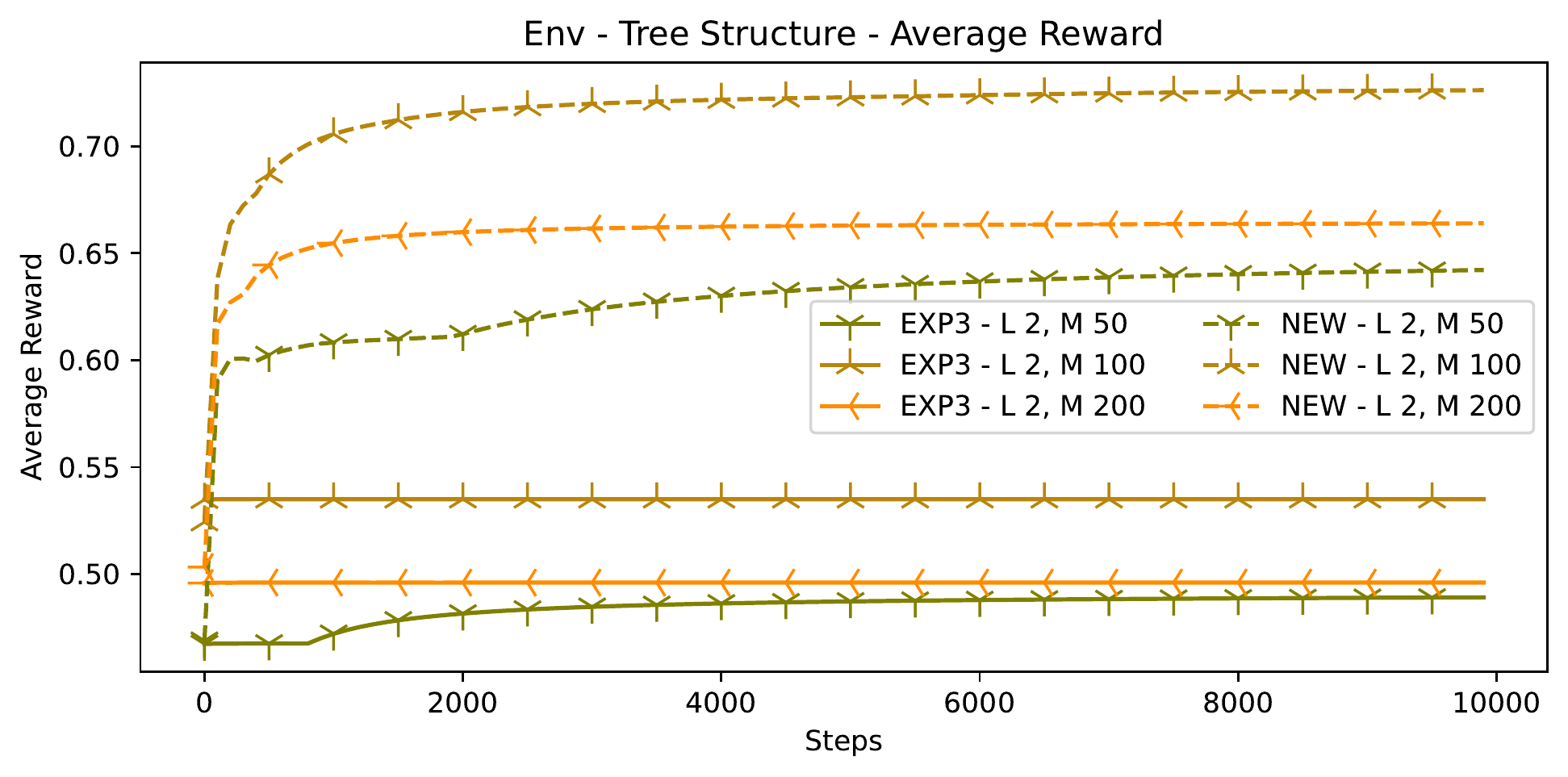}
\end{minipage}%
\caption{Regret and Average Reward of NEW and EXP3 on the synthetic environment with varying number of child per class $M=\vert \class \vert$.}
\label{fig:app-breadth}
\end{figure}

\label{sec:num-tree}

\subsection{A visualisation of the effects of NEW}

In this appendix we want to show the effects of NEW through the simple setting where we assume a nested structure with $\nLvls=4$ and $M=\vert \class \vert=3$. We illustrate in Figure \ref{fig:visualisation-new} the score vectors of the NEW method along the optimal path in the tree (path which nodes have the highest cumulated mean, i.e which generates the highest reward) along with the oracle means of the child nodes. We can see that the algorithm takes advantage of the nested structure and updates the scores vectors optimally with regards to the oracle means of all the nodes. The NEW algorithm therefore estimates correctly the rewards of the environment. 

\begin{figure}[h!]
\begin{minipage}{.245\textwidth}
\includegraphics[width=0.99\linewidth]{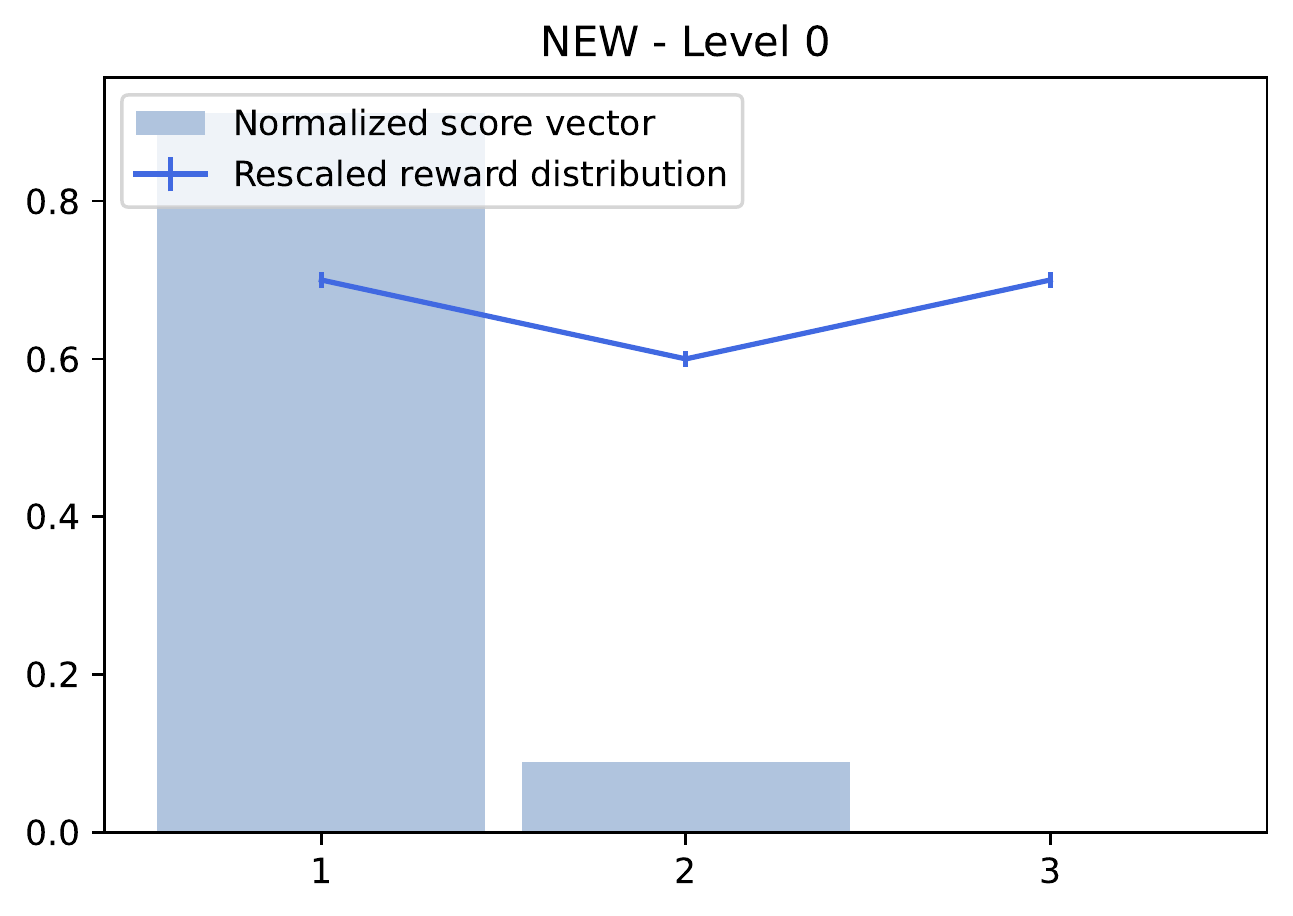}
\end{minipage}%
\begin{minipage}{.245\textwidth}
\includegraphics[width=0.99\linewidth]{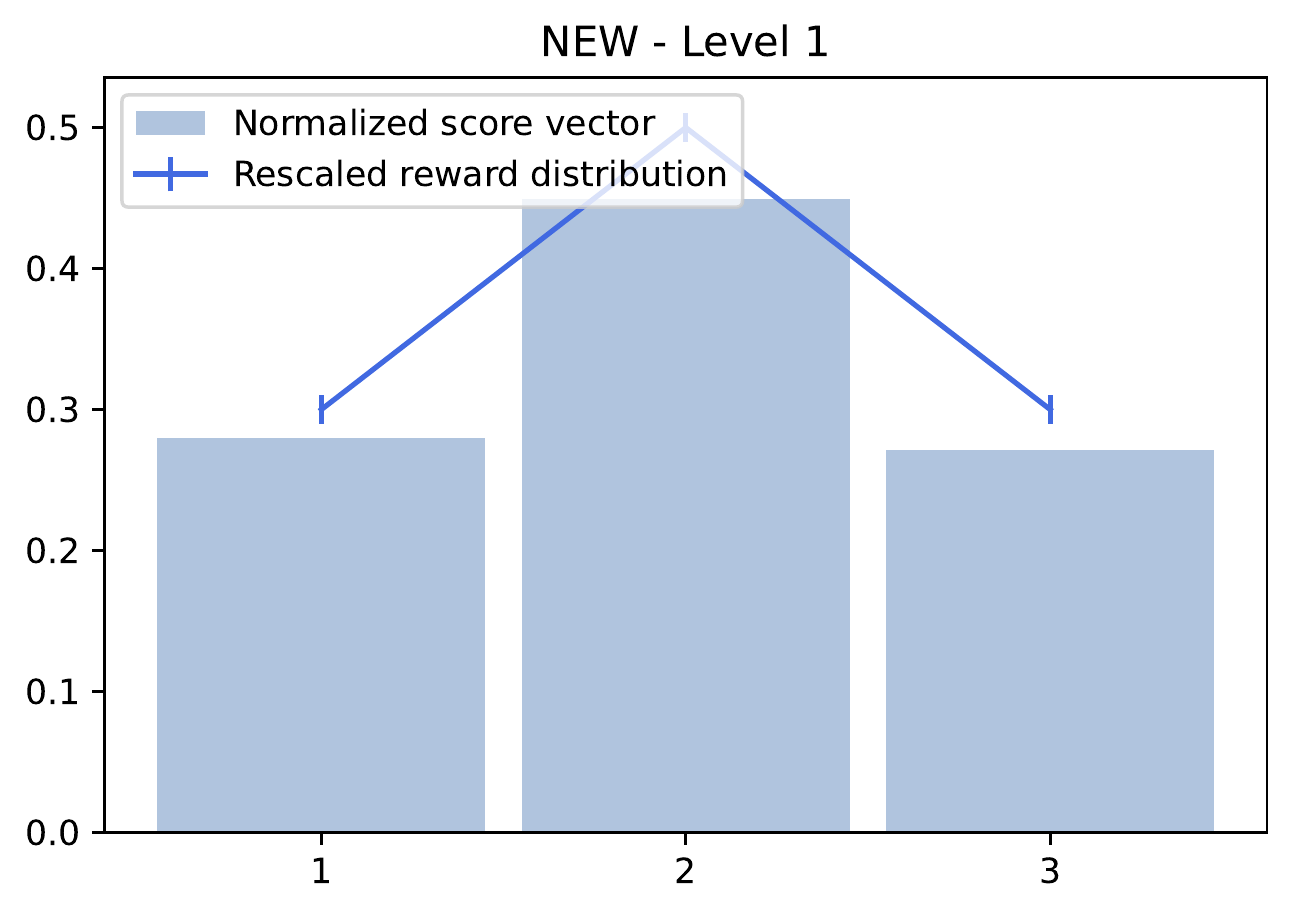}
\end{minipage}%
\begin{minipage}{.245\textwidth}
\includegraphics[width=0.99\linewidth]{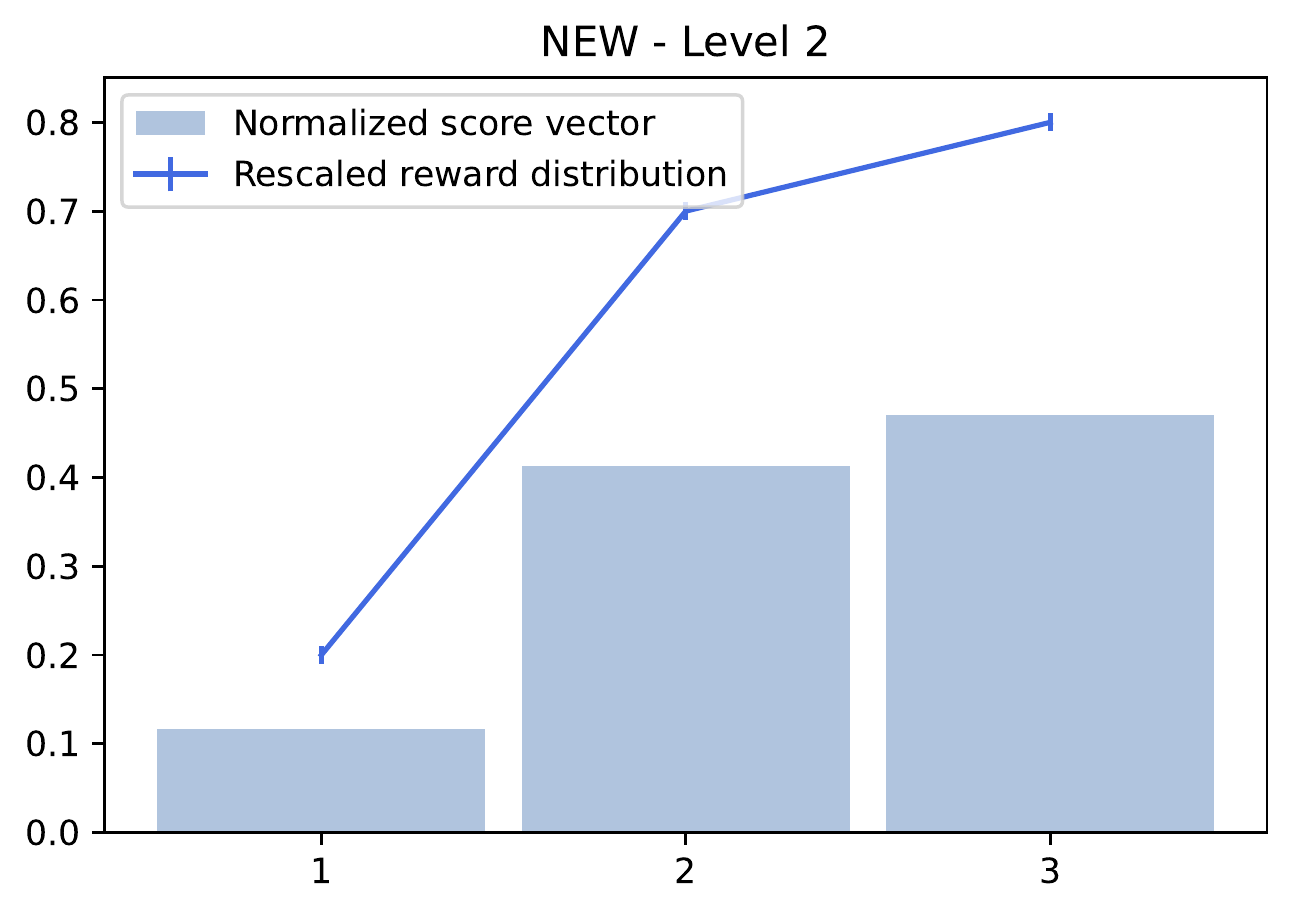}
\end{minipage}%
\begin{minipage}{.245\textwidth}
\includegraphics[width=0.99\linewidth]{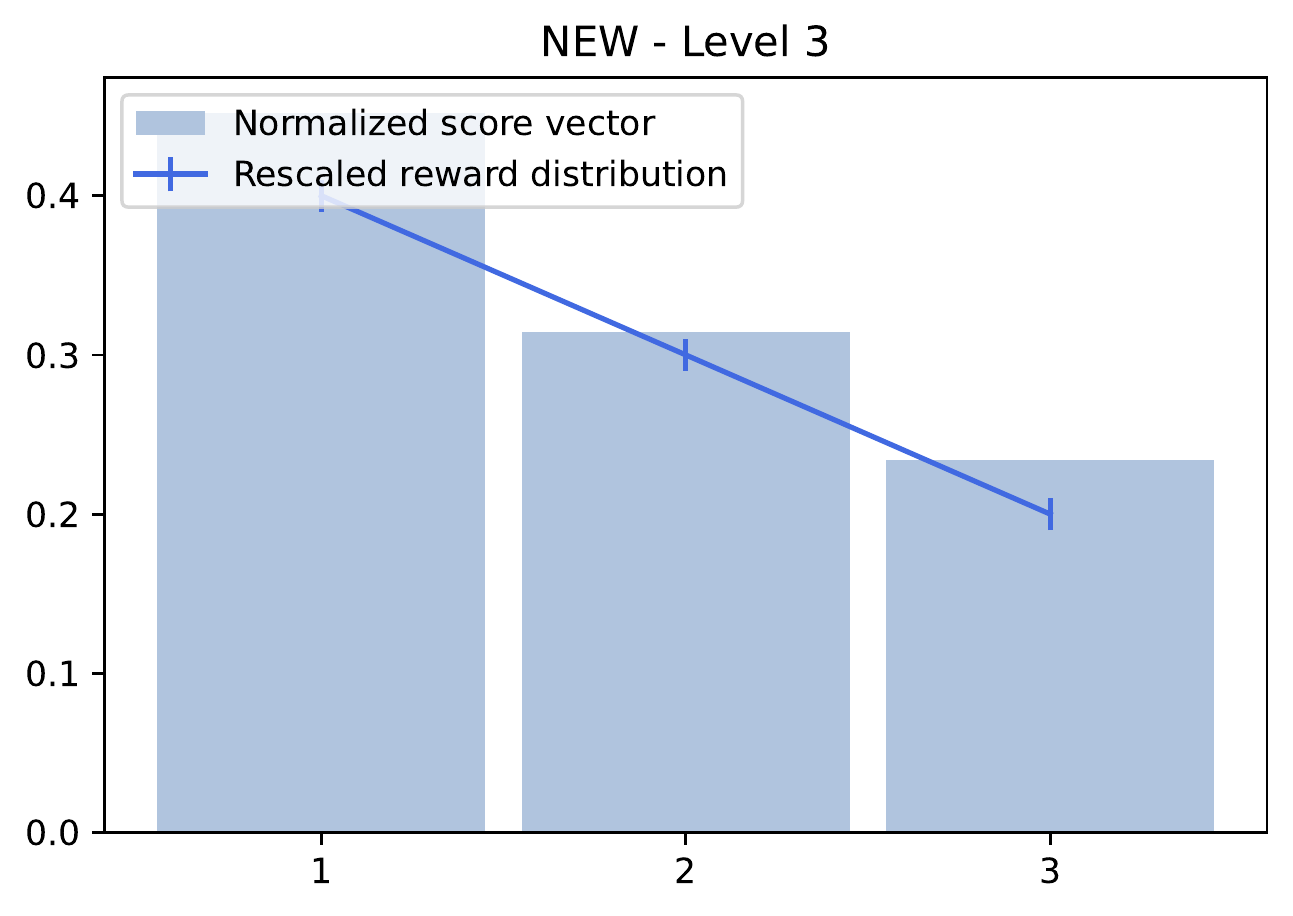}
\end{minipage}%
\caption{Histograms of the score vectors along the optimal path in the nested structure, with visualisation of the mean value of the node.}
\label{fig:visualisation-new}
\end{figure}

Inversely we see in Figure \ref{fig:visualisation-exp3} that the EXP3 method has suffered from variance issue and selected a suboptimal alternative among the $\vert \class \vert^\nLvls = 81$ possible ones. The EXP3 did not take advantage of the nested structure and therefore did not learn as correctly as the NEW algorithm the reward values.

\begin{figure}[h!]
    \centering
    \includegraphics[width=.245\textwidth]{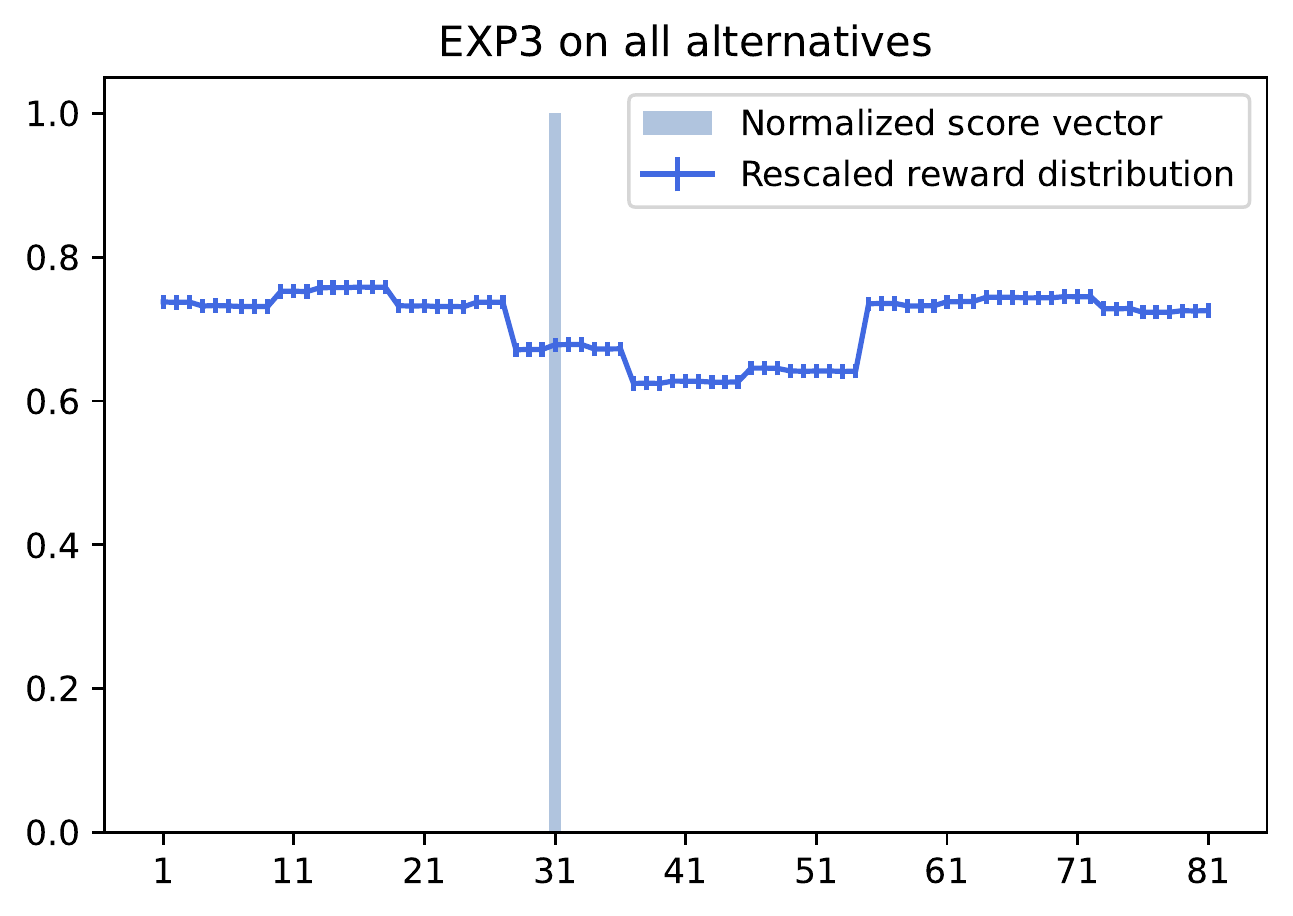}
    \caption{Histogram of the score vector of the all alternatives, with a visualisation of the mean value of all nodes.}
    \label{fig:visualisation-exp3}
\end{figure}

\subsection{Cases where both algorithms perform identically}

\label{sec:num-discussions}

In this appendix we merely show that the implementation of the NEW and EXP3 algorithm match exactly and observe the same behavior when the number of levels $\nLvls$ is set to 1. This setting is where we have no knowledge of any nested structure, therefore both algorithms perform identically in Figure \ref{fig:app-check}.

\begin{figure}[h!]
\begin{minipage}{.5\textwidth}
\includegraphics[width=0.99\linewidth]{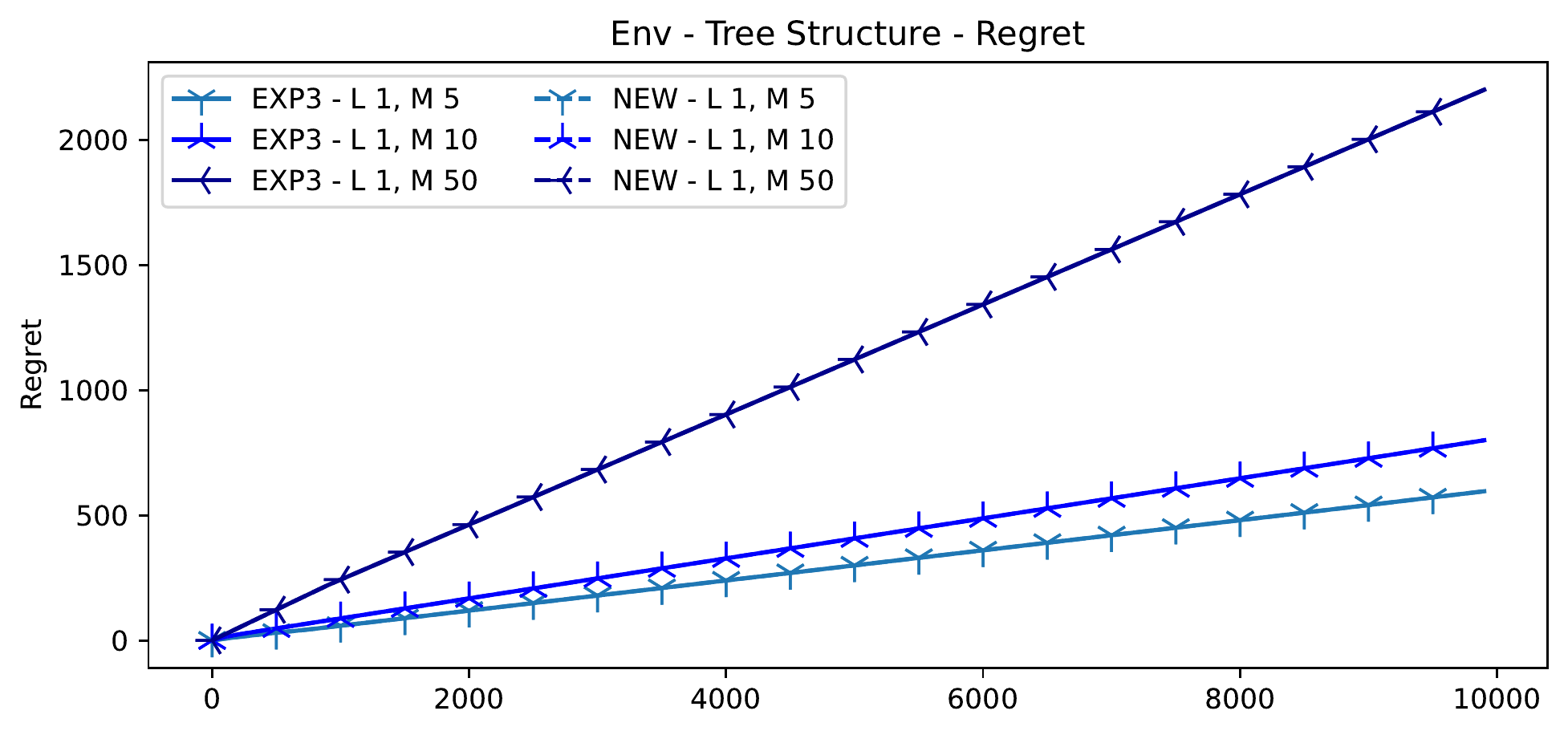}
\end{minipage}%
\begin{minipage}{.5\textwidth}
\includegraphics[width=0.99\linewidth]{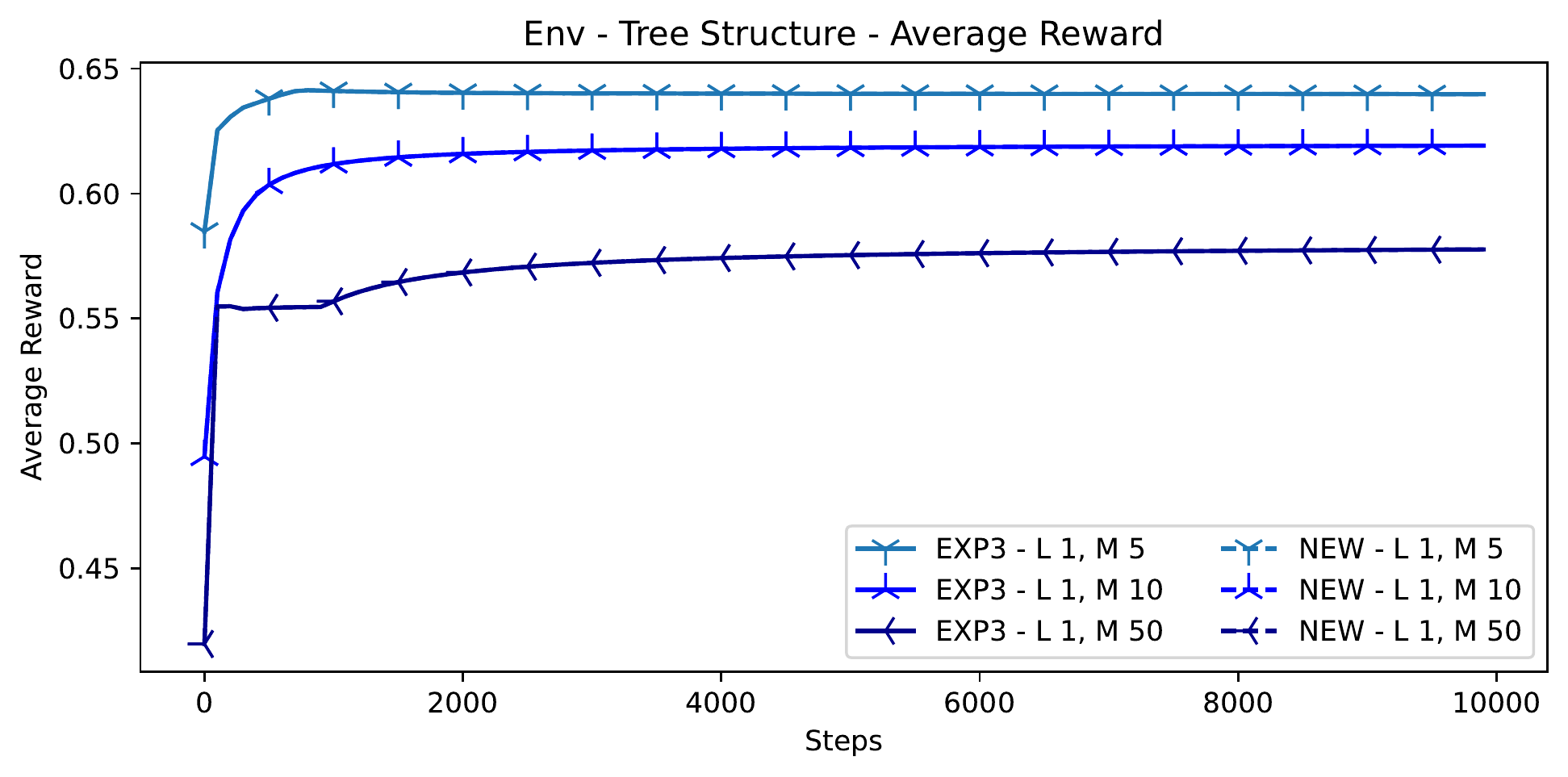}
\end{minipage}%
\caption{Regret and Average Reward of NEW and EXP3 on the synthetic environment where $\nLvls=1$.}
\label{fig:app-check}
\end{figure}


\subsection{Variance plots for the synthetic experiments}

\label{sec:num-monotonic}
We discuss here the variance of the regret at the final timestep $T=10000$. Indeed, as shown on Figure \ref{fig:app-bbrb} for the NEW algorithm , on Figure \ref{fig:app-depth} for both algorithms EXP3 and NEW, and on Figure \ref{fig:app-breadth} for EXP3, some of the plots do no exhibit the monotonicity one would have expected when increasing the number of arms through $\nLvls$ or $M$, and are even overlapping on the regret plot.
This can be explained on Figures \ref{fig:moustache_rbbb} for the Red Bus/Blue Bus environment, and in Figures \ref{fig:moustache_depth} and \ref{fig:moustache_wide} respectively for depth and wideness tree experiments. Those plots show the variances (across the 20 random seeds) of the final regret for both methods at the final step-size. In Figure \ref{fig:moustache_depth} we see that the EXP3 arms have similar mean values with large variances, which explains why they are overlapping on the plot in Figure \ref{fig:tree_regret_depth}. In Figure \ref{fig:moustache_wide} when varying $M$ we can also have a closer look on how NEW outperforms EXP3 and how the close values of NEW regrets through different $M$ can be explained by their high variance.

\begin{figure}[h!]
\includegraphics[width=0.99\linewidth]{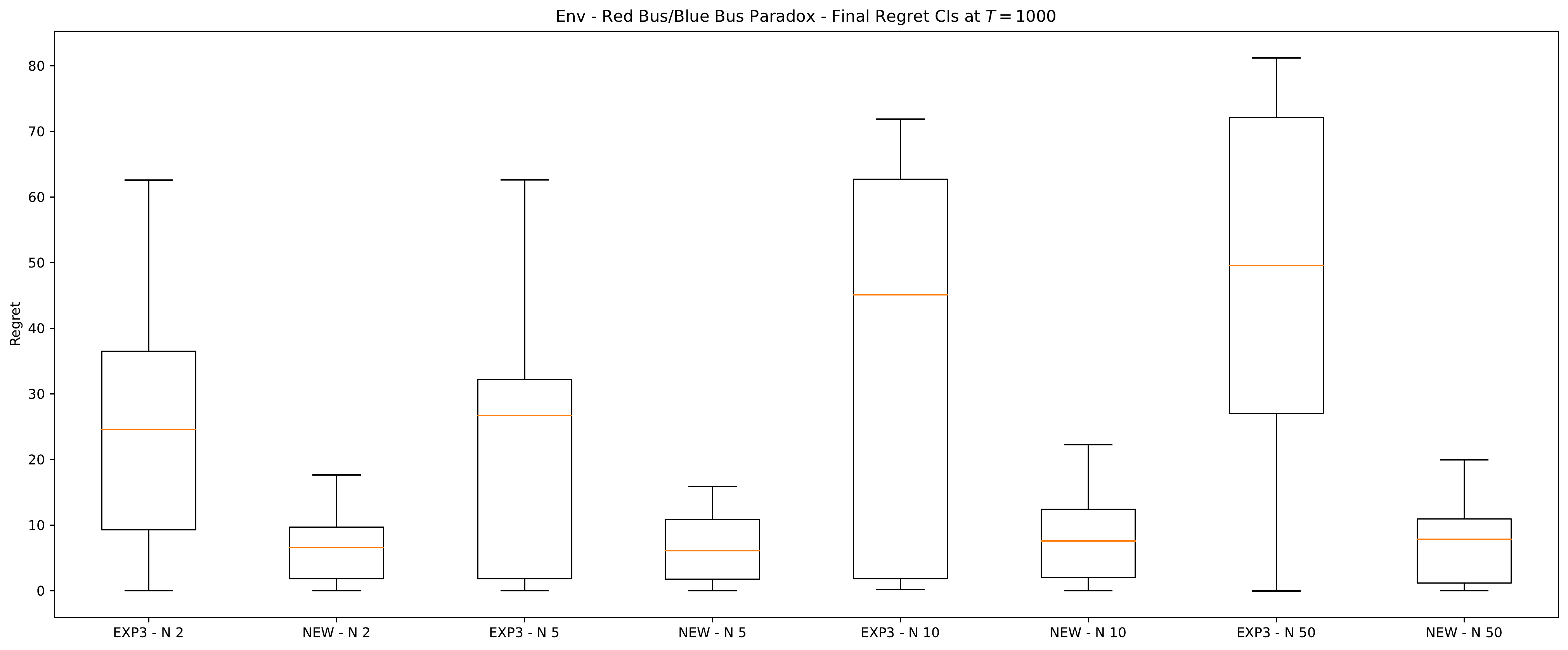}
\caption{Regret distribution at the final stepsize $T=1000$ for the Red Bus/Blue Bus environment.}
\label{fig:moustache_rbbb}
\end{figure}

\begin{figure}[h!]
\includegraphics[width=0.99\linewidth]{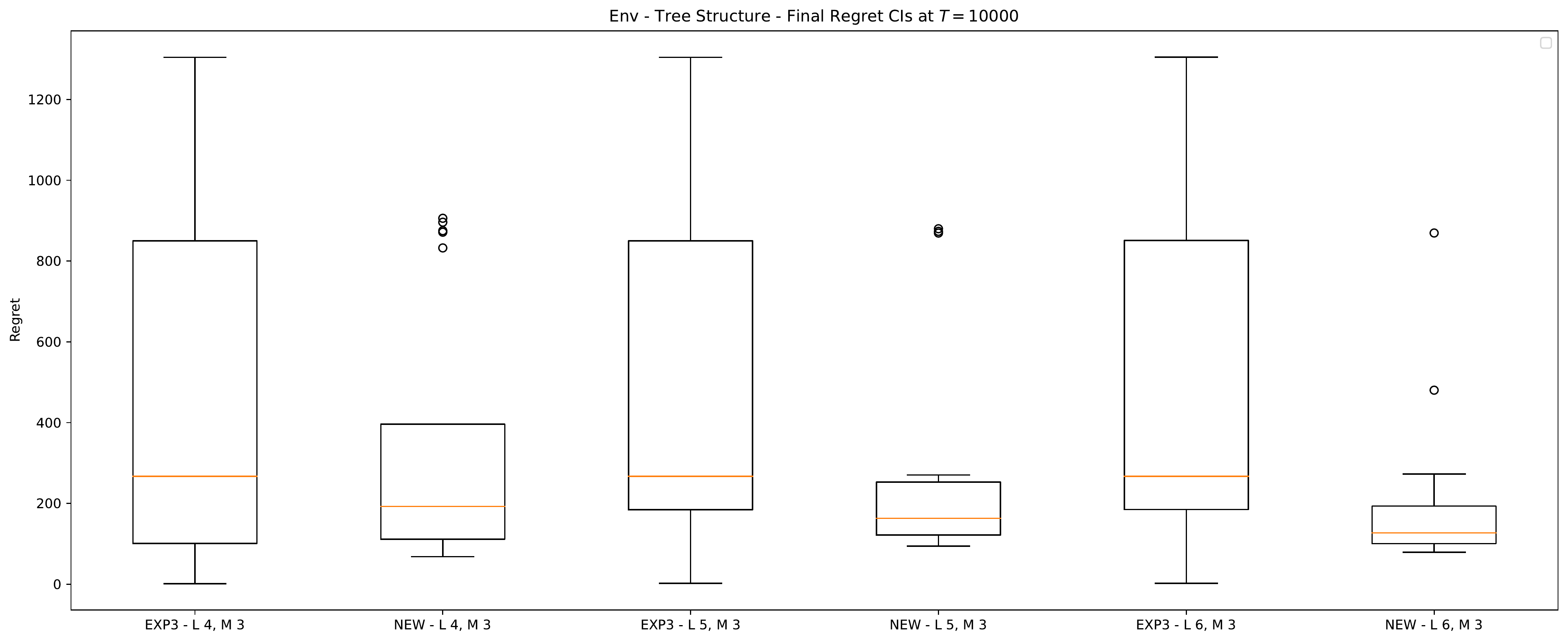}
\caption{Regret distribution at the final stepsize $T=10000$ when varying the depth parameter $\nLvls$.}
\label{fig:moustache_depth}
\end{figure}

\begin{figure}[h!]
\includegraphics[width=0.99\linewidth]{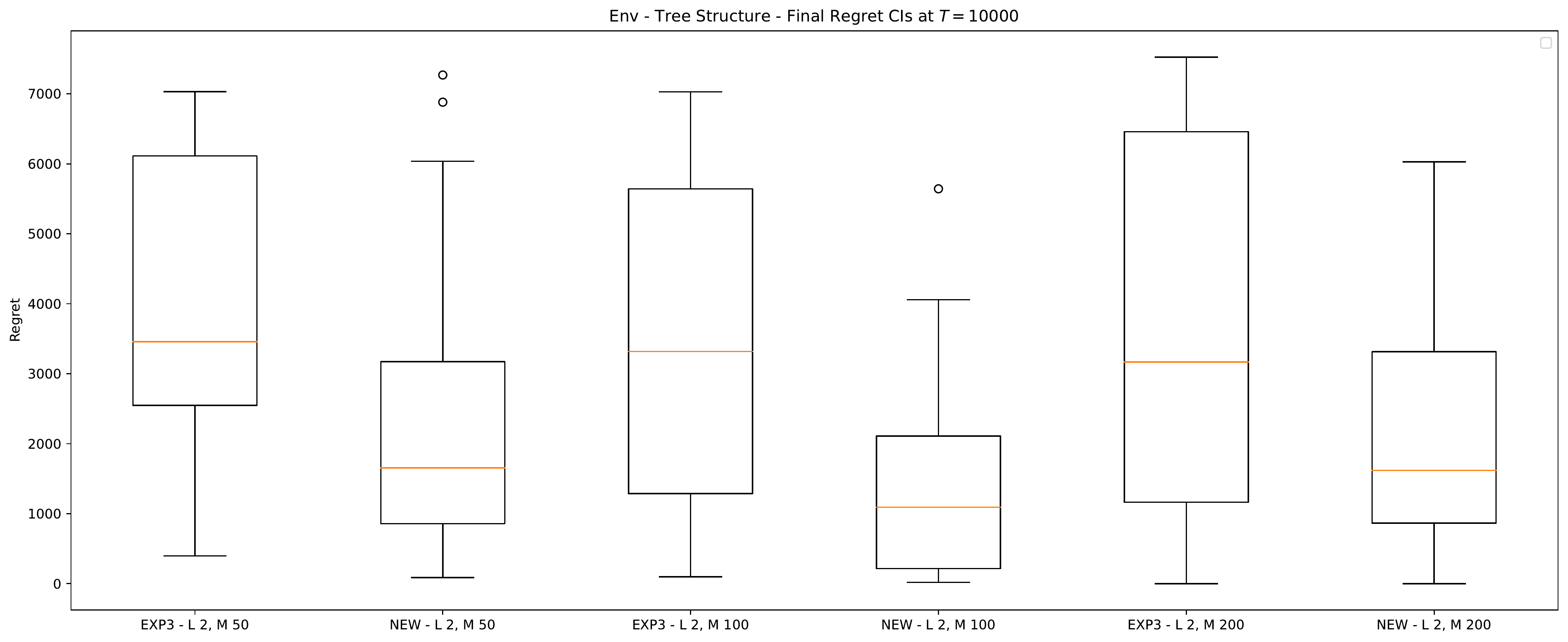}
\caption{Regret distribution at the final stepsize $T=10000$ when varying the wideness parameter $M$.}
\label{fig:moustache_wide}
\end{figure}

\subsection{Reproducibility}
We provide code for reproducibility of our experiments and plots, in addition to a more general implementation of both the NEW algorithm and EXP3 baseline. All experiments were run on a Mac book pro laptop, with 1 processor of 6 cores @2.6GHz (6-Core Intel Core i7).
The code and all experiments can be found in the attached .zip.

\section*{Acknowledgements}
\begingroup
\small
P.~Mertikopoulos is grateful for financial support by
the French National Research Agency (ANR) in the framework of
the ``Investissements d'avenir'' program (ANR-15-IDEX-02),
the LabEx PERSYVAL (ANR-11-LABX-0025-01),
MIAI@Grenoble Alpes (ANR-19-P3IA-0003),
and the bilateral ANR-NRF grant ALIAS (ANR-19-CE48-0018-01).
\endgroup

\bibliographystyle{icml2022}
\bibliography{bibtex/IEEEabrv,bibtex/Bibliography-PM}

\end{document}